\PassOptionsToPackage{hyphens}{url}
\documentclass[a4paper,fleqn]{cas-sc}

\usepackage[authoryear,longnamesfirst]{natbib}
\usepackage{tabularx}
\usepackage{placeins}
\usepackage{needspace}

\usepackage{amsthm}
\usepackage{mathtools}
\usepackage{algpseudocode}
\usepackage{microtype}
\graphicspath{{seam_assets/}}

\definecolor{seamCodeInk}{HTML}{1F2933}
\definecolor{seamCodeKeyword}{HTML}{1F4E79}
\definecolor{seamCodeBuiltin}{HTML}{0B6E69}
\definecolor{seamCodeString}{HTML}{A14F03}
\definecolor{seamCodeNumber}{HTML}{697783}

\makeatletter
\newcounter{algorithm}
\renewcommand{\thealgorithm}{\arabic{algorithm}}
\newcommand{\fps@algorithm}{t!}
\newcommand{\ftype@algorithm}{4}
\newcommand{\ext@algorithm}{loa}
\newcommand{\fnum@algorithm}{Algorithm~\thealgorithm}
\newcommand{\seam@algorithmrule}[1]{%
  \par\nointerlineskip\hrule height #1 depth 0pt\relax}
\long\def\seam@makealgorithmcaption#1#2{%
  \noindent{\bfseries #1:}\space #2\par
  \kern 2pt\seam@algorithmrule{0.4pt}\kern 2pt}
\newenvironment{algorithm}[1][t!]
  {\@float{algorithm}[#1]%
   \let\@makecaption\seam@makealgorithmcaption
   \seam@algorithmrule{0.8pt}\kern 2pt}
  {\par\kern 2pt\seam@algorithmrule{0.4pt}\end@float}
\makeatother

\newcommand{\R}{\mathbb{R}}
\DeclareMathOperator{\im}{im}
\DeclareMathOperator{\rank}{rank}
\DeclareMathOperator{\Unif}{Unif}
\newcommand{\dz}{d^{0}}
\newcommand{\Fui}{F(U_i)}
\newcommand{\Fuij}{F(U_{ij})}
\newcommand{\norm}[1]{\left\|#1\right\|}
\newcommand{\pinv}{{^{+}}}
\newcommand{\Id}{\mathrm{I}}
\newcommand{\cover}{\mathcal{U}}
\newcommand{\nerve}{\mathcal{N}_{\cover}}
\newcommand{\bigO}{\mathcal{O}}
\newcommand{\ie}{\textit{i.e.}}
\newcommand{\onehalf}{\tfrac{1}{2}}
\renewcommand{\eqref}[1]{\textup{Eq.~\ref{#1}}}
\newcommand{\Domain}{\mathcal{X}}        
\newcommand{\Nblind}{\mathcal{B}}        
\newcommand{\Iobs}{\mathcal{I}}          
\DeclareMathOperator{\rowsp}{row}          
\usepackage{tikz}
\usetikzlibrary{positioning,arrows.meta,fit,backgrounds}

\theoremstyle{plain}
\newtheorem{theorem}{Theorem}[section]
\newtheorem{lemma}[theorem]{Lemma}
\newtheorem{corollary}[theorem]{Corollary}
\newtheorem{proposition}[theorem]{Proposition}
\newtheorem{seamtheorem}{Theorem}

\newtheorem*{assumptionAzero}{Assumption A0}
\newtheorem*{theorem*}{Theorem}
\newtheorem*{proposition*}{Proposition}
\newtheorem*{corollary*}{Corollary}

\theoremstyle{definition}
\newtheorem{definition}[theorem]{Definition}

\theoremstyle{remark}
\newtheorem{remark}[theorem]{Remark}

\ExplSyntaxOn
\cs_set:Npn \__first_footerline:
{
  \group_begin:
  \small
  \sffamily
  \ifnum\theblind>0\relax
  \else
    \__short_authors: :~
  \fi
  { \rmfamily \itshape Preprint }
  \group_end:
}
\ExplSyntaxOff

\begin{document}
\let\WriteBookmarks\relax
\let\printorcid\relax

\shorttitle{Global consistency beyond local accuracy}
\shortauthors{G.L.R. N'guessan and B.J. Kim}

\title[mode=title]{SEAM: Global consistency beyond local accuracy in scientific machine learning}
\hypersetup{ pdftitle={SEAM: Global consistency beyond local accuracy in scientific machine learning}, pdfauthor={Gnankan Landry Regis N'guessan and Bum Jun Kim}, pdfkeywords={scientific machine learning, cellular sheaves, explanation-admissibility, obstruction-guided diagnosis, identifiability, model monitoring}}

\author[1,2,3]{Gnankan Landry Regis N'guessan}
\ead{rnguessan@aimsric.org}
\credit{Conceptualization, Methodology, Software, Formal analysis, Investigation, Data curation, Validation, Visualization, Writing -- Original draft, Writing -- Review and editing}

\author[4]{Bum Jun Kim}
\cormark[1]
\ead{bumjun.kim@weblab.t.u-tokyo.ac.jp}
\credit{Formal analysis, Visualization, Supervision, Validation, Writing -- Review and editing}

\affiliation[1]{op={},organization={Axiom Research Group}}
\affiliation[2]{organization={Department of Applied Mathematics and Computational Science, The Nelson Mandela African Institution of Science and Technology}, addressline={404 Nganana, Kikwe, Arumeru, P.O. Box 447}, city={Arusha}, postcode={23311}, country={Tanzania}}
\affiliation[3]{organization={African Institute for Mathematical Sciences, Research and Innovation Centre}, addressline={KN 3 Road, Gasharu Cell, Kicukiro Sector, Kicukiro District}, city={Kigali}, country={Rwanda}}
\affiliation[4]{organization={Graduate School of Engineering, The University of Tokyo}, addressline={7-3-1 Hongo, Bunkyo-ku}, city={Tokyo}, postcode={113-8656}, country={Japan}}

\cortext[cor1]{Corresponding author}

\begin{abstract}
	Scientific machine learning commonly validates models at the level of a subdomain, a benchmark split, or an explanation for one prediction. Yet such local checks cannot establish whether the resulting explanations can be assembled into one globally admissible explanation. We introduce Scientific Explanation-Admissibility Machines (SEAM), a generator-agnostic framework that makes this local-to-global consistency question computable across regions, sensors, regimes, and model components. The finite explanation-sheaf instantiation SEAM-$\Omega$ represents each region by a structured explanation with state, closure, and observation channels together with optional contract metadata; compares neighboring explanations on their overlaps; and converts disagreement into a channel-resolved obstruction. This obstruction locates inconsistency and tests competing declared accounts by restricting each repair to the revisions that one account permits. Exact feasibility refutes or retains an account; when exact repair is unavailable, residual-aware regularized records provide a separately labeled empirical attribution. The framework also separates inconsistency from non-identifiability and monitors learned generators under distribution shift. We establish theorems for minimum-cost intervention and conservation-contract detectability, together with companion results for identifiability and closure recoverability. Across nineteen experiments involving synthetic partial differential equation systems; out-of-distribution Fourier neural operator (FNO) monitoring; four open datasets spanning traffic, hydrology, air quality, and electric power; and synthetic financial and industrial systems, SEAM detects incompatible explanations even when local predictions are accurate, and attributes failures to specific channels and overlaps. SEAM adds a global explanation-consistency audit to existing solvers and learning models, testing whether their local explanations form a coherent scientific account.
\end{abstract}


\maketitle

\section{Introduction}
\label{sec:intro}

Scientific models, whether physics-based or data-driven, are commonly judged by what happens inside one region of a problem. For example, a computational fluid dynamics solver is validated on a subdomain where reference data are available. A neural operator is evaluated on a benchmark split. A regional regression model is tested on a held-out fold of the local time series. Each such local assessment is made without testing whether the same system's explanations remain consistent across neighboring regions, different sensors, adjacent regimes, or solver interfaces.

These local checks therefore do not establish whether explanations from different regions form a coherent account of the full system. Those explanations must agree across regimes, scales, sensors, boundary interfaces, and modeling assumptions. A family of individually plausible explanations can be scientifically inadmissible: Every piece looks correct, while the whole cannot be glued into a coherent account. The same pattern appears in several routine scientific workflows. Two regional partial differential equation (PDE) solvers can independently pass local residual tests yet disagree at their shared boundary. Two seasonal demand models can fit their respective seasons yet contradict each other at the transition. An industrial monitoring system can have one zone whose drifted sensors lead its regional model to infer a source term that a neighboring zone's cleaner sensors contradict. A learned neural operator can score well on in-distribution test sets yet produce explanations that no finite-difference solver could ever produce. The central operational problem is therefore to determine when local scientific explanations can be assembled into one globally admissible explanation and, when assembly fails, how to diagnose the incompatibility and which repairs are feasible.

Answering the second half of that question is a hypothesis-testing problem rather than a scoring problem. Consider three neighboring zones of a monitored system whose regional models disagree on their shared overlaps. A practitioner facing that disagreement is rarely short of candidate explanations; the difficulty is that several explanations are available at once, and each one is consistent with every measurement taken so far. One account holds that the sensors feeding one zone have drifted, so the predicted states are right and the readings are wrong. A competing account holds that the sensors are sound, and the assumed physics of that zone omits a source term, so the readings are right and the closure is wrong. A third account holds that all three zones are individually sound and were merely calibrated on data that never constrained their common interfaces. Each account is a scientific hypothesis about the cause of the disagreement, and no regional residual or held-out score distinguishes these hypotheses, because each of those quantities is computed inside a single zone, whereas the accounts differ in their assertions about the interface between zones.

The Scientific Explanation-Admissibility Machines (SEAM) framework separates such accounts by making each account computable. A hypothesis about the cause of a disagreement is at the same time a statement about which parts of the explanations one is prepared to revise, and the framework formalizes the hypothesis through that statement: The drifted-sensor account permits revision of the sensor readings alone, the missing-physics account permits revision of the assumed closure alone, and the interface-calibration account permits revision of the predicted states. Confining a repair to the revisions that one hypothesis permits makes the hypothesis testable. Either some permitted revision removes the entire disagreement, in which case the hypothesis survives and carries a price, namely the cost of the cheapest revision that works; or no permitted revision removes that disagreement, in which case the overlap evidence refutes the hypothesis, and the irreducible remainder measures how badly that hypothesis fails. Costs quantify the repair burden within each surviving account; those costs are not used to rank an unrestricted repair against a more specific account. The test is auditable because the permitted revisions, cost metric, and remainder are all recorded. Figure~\ref{fig:overview} traces one such audit end to end, from the regional records through the overlap comparison and the channel-wise diagnosis to the single surviving account.

\begin{figure}[pos=t!]
	\centering
	\includegraphics[width=\linewidth]{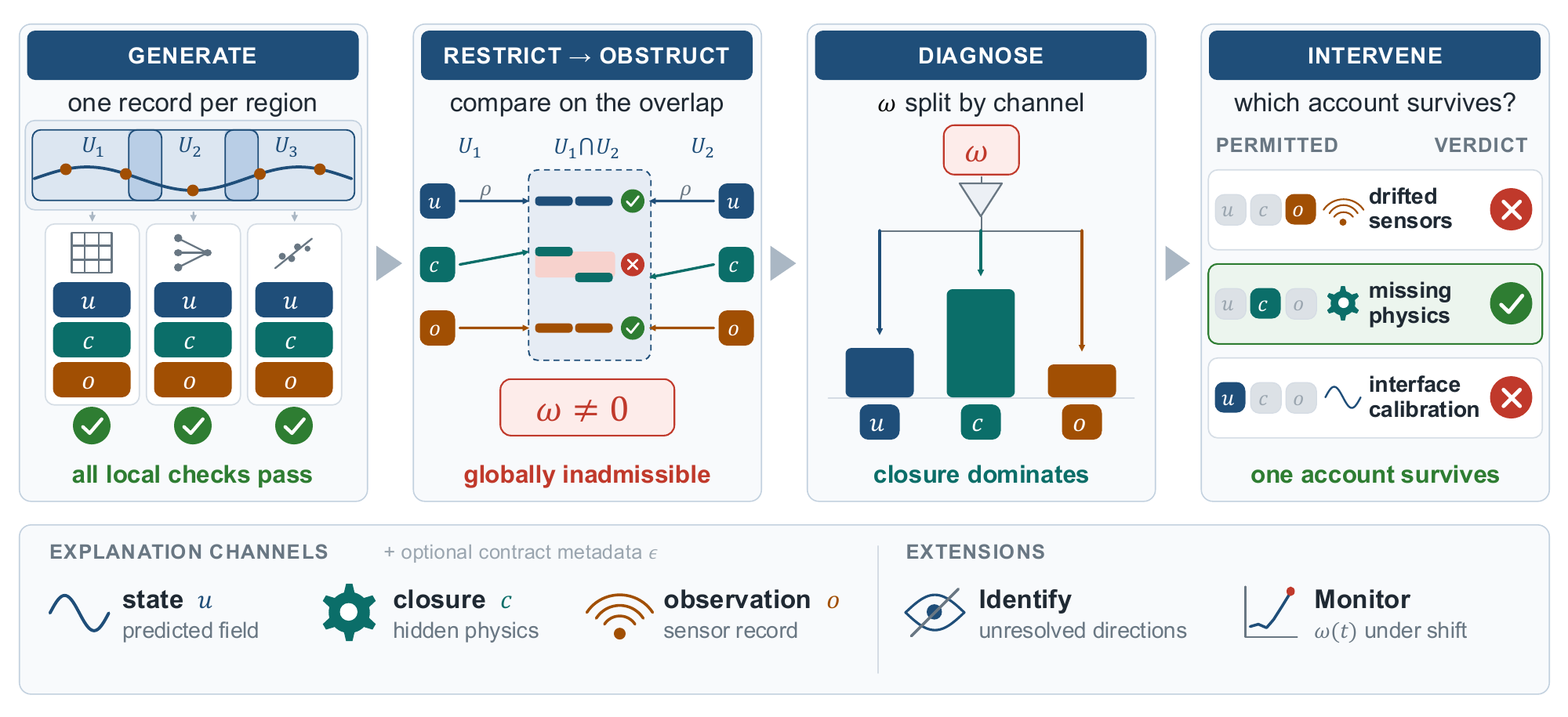}
	\caption{Overview of a SEAM audit, read from left to right. Each region $U_{i}$ of the cover carries one structured explanation record, emitted by its own backend and partitioned into the state, closure, and observation channels $u$, $c$, and $o$, together with optional contract metadata $\epsilon$; every record passes its own local check. Restriction compares two records on the overlap, and the disagreement that survives is the obstruction $\omega \neq 0$, so the family is globally inadmissible, although no local check failed. Diagnosis splits $\omega$ by channel and localizes the defect, here to the closure channel. Intervention states the drifted-sensor, missing-physics, and interface-calibration accounts as budgets that name the revisions each one permits, and only the account whose permitted revisions remove the entire obstruction survives. The Identify and Monitor extensions record unresolved admissible directions and track $\omega(t)$ under shift. The figure is schematic.}
	\label{fig:overview}
\end{figure}

A catalog of specific hypotheses can also fail as a whole, when every proper budget is refuted and only unrestricted revision of everything at once restores agreement. SEAM then records the attribution as unresolved instead of treating the unrestricted fallback as a causal account. The identifiability analysis separately records explanation directions that satisfy every overlap constraint and yet leave no trace in the current observations.

Current scientific machine learning toolchains address important parts of this problem. Physics-informed neural networks (PINNs) embed PDE residuals in training losses \citep{raissi2019physics}, whereas universal differential equations place trainable components inside differential-equation models \citep{rackauckas2020universal}. Neural operators learn maps between input and solution function spaces \citep{kovachki2023neural,lu2021learning}. Domain-decomposed PINN variants are important exceptions to the single-domain pattern: Conservative PINNs (cPINNs) and extended PINNs (XPINNs) impose solution or flux compatibility at subdomain interfaces, while finite-basis PINNs (FBPINNs) combine networks supported on overlapping subdomains \citep{jagtap2020conservative,jagtap2020extended, moseley2023finite}. The cPINN and XPINN formulations also support distributed execution on central processing units (CPUs) and graphics processing units (GPUs), as implemented by \citet{shukla2021parallel}. Domain-decomposed PINNs construct coupled PDE solutions, and knowledge-integration surveys organize mechanisms for incorporating scientific knowledge into learning \citep{willard2022integrating}. SEAM complements these approaches by auditing the compatibility of the structured regional records that their generators emit and by testing and comparing admissibility-restoring hypotheses. Relative to interface-residual and sheaf-consistency methods, SEAM integrates named scientific channels with residual-aware budgeted hypothesis evaluation, identifiability, and streaming monitoring.

\subsection{Explanation-admissibility and SEAM}

We treat explanation-admissibility as a local-to-global problem in the sense of sheaf theory \citep{curry2014sheaves,ghrist2014elementary,robinson2014topological}. A family of structured local explanations is admissible when neighboring restrictions agree on every overlap. SEAM-$\Omega$ stacks the regional records into $s\in C^0(F)$, assembles the restriction maps into a coboundary matrix $D$, and reports the resulting defect $\omega=Ds$ by channel and overlap. Section~\ref{sec:framework} gives the formal construction, and Appendix~\ref{app:notation} collects the symbols used throughout.

The SEAM paradigm has five stages: Generate, Restrict, Obstruct, Diagnose, and Intervene. An identifiability extension records what the current observations cannot determine, while a monitoring extension tracks $\omega$ under streaming inputs.

Throughout this paper, a backend is an implementation of a local generator $G_i$, such as an analytic rule, a finite-volume solver, a regression model, or a learned operator. Once the backend output has been mapped to the common channel schema, the downstream SEAM audit is unchanged.

\paragraph{SEAM-$\Omega$.}
SEAM-$\Omega$ instantiates this paradigm with a finite explanation sheaf $F$. The stalks of $F$ carry three primary diagnostic channels for the state $u$, closure $c$, and observation $o$, with an optional contract-metadata block $\epsilon$. The restriction maps are hand-crafted linear maps to overlap stalks, the coboundary is the matrix $D$ assembled from those restrictions, and the admissibility defect $\omega=Ds$ is computed in closed form. The central object of SEAM-$\Omega$ is the 1-cochain $\omega$ on the overlaps, whose decomposition and budgeted preimages define the diagnostic beyond any individual predicted state or learned residual.

The channel decomposition localizes disagreement in predicted states, hidden physics, or sensor records; optional contract metadata can be audited when an application supplies a nonzero restriction for that metadata block. For a declared intervention projector $P$, exact-hard repair is feasible when the defect lies in the range of the restricted coboundary, $\omega\in\im(DP)$; a feasible minimum-cost repair prices a surviving hypothesis, whereas an infeasible budget is reported with its irreducible residual. Regularized-soft records provide a separate empirical attribution used when the single-channel hard problems cannot fit the full defect. The blind admissible subspace $\Nblind=\ker D\cap\ker O$ records admissible directions invisible to the current observations, and $\omega(t)$ supplies the streaming monitoring signal.

\subsection{Contributions and paper organization}

The paper makes four contributions.

\paragraph{Contribution 1: explanation-admissibility as a new scientific object.}
Explanation-admissibility is the local-to-global gluing of structured local explanations, a scientific object distinct from local accuracy and residual satisfaction.

\paragraph{Contribution 2: the SEAM paradigm.}
The five-stage pipeline sketched above, together with identifiability and monitoring extensions, defines a reusable design pattern for scientific systems that report whether explanations can be globally assembled.

\paragraph{Contribution 3: the finite explanation-sheaf instantiation SEAM-$\Omega$.}
We give a concrete finite-dimensional linear instantiation of the paradigm using vector-space stalks and block-diagonal linear restrictions on the 1-skeleton of the cover's nerve. This instantiation yields two central theorems on minimum-cost budgeted intervention and conservation-contract detectability, with an identifiability proposition and a closure-recovery corollary.

\paragraph{Contribution 4: computational evidence.}
Nineteen experiments across synthetic PDE systems and six application domains evaluate the framework. Reporting for stochastic and deterministic experiments follows the aggregate run protocol documented in Appendix~\ref{app:seed_protocol}. The experiments evaluate local--global disagreement, closure recovery, exact-hard and regularized-soft intervention records, identifiability, backend interoperability, cross-domain consistency, and learned-generator monitoring. Section~\ref{sec:experiments} reports the complete numerical outcomes.

\paragraph{Paper organization.}

Section~\ref{sec:paradigm} develops the SEAM paradigm at a level abstracted from any one instantiation. Section~\ref{sec:framework} specializes to SEAM-$\Omega$ by introducing finite explanation sheaves over the 1-skeleton of a cover's nerve, with structured stalks and the channel-decomposed coboundary. Section~\ref{sec:diag} develops obstruction-guided diagnosis, the channel decomposition of $\omega$, and the budgeted intervention problem. Section~\ref{sec:ident} formalizes identifiability. Section~\ref{sec:monitor} extends the diagnostic to streaming monitoring of learned generators. Section~\ref{sec:theory} states the finite-dimensional result suite with explicit assumptions; detailed proofs are consolidated in Appendix~\ref{app:proofs}. Section~\ref{sec:algos} gives the algorithms and the computational workflow of SEAM-$\Omega$. Section~\ref{sec:experiments} reports nineteen experiments organized into eight groups aligned with the scientific claims that each group evaluates. Section~\ref{sec:related} compares SEAM with existing methods. Section~\ref{sec:discussion} discusses the design choices, interpretation guide, and relationship between the paradigm and the SEAM-$\Omega$ instantiation. The appendices consolidate notation, proofs, numerical details, the conservation-theoretic SEAM (CT-SEAM) backend, experimental protocols, data provenance, and robustness checks.

\section{The SEAM paradigm}
\label{sec:paradigm}

This section presents SEAM as a reusable local-to-global design principle.

\subsection{From local explanations to global admissibility}
\label{sec:local_to_global}

In SEAM, a scientific explanation is a structured regional object that records why a model produces its output. In a PDE solver, the explanation includes the predicted state $u$, the assumed source or closure $c$, the observations $o$ used for calibration, and optional contract metadata $\epsilon$, such as conservation-repair magnitudes. The metadata block may be zero-dimensional and is not an intervention budget. In a tabular regression model, the explanation includes the predicted output, residual statistics, sufficient statistics of the input distribution, and auxiliary information needed to interpret the prediction scientifically.

The defining property that distinguishes a scientific explanation from a prediction is that explanations from neighboring regions must be jointly assertable. If the explanation on region $U_{i}$ asserts that the local state is $u_{i}$ and the closure is $c_{i}$, and the explanation on region $U_{j}$ makes the analogous assertion, then, on the overlap $U_{i}\cap U_{j}$, those two assertions must agree as statements about the same physical quantity. Locally, each assertion may be plausible, internally consistent, and well supported by data. Jointly, the two assertions may be incompatible, because the evidence that certifies each assertion is confined to the region on which that evidence was produced. A regional residual, a held-out score, or a goodness-of-fit statistic on $U_{i}$ constrains the explanation $e_{i}$ only through quantities visible inside $U_{i}$, typically admitting an entire family of explanations that pass equally well without ever referring to $e_{j}$. Each generator is thus free to select a different member of its own locally plausible family, so agreement on $U_{i}\cap U_{j}$ is a further, joint requirement that no amount of local validation tests. Local plausibility is a property of the individual explanations, whereas admissibility is a property of the family. Figure~\ref{fig:local_global} contrasts a family of highly accurate regional models that disagree on their overlaps with a rougher family that agrees on every overlap.

This phenomenon is generic. Four mechanisms produce such joint incompatibility. We refer to these mechanisms as M1--M4 throughout:

\paragraph{Mechanism M1: independent regional construction.}
Regional models specified without reference to one another agree with their own local evidence and need not agree with each other on an overlap. Two PDE solvers with independently chosen meshes, time steps, or numerical fluxes are both convergent in their interiors, yet their boundary fluxes do not exactly cancel on the overlap, producing an asymptotically vanishing but operationally measurable disagreement. When the regional specifications themselves differ, as with independently calibrated amplitudes or coefficients, the disagreement persists under refinement. The analogous coupling problem is central to co-simulation, where independently implemented simulator units exchange interface variables at communication points \citep{gomes2018cosimulation}.

\paragraph{Mechanism M2: regime non-stationarity.}
Two seasonal time-series models, each accurate within its season, fit different autoregressive coefficients or different residual volatilities. At the inter-season transition, both models extrapolate, and the two extrapolations disagree by a magnitude unrelated to either model's in-regime error.

\paragraph{Mechanism M3: sensor--physics conflict.}
A subset of sensor streams is corrupted, drifted, or biased. A regional model that relies on the affected sensors infers the physics needed to fit those streams. A neighboring model fits cleaner sensors, inferring different physics. The disagreement appears as inconsistency between the two regional closures, even though both models are locally calibrated.

\paragraph{Mechanism M4: learned-generator distribution shift.}
A neural operator trained on a distribution of inputs is evaluated on a covariate-shifted input. Each region of the new input may remain close to inputs the operator has seen, while the joint output across regions is no longer on the operator's learned manifold. The local predictions are still well-formed and locally plausible;
the global assembly is not.

Together, these mechanisms expose the cross-region consistency gap that SEAM audits, and the same mechanisms organize the evaluation: Mechanism~M1 is realized by the independently constructed regional models of Section~\ref{sec:group1}, M2 by the seasonal-regime studies of Section~\ref{sec:group6}, M3 by the sensor-corruption study of Section~\ref{sec:group3}, and M4 by the neural-operator monitoring study of Section~\ref{sec:group7}. The four cases listed in Section~\ref{sec:intro} instantiate M1--M4 in that order.

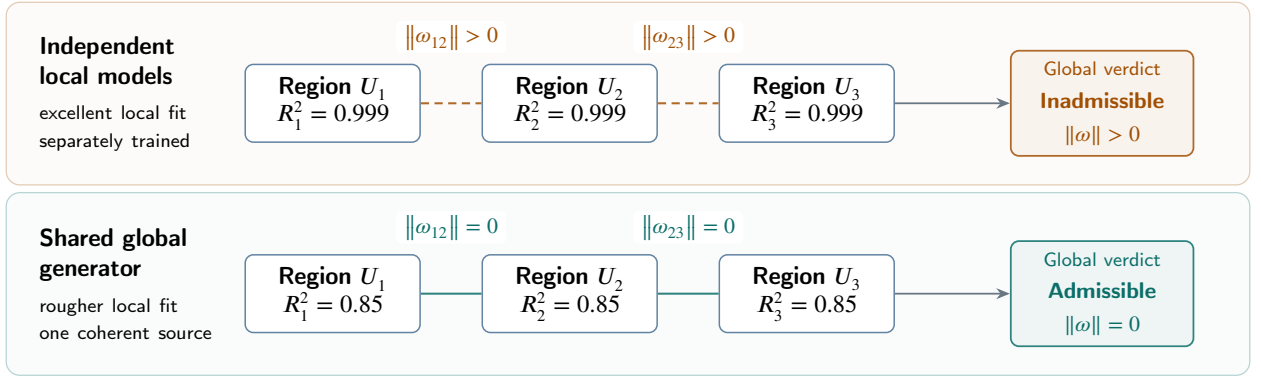
\begin{figure}[pos=t!]
	\centering
	\resizebox{\linewidth}{!}{%
		\begin{tikzpicture}[
			>=Stealth, every node/.style={font=\small}, rowtitle/.style={anchor=west, align=left, text width=2.35cm, inner sep=0pt}, region/.style={rectangle, rounded corners=3pt, draw=seamCodeKeyword!70, fill=white, minimum width=2.25cm, minimum height=1.0cm, align=center, inner sep=4pt, line width=0.55pt}, verdict/.style={rectangle, rounded corners=3pt, minimum width=2.35cm, minimum height=1.0cm, align=center, inner sep=4pt, line width=0.65pt}, badverdict/.style={verdict, draw=seamCodeString!85, fill=seamCodeString!7, text=seamCodeString}, goodverdict/.style={verdict, draw=seamCodeBuiltin!85, fill=seamCodeBuiltin!7, text=seamCodeBuiltin}, badlabel/.style={font=\footnotesize, text=seamCodeString, fill=white, rounded corners=2pt, inner xsep=3pt, inner ysep=1.5pt}, goodlabel/.style={font=\footnotesize, text=seamCodeBuiltin, fill=white, rounded corners=2pt, inner xsep=3pt, inner ysep=1.5pt}, badlink/.style={draw=seamCodeString!85, densely dashed, line width=0.75pt}, goodlink/.style={draw=seamCodeBuiltin!85, line width=0.75pt}, resultarrow/.style={-{Stealth[length=5pt,width=4pt]}, draw=seamCodeNumber, line width=0.65pt} ]
			\path[draw=seamCodeString!30, fill=seamCodeString!2, rounded corners=5pt, line width=0.45pt] (0,1.20) rectangle (16.0,3.55);
			\path[draw=seamCodeBuiltin!30, fill=seamCodeBuiltin!2, rounded corners=5pt, line width=0.45pt] (0,-1.25) rectangle (16.0,1.10);

			\node[rowtitle] at (0.42,2.38) {{\bfseries Independent local models}\\[2pt] {\scriptsize excellent local fit}\\[-1pt] {\scriptsize separately trained}};
			\node[region] (L1) at (4.20,2.25) {{\bfseries Region $U_{1}$}\\[-1pt]$R^{2}_{1}=0.999$};
			\node[region] (L2) at (7.25,2.25) {{\bfseries Region $U_{2}$}\\[-1pt]$R^{2}_{2}=0.999$};
			\node[region] (L3) at (10.30,2.25) {{\bfseries Region $U_{3}$}\\[-1pt]$R^{2}_{3}=0.999$};
			\draw[badlink] (L1.east) -- node[badlabel, above=18pt] {$\norm{\omega_{12}}>0$} (L2.west);
			\draw[badlink] (L2.east) -- node[badlabel, above=18pt] {$\norm{\omega_{23}}>0$} (L3.west);
			\node[badverdict] (LV) at (14.10,2.25) {{\scriptsize Global verdict}\\[1pt] {\footnotesize\bfseries Inadmissible}\\[1pt] {\footnotesize$\norm{\omega}>0$}};
			\draw[resultarrow] (L3.east) -- (LV.west);

			\node[rowtitle] at (0.42,-0.07) {{\bfseries Shared global generator}\\[2pt] {\scriptsize rougher local fit}\\[-1pt] {\scriptsize one coherent source}};
			\node[region] (B1) at (4.20,-0.20) {{\bfseries Region $U_{1}$}\\[-1pt]$R^{2}_{1}=0.85$};
			\node[region] (B2) at (7.25,-0.20) {{\bfseries Region $U_{2}$}\\[-1pt]$R^{2}_{2}=0.85$};
			\node[region] (B3) at (10.30,-0.20) {{\bfseries Region $U_{3}$}\\[-1pt]$R^{2}_{3}=0.85$};
			\draw[goodlink] (B1.east) -- node[goodlabel, above=18pt] {$\norm{\omega_{12}}=0$} (B2.west);
			\draw[goodlink] (B2.east) -- node[goodlabel, above=18pt] {$\norm{\omega_{23}}=0$} (B3.west);
			\node[goodverdict] (BV) at (14.10,-0.20) {{\scriptsize Global verdict}\\[1pt] {\footnotesize\bfseries Admissible}\\[1pt] {\footnotesize$\norm{\omega}=0$}};
			\draw[resultarrow] (B3.east) -- (BV.west);
		\end{tikzpicture}}%
	\caption{Local accuracy and global admissibility are orthogonal. The top row shows a schematic family of three locally trained models in which each model achieves high local accuracy but disagrees on overlaps, so the family is globally inadmissible. The bottom row shows three deliberately rougher local fits with $R^{2}_{i} = 0.85$ that share a single global generator and agree exactly on overlaps: $\norm{\omega} = 0$.}
	\label{fig:local_global}
\end{figure}

\subsection{The five-stage paradigm and its extensions}
\label{sec:five_stages}

We define SEAM as a five-stage data flow with identifiability and monitoring extensions, shown in Figure~\ref{fig:seam_paradigm}.

\begin{definition}[SEAM paradigm]
	\label{def:seam_paradigm}
	A SEAM system consists of the following five stages:
	\paragraph{Generate.}
	A family of local generators $\{G_{i}\}_{i\in I}$ produces, for each region $U_{i}$ of a finite cover $\cover$ of the problem domain, a structured local explanation $e_{i}$.

	\paragraph{Restrict.}
	For each overlap $U_{i}\cap U_{j}$, a restriction map $\rho_{i,ij}$ extracts from $e_{i}$ the object that must be compared on the overlap. The object can be a state value, a closure parameter, a posterior mean, a sensor mean, or an admissibility flag.

	\paragraph{Obstruct.}
	The disagreement between $\rho_{i,ij}(e_{i})$ and $\rho_{j,ij}(e_{j})$ on every overlap is assembled into an obstruction object $\omega$.

	\paragraph{Diagnose.}
	The obstruction is summarized using the diagnostic attributes exposed by the chosen instantiation, producing a report that localizes or qualifies the inadmissibility.

	\paragraph{Intervene.}
	A family of allowable intervention budgets is considered. Each budget states one scientific hypothesis about the cause of the obstruction by declaring the revisions that the hypothesis permits. For each budget, the system tests whether the corresponding repair problem can eliminate the obstruction. Feasible hard repairs are retained by the verdict; infeasible budgets contribute residual or diagnostic information declared by the chosen instantiation.

	Two further extensions are part of the paradigm:
	\paragraph{Identify.}
	An identifiability analyzer determines which parts of the admissible explanation object remain unresolved by current observations.

	\paragraph{Monitor.}
	A SEAM system is also expected to support streaming evaluation of $\omega(t)$ for a fixed family of local generators applied to streaming inputs.
\end{definition}

Definition~\ref{def:seam_paradigm} leaves the mathematical type of $e_{i}$, $\rho_{i,ij}$, and $\omega$ open. Four conditions suffice for an instantiation to realize this abstract definition. Each $\rho_{i,ij}$ must be well-defined for the local explanation object used by the instantiation. The equation $\omega=0$ must provide a notion of admissibility, and the intervention problem must provide a notion of cost. Finally, the instantiation must expose the diagnostic attributes needed to interpret $\omega$. The finite explanation-sheaf instantiation of Section~\ref{sec:framework} fulfills these requirements with vector-space stalks, a named channel decomposition, and linear repair equations.

\begin{figure}[pos=t!]
	\centering
	\resizebox{\linewidth}{!}{%
		\begin{tikzpicture}[
			>=Stealth, node distance=1.4cm and 0.84cm, every node/.style={font=\small}, stage/.style={rectangle, rounded corners=3pt, draw=seamCodeKeyword!70, fill=white, minimum width=1.95cm, minimum height=0.9cm, align=center, inner sep=4pt, line width=0.55pt}, extension/.style={rectangle, rounded corners=3pt, draw=seamCodeBuiltin!72, fill=white, minimum width=1.95cm, minimum height=0.72cm, align=center, inner sep=3pt, line width=0.55pt, font=\footnotesize}, monitor/.style={extension, draw=seamCodeNumber!68}, corepanel/.style={rounded corners=5pt, draw=seamCodeKeyword!28, fill=seamCodeKeyword!2, line width=0.45pt}, extensionpanel/.style={rounded corners=5pt, draw=seamCodeBuiltin!28, fill=seamCodeBuiltin!2, line width=0.45pt}, arr/.style={-{Stealth[length=5pt,width=4pt]}, line width=0.65pt, draw=seamCodeKeyword!78, shorten <=1.5pt, shorten >=1.5pt}, extarr/.style={arr, draw=seamCodeBuiltin!82}, monitorarr/.style={arr, draw=seamCodeNumber!82} ]
			\node[stage] (S1) at (0, 0) {{\bfseries Generate}\\[-1pt]$\{G_{i}\}_{i\in I}$};
			\node[stage,right=of S1] (S2) {{\bfseries Restrict}\\[-1pt]$\rho_{i,ij}$};
			\node[stage,right=of S2] (S3) {{\bfseries Obstruct}\\[-1pt]$\omega$};
			\node[stage,right=of S3] (S4) {{\bfseries Diagnose}\\[-1pt]report};
			\node[stage,right=of S4] (S5) {{\bfseries Intervene}\\[-1pt]repair};
			\draw[arr] (S1) -- (S2);
			\draw[arr] (S2) -- (S3);
			\draw[arr] (S3) -- (S4);
			\draw[arr] (S4) -- (S5);
			\node[extension,below=1.10cm of S4] (S6) {{\bfseries Identify}\\[-1pt]unresolved};
			\draw[extarr] (S5.south) -- ++(0,-0.55) -| (S6.north);
			\node[monitor,below=1.10cm of S3] (S8) {{\bfseries Monitor}\\[-1pt]$\omega(t)$};
			\draw[monitorarr] (S3.south) -- (S8.north);

			\begin{scope}[on background layer]
				\node[corepanel, fit=(S1)(S2)(S3)(S4)(S5), inner xsep=0.34cm, inner ysep=0.26cm] {};
				\node[extensionpanel, fit=(S8)(S6), inner xsep=0.34cm, inner ysep=0.26cm] {};
			\end{scope}
		\end{tikzpicture}}%
	\caption{The SEAM paradigm. The five core stages in the top row take a family of local generators through restriction, obstruction, diagnosis, and budgeted intervention. The Identify extension records unresolved admissible directions, and the Monitor extension reads $\omega$ as a stream. SEAM-$\Omega$ instantiates these stages with finite cellular sheaves.}
	\label{fig:seam_paradigm}
\end{figure}
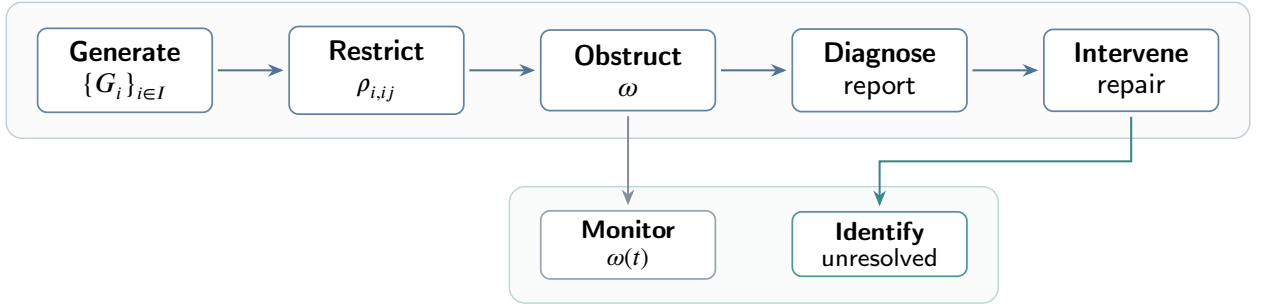

\subsection{Explanation-admissibility}
\label{sec:admissibility}

The central object of SEAM is explanation-admissibility.

\begin{definition}[Explanation-admissibility]
	\label{def:exp_admissibility}
	A family of local explanations $(e_{i})_{i\in I}$ on a cover $\cover$ is globally admissible if, for every overlap $(i,j)\in E$,
	\begin{align*}
		\rho_{i,ij}(e_{i})  =  \rho_{j,ij}(e_{j}).
	\end{align*}
	Equivalently, the obstruction object of the family vanishes: $\omega = 0$.
\end{definition}

In cellular-sheaf terms, admissibility is the global-section condition. The term emphasizes its scientific interpretation: Neighboring explanations make mutually compatible claims on every overlap. Together with admissibility, the following three definitions specify how a SEAM instantiation represents the obstruction diagnostically, delimits allowable repairs, and converts repair feasibility and cost information into a scientific verdict.

\begin{definition}[Diagnostic decomposition]
	\label{def:channel_omega}
	The obstruction object $\omega$ is reported together with diagnostic attributes exposed by the chosen SEAM instantiation. These attributes may be named channels, regions, uncertainty components, status fields, or other interpretable summaries. In SEAM-$\Omega$, the attributes are the three primary channels and the optional contract-metadata channel introduced in Section~\ref{sec:stalks}.
\end{definition}

\begin{definition}[Intervention budget]
	\label{def:budget}
	An intervention budget is a declared class of allowable repairs together with its cost convention and feasibility test. A hard budget is feasible for $\omega$ if some allowed repair eliminates the obstruction under the semantics of the chosen instantiation. In SEAM-$\Omega$, a budget is represented by an orthogonal projector $P$ and a cost metric $C$, and hard feasibility is the linear condition $\omega\in\im(DP)$ developed in Section~\ref{sec:budgeted}.
\end{definition}

\begin{definition}[Scientific verdict]
	\label{def:verdict}
	Given a candidate set of budgets $\{b_{k}\}$, a numerical zero-obstruction tolerance $\tau_{\mathrm{zero}}$, and an optional unrestricted fallback budget $b_{\mathrm{all}}$, a SEAM verdict is defined piecewise. If $\norm{\omega}\le\tau_{\mathrm{zero}}$, the verdict is $\mathsf{globally\_admissible}$, and no repair budget is selected. When $\norm{\omega}>\tau_{\mathrm{zero}}$, define the exact-feasible budget set by
	\begin{align*}
		\mathcal F(\omega) \coloneqq \{b_k:\text{the hard repair problem for }b_k \text{ is feasible for }\omega\}.
	\end{align*}
	In SEAM-$\Omega$, this feasible-set definition becomes $\mathcal F(\omega)=\{b_k:\omega\in\im(DP_{b_k})\}$. The feasible set of specific accounts is
	\begin{align*}
		\mathcal F_{\mathrm{proper}}(\omega) \coloneqq \mathcal F(\omega)\setminus\{b_{\mathrm{all}}\},
	\end{align*}
	with the set subtraction omitted when no unrestricted fallback is in the catalog. If $\mathcal F_{\mathrm{proper}}(\omega)=\varnothing$, the verdict is unresolved by the declared specific accounts and reports their residuals. If this set has one element, that account is the only specific account retained by exact feasibility. If the set has several elements, all are retained, and the attribution is non-unique. Each feasible budget carries its own minimum repair cost, but costs do not select among accounts with different permitted subspaces.
\end{definition}

In SEAM-$\Omega$, scientific verdicts are based on exact feasibility, while soft records provide comparative attribution together with their residuals. The exact branch is decisive when the obstruction concentrates in the channel that a budget permits. Under the block-diagonal channel restrictions adopted in Section~\ref{sec:channel_decomp}, a single-channel hard budget is feasible only when the obstruction components outside that channel vanish, so an obstruction spread across several channels refutes every single-channel account, and the attribution then rests on the soft records with their residuals. Section~\ref{sec:diag} develops this interpretation.

The three definitions above form one hypothesis test. A candidate set of budgets is a set of declared diagnostic accounts of the same obstruction, and $\mathcal F_{\mathrm{proper}}(\omega)$ collects the specific accounts that survive the overlap evidence. A budget outside $\mathcal F(\omega)$ is not merely expensive: No revision the budget permits removes the obstruction at all, so the evidence rules that account out, and only its residual is reported. An empty or multi-element proper feasible set makes the attribution unresolved or non-unique, respectively. The unrestricted fallback certifies repairability but is never interpreted as a causal account.

The objects compared on an overlap can be regional predictions, residuals, closure summaries, observation records, or contract metadata. Disagreements among those objects enter the same channel and overlap diagnosis, budgeted intervention, and blind-direction analysis. The local--global amplitude-mismatch study in Section~\ref{sec:local_global_amplitude_mismatch} provides the first numerical example.

\section{\texorpdfstring{SEAM-$\Omega$}{SEAM-Omega}: finite explanation sheaves}
\label{sec:framework}

We now instantiate the SEAM paradigm using finite cellular sheaves \citep{curry2014sheaves,robinson2014topological,ghrist2014elementary}. SEAM-$\Omega$ assembles the degree-zero coboundary of the pairwise-overlap complex to obtain closed-form diagnostic, intervention, identifiability, and monitoring operators. The channel-resolved results use the block-diagonal restrictions stated in Assumption~A0. Figure~\ref{fig:finite_sheaf} shows the resulting construction on a three-region cover.

\subsection{Covers and nerves}
\label{sec:covers}

Let $\Domain \subseteq \R^{d}$ be the computational domain. A cover $\cover = (U_{i})_{i\in I}$ is a finite collection of regions $U_i\subseteq \Domain$ with $\bigcup_{i\in I} U_{i} = \Domain$. Overlaps are permitted, and pairwise intersections $U_{ij} \coloneqq U_{i}\cap U_{j}$ may be empty. A nonempty $U_{ij}$ is the comparison support for the corresponding edge. In spatial applications, comparison supports are normally required to have nontrivial measure; an abstract cover may instead specify a finite set of shared evaluation points.

The nerve of $\cover$, denoted $\nerve$, is the simplicial complex with $k$-simplices given by the $(k+1)$-tuples $(i_{0}, \ldots, i_{k})$ of indices whose corresponding intersection is nonempty. For SEAM-$\Omega$, we truncate $\nerve$ to its 1-skeleton. The vertices of that 1-skeleton are regions $i\in I$, and the edges are pairs $E = \{(i,j) : i < j,  U_{ij}\neq\emptyset\}$.

\begin{remark}[Cover construction in practice]
	For PDE applications, $\cover$ is typically chosen by domain decomposition with a prescribed overlap fraction $\alpha_{\mathrm{ov}}$. For time-series applications, $\cover$ is chosen by regime decomposition across seasons, market regimes, or fault states, with overlap windows at regime transitions. For purely data-driven applications, $\cover$ may be specified abstractly via a list of overlapping evaluation points. Evaluation partitions such as benchmark splits or held-out folds use this same construction: Each partition is one region, and the evaluation points scored under two partitions form their overlap. Disjoint partitions carry no edge until shared evaluation points are designated. The SEAM-$\Omega$ implementation supports both spatial and abstract covers, as described in Section~\ref{sec:algos}.
\end{remark}

\begin{figure}[pos=t!]
	\centering
	\resizebox{\linewidth}{!}{%
		\begin{tikzpicture}[
			>=Stealth, every node/.style={font=\small}, region/.style={rectangle, rounded corners=3pt, draw=seamCodeKeyword!70, fill=white, minimum width=2.80cm, minimum height=2.68cm, align=center, inner sep=0pt, line width=0.55pt}, overlap/.style={rectangle, rounded corners=3pt, draw=seamCodeBuiltin!75, fill=seamCodeBuiltin!3, minimum width=1.50cm, minimum height=2.20cm, align=center, inner sep=0pt, line width=0.55pt}, regiontitle/.style={font=\bfseries, text=seamCodeKeyword}, overlaptitle/.style={font=\bfseries\footnotesize, text=seamCodeBuiltin}, stalklabel/.style={font=\scriptsize, text=seamCodeNumber}, channel/.style={font=\scriptsize, anchor=west, text=seamCodeInk}, overlapbody/.style={font=\scriptsize, align=center, text=seamCodeInk}, restriction/.style={-{Stealth[length=5pt,width=4pt]}, line width=0.65pt, draw=seamCodeBuiltin!85, shorten <=1.5pt, shorten >=1.5pt}, maplabel/.style={font=\scriptsize, text=seamCodeKeyword, fill=white, rounded corners=1.5pt, inner xsep=2.5pt, inner ysep=1pt}, defect/.style={align=center, inner sep=0pt, font=\scriptsize, text=seamCodeString}, defectlink/.style={-{Stealth[length=4.5pt,width=3.5pt]}, draw=seamCodeString!72, line width=0.55pt, shorten >=1pt}, panel/.style={rounded corners=5pt, draw=seamCodeKeyword!24, fill=seamCodeKeyword!2, line width=0.45pt} ]
			\node[region]  (U1)  at (0.0,0) {};
			\node[overlap] (U12) at (3.2,0) {};
			\node[region]  (U2)  at (6.4,0) {};
			\node[overlap] (U23) at (9.6,0) {};
			\node[region]  (U3)  at (12.8,0) {};

			\foreach \nodeid/\idx in {U1/1,U2/2,U3/3}{ \node[regiontitle] at ([yshift=0.93cm]\nodeid.center) {Region $U_{\idx}$}; \node[stalklabel] at ([yshift=0.61cm]\nodeid.center) {$F(U_{\idx})$}; \draw[seamCodeKeyword!20,line width=0.4pt] ([xshift=-1.16cm,yshift=0.39cm]\nodeid.center) -- ([xshift=1.16cm,yshift=0.39cm]\nodeid.center); \node[channel] at ([xshift=-1.12cm,yshift=0.11cm]\nodeid.center) {state $u_{\idx}$}; \node[channel] at ([xshift=-1.12cm,yshift=-0.24cm]\nodeid.center) {closure $c_{\idx}$}; \node[channel] at ([xshift=-1.12cm,yshift=-0.59cm]\nodeid.center) {observation $o_{\idx}$}; \node[channel] at ([xshift=-1.12cm,yshift=-0.94cm]\nodeid.center) {metadata $\epsilon_{\idx}$}; }

			\foreach \nodeid/\pair in {U12/12,U23/23}{ \node[overlaptitle] at ([yshift=0.67cm]\nodeid.center) {$U_{\pair}$}; \draw[seamCodeBuiltin!24,line width=0.4pt] ([xshift=-0.55cm,yshift=0.39cm]\nodeid.center) -- ([xshift=0.55cm,yshift=0.39cm]\nodeid.center); \node[overlapbody] at ([yshift=-0.10cm]\nodeid.center) {overlap\\[-1pt]stalk}; \node[stalklabel] at ([yshift=-0.69cm]\nodeid.center) {$F(U_{\pair})$}; }

			\draw[restriction] (U1.east) -- node[maplabel,above=3pt] {$\rho_{1,12}$} (U12.west);
			\draw[restriction] (U2.west) -- node[maplabel,above=3pt] {$\rho_{2,12}$} (U12.east);
			\draw[restriction] (U2.east) -- node[maplabel,above=3pt] {$\rho_{2,23}$} (U23.west);
			\draw[restriction] (U3.west) -- node[maplabel,above=3pt] {$\rho_{3,23}$} (U23.east);

			\node[defect] (D12) at (3.2,-1.84) {$\omega_{12}=\rho_{2,12}s_{2}-\rho_{1,12}s_{1}$};
			\node[defect] (D23) at (9.6,-1.84) {$\omega_{23}=\rho_{3,23}s_{3}-\rho_{2,23}s_{2}$};
			\draw[defectlink] (U12.south) -- (D12.north);
			\draw[defectlink] (U23.south) -- (D23.north);

			\begin{scope}[on background layer]
				\node[panel,fit=(U1)(U12)(U2)(U23)(U3)(D12)(D23), inner xsep=0.28cm,inner ysep=0.36cm] {};
			\end{scope}
		\end{tikzpicture}}%
	\caption{Finite explanation sheaf $F$ on a three-region cover. Each region $U_{i}$ carries a structured stalk partitioned into state, closure, and observation channels plus optional contract metadata. Each overlap $U_{ij}$ carries an overlap stalk $\Fuij$ and two restriction maps. The raw admissibility defect $\omega = Ds$ collects the per-overlap disagreements $\omega_{ij} = \rho_{j,ij}s_{j} - \rho_{i,ij}s_{i}$. Under the block-diagonal channel restrictions of Assumption~A0, the channel decomposition of $\omega$ commutes with the coboundary.}
	\label{fig:finite_sheaf}
\end{figure}
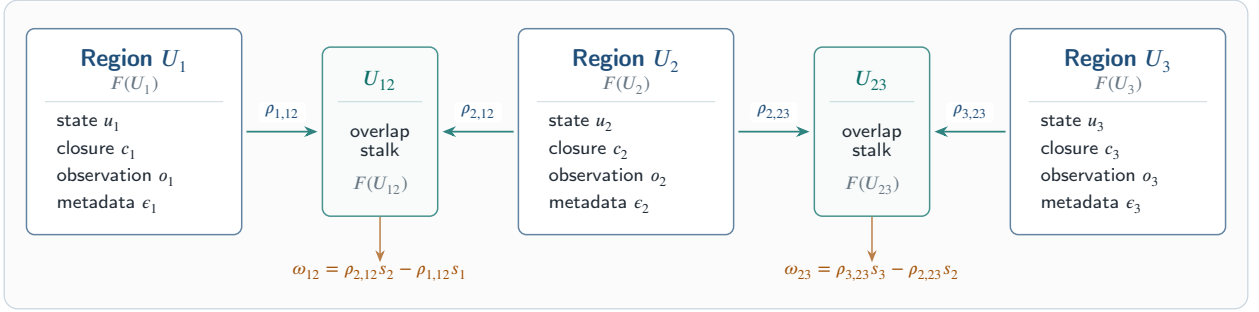

\subsection{Explanation stalks and channel conventions}
\label{sec:stalks}

We now equip each region with a structured local explanation space.

\begin{definition}[Local explanation]
	\label{def:local_explanation_ext}
	A local explanation on region $U_{i}$ is a tuple
	\begin{align*}
		e_{i}  =  (u_{i},  c_{i},  o_{i},  \epsilon_{i}),
	\end{align*}
	where $u_{i}\in\R^{n_{u,i}}$ is the discretized state field on the per-region grid and contains predicted values of the dependent variable of interest. The closure vector $c_{i}\in\R^{n_{c,i}}$ represents parameters or discretized values of hidden source terms, unresolved subgrid physics, or another latent block that the local generator attributes to the region. The observation summary $o_{i}\in\R^{n_{o,i}}$ contains sensor means and the flattened diagonal of the local measurement covariance. The optional contract-metadata vector $\epsilon_{i}\in\R^{n_{\epsilon,i}}$ contains per-contract repair magnitudes, conservation residuals, and other scalar admissibility flags emitted by the local generator; $n_{\epsilon,i}$ may be zero. Formulas label this optional block as $\mathrm{meta}$. The block is metadata and is distinct from an intervention budget $P$, which specifies a subspace of allowable corrections as in Definition~\ref{def:budget}. The embedding of the explanation is the concatenation $\mathrm{vec}(e_{i}) \in \R^{d_{i}}$ with $d_{i} \coloneqq n_{u,i} + n_{c,i} + n_{o,i} + n_{\epsilon,i}$. The common uniform-dimensional case is recovered by setting the nonzero dimensions independent of $i$.
\end{definition}

\begin{definition}[Explanation stalk]
	\label{def:stalk}
	The explanation stalk at region $U_{i}$ is $\Fui \coloneqq \R^{d_{i}}$, equipped with the standard inner product and the partition into three primary diagnostic blocks and one optional metadata block, $\Fui = \Fui^{\mathrm{state}} \oplus \Fui^{\mathrm{closure}} \oplus \Fui^{\mathrm{obs}} \oplus \Fui^{\mathrm{meta}}$ with dimensions $n_{u,i}$, $n_{c,i}$, $n_{o,i}$, and $n_{\epsilon,i}$.
\end{definition}

The choice of which fields go into which channel is part of the modeling design. We discuss two canonical conventions.

\paragraph{PDE convention.}
For PDE generators, the state block holds the discretized solution on the per-region grid; the closure block holds parameters or discretized values of any source, forcing, or subgrid model; and the observation block holds the sensor means and covariance diagonal on the region. When present, the optional metadata block holds cumulative magnitudes of conservation repairs applied during the solve.

\paragraph{Regression convention.}
For wrapped regression predictors, as used throughout the tabular regression studies, the state block holds the predictor outputs at the overlap evaluation points; the closure block is zero or carries residual statistics; the observation block is zero; and the optional metadata block is absent. The framework gracefully degenerates to a pure state-channel diagnostic in this case.

\subsection{Restriction maps, coboundary assembly, and the raw admissibility defect}
\label{sec:restrictions}

For each overlap $(i,j)\in E$, we fix an overlap stalk $\Fuij \coloneqq \R^{m_{ij}}$ and two restriction maps $\rho_{i,ij}\colon \Fui \to \Fuij$ and $\rho_{j,ij}\colon F(U_{j}) \to \Fuij$.

\begin{definition}[Restriction map]
	\label{def:restriction}
	A restriction map $\rho_{i,ij}\colon \Fui \to \Fuij$ is a linear map represented by a real $m_{ij}\times d_{i}$ matrix that sends $\mathrm{vec}(e_{i})$ to the quantities that must be compared on the overlap. The map respects channels, meaning that $\Fuij$ decomposes as $\Fuij = \Fuij^{\mathrm{state}} \oplus \Fuij^{\mathrm{closure}} \oplus \Fuij^{\mathrm{obs}} \oplus \Fuij^{\mathrm{meta}}$ and $\rho_{i,ij}$ acts block-diagonally with respect to the two partitions.
\end{definition}

The state-channel block of $\rho_{i,ij}$ is typically an interpolation operator that resamples the per-region state grid onto the overlap grid. The closure-channel block extracts the closure quantities relevant to the overlap region. In the simplest case, that block is the identity on $\Fui^{\mathrm{closure}}$ followed by a linear map into $\Fuij^{\mathrm{closure}}$. The observation-channel block selects the sensor entries that fall in the overlap. The optional metadata block is zero unless a contract quantity, such as a global conservation residual, is explicitly shared across regions.

\begin{remark}[Hand-crafted and learned restrictions]
	\label{rem:learned_restr}
	SEAM-$\Omega$ uses geometry-defined restrictions, as in application-specific constructions such as the discourse sheaves of \citet{hansen2021opinion}. Learned restriction maps, reviewed in Section~\ref{sec:related_sheaves}, can implement the same interface whenever those maps remain interpretable.
\end{remark}

We assemble the per-region stalks into the cochain spaces
\begin{align*}
	C^{0}(F)  \coloneqq  \bigoplus_{i\in I} \Fui, \qquad C^{1}(F)  \coloneqq  \bigoplus_{(i,j)\in E} \Fuij.
\end{align*}

\begin{definition}[Coboundary]
	\label{def:coboundary}
	The coboundary of $F$ is the linear map $\dz\colon C^{0}(F) \to C^{1}(F)$ defined by
	\begin{align*}
		(\dz s)_{ij}  \coloneqq  \rho_{j,ij}  s_{j} - \rho_{i,ij}  s_{i}, \qquad s = (s_{i})_{i\in I} \in C^{0}(F), (i,j) \in E.
	\end{align*}
	We write $D \coloneqq \mathrm{mat}(\dz)$ for the assembled coboundary matrix in the standard bases.
\end{definition}

A global section of $F$ is an element $s\in\ker\dz$: a family of stalk vectors whose restrictions agree on every overlap.

\begin{definition}[Raw admissibility defect]
	\label{def:raw_defect}
	Given a family of local explanations $(e_{i})_{i\in I}$, let $s_{i} \coloneqq \mathrm{vec}(e_{i})$. The raw admissibility defect $\omega \in C^{1}(F)$ is the 1-cochain
	\begin{align*}
		\omega  \coloneqq  \dz s, \qquad \omega_{ij}  = \rho_{j,ij}  s_{j} - \rho_{i,ij}  s_{i}, \qquad (i,j) \in E.
	\end{align*}
	The terms raw admissibility defect and raw obstruction cochain are synonymous for this $\omega$.
\end{definition}

The symbol $\norm{\cdot}$ denotes the Euclidean norm induced by the standard inner products, and an explicit subscript $2$ marks that same norm. The one departure is $\norm{\cdot}_{C}$, the norm induced by the positive-definite cost metric $C$ introduced in Section~\ref{sec:raw_budgeted}. Channel weights and whitening maps act on the argument, as in $\norm{W_{\bullet}\omega^{\bullet}}_{2}$, leaving the norm Euclidean; the remaining subscripts name the object being measured rather than a metric. The norm $\norm{\omega}$ is the scalar summary of inadmissibility, and the channel norms given in \eqref{eq:channel_decomp}, together with the per-overlap norms, refine that summary. The subsection below introduces hypothesis-specific budget residuals.

\subsection{\texorpdfstring{Channel decomposition of $\omega$}{Channel decomposition of omega}}
\label{sec:channel_decomp}

\begin{assumptionAzero}[Block-diagonal channel restrictions]
	For the channel decomposition to commute with the coboundary, restriction maps act block-diagonally with respect to the channel partition. Each $\rho_{i,ij}$ has no off-block entries mixing state, closure, observation, and optional metadata components. This block-diagonal restriction structure is a modeling choice we adopt throughout SEAM-$\Omega$. We index this assumption by zero because the block-diagonal structure is adopted here for the construction itself rather than for any individual result;
	Section~\ref{sec:theory} lists Assumption~A0 alongside the repair and model conventions used by the result suite.
\end{assumptionAzero}

Under Assumption~A0, $D$ is block diagonal with respect to the channel decomposition; its pseudoinverse and the projector $DD\pinv$ inherit the block structure. The defect therefore decomposes:
\begin{equation}
	\label{eq:channel_decomp}
	\omega  = \omega^{\mathrm{state}} \oplus \omega^{\mathrm{closure}} \oplus \omega^{\mathrm{obs}} \oplus \omega^{\mathrm{meta}},
\end{equation}
where $\bullet$ ranges over the three primary channel labels and the optional metadata label $\{\mathrm{state},\mathrm{closure},\mathrm{obs},\mathrm{meta}\}$, and each overlap component $\omega^{\bullet}_{ij}$ lives in the corresponding sub-block $\Fuij^{\bullet}$. Channel comparisons use a fixed channel metric or normalization. In the simplest unit-scaled case, we summarize each channel by its norm: $\norm{\omega^{\mathrm{state}}}$, $\norm{\omega^{\mathrm{closure}}}$, $\norm{\omega^{\mathrm{obs}}}$, and $\norm{\omega^{\mathrm{meta}}}$. The three primary scalars provide the diagnostic report; the optional meta-labeled scalar records contract-metadata disagreement when that block is present. With channel weights or whitening matrices $W_{\bullet}$, we compute the denominator and, when $Z>0$, the calibrated fractions
\begin{align*}
	Z\coloneqq\sum_{\circ}\norm{W_{\circ}\omega^{\circ}}_{2}^{2}, \qquad p_{\bullet}\coloneqq \frac{\norm{W_{\bullet}\omega^{\bullet}}_{2}^{2}}{Z} \quad (Z>0),
\end{align*}
where $\circ$ ranges over the active labels in the same set as $\bullet$. The channel reported as dominant is the one maximizing $p_{\bullet}$, so a dominance report is meaningful only relative to the declared $W_{\bullet}$. Section~\ref{sec:channel_diagnosis} states the reporting rule that these fractions feed, including the degenerate $Z=0$ case.

\begin{remark}[Resolution of the channel summary]
	The channel decomposition is a first-order summary of inadmissibility. Two systems with identical channel norms can have qualitatively different obstruction structures, differing in their $\omega_{ij}$ profiles across overlaps. The full $\omega$, the per-overlap norms $\norm{\omega_{ij}}$, the channel report, and the budgeted intervention costs together provide the richer report detailed in Section~\ref{sec:diag}.
\end{remark}

\subsection{Budgeted obstruction and residual-aware attribution}
\label{sec:raw_budgeted}

A scientific hypothesis restricts the allowed correction to a strict subspace $\im P \subsetneq C^{0}(F)$. For example, a hypothesis may allow revision only of the closure block. The hard residual below tests whether the restricted operator $DP$ can reproduce the full raw defect. A regularized-soft record is reported separately when the restricted operator cannot.

\begin{definition}[Budgeted obstruction]
	\label{def:budgeted_obstruction}
	Let $P\in\R^{N\times N}$ be the orthogonal projector onto a closed budget subspace of $C^{0}(F)$, and let $C\succ0$ be the metric used to measure corrections. The hard budgeted obstruction residual at budget $P$ is
	\begin{align*}
		r_{P}^{\mathrm{hard}}(\omega)  \coloneqq \omega - (DP)(DP)\pinv \omega \in  \ker (DP)^{\top}.
	\end{align*}
	The budgeted feasibility indicator is the Boolean $\mathbf{1}[r_{P}^{\mathrm{hard}}(\omega) = 0]$, equivalently $\mathbf{1}[\omega\in\im(DP)]$. The hard budgeted intervention cost is the metric-dependent extended-real value
	\begin{align*}
		c_{P,C}^{\mathrm{hard}}(\omega) \coloneqq
		\begin{cases}
			\norm{\delta_{P,C}^{\star}}_{C}^{2}, & \omega\in\im(DP),    \\
			+\infty,                             & \omega\notin\im(DP),
		\end{cases}
	\end{align*}
	where $\delta_{P,C}^{\star}$ is the minimum-cost feasible correction from Theorem~\ref{thm:intervention} in the first case. The residual depends on $P$ and the codomain metric used for projection, while the minimizer and cost also depend on $C$. We often write $r_P$ and $c_P^{\mathrm{hard}}$ when $C$ and the hard-residual convention are fixed. A budget status is exact-feasible, tolerance-feasible, or infeasible with a soft diagnostic. The modifier exact-hard is reserved for costs and verdicts derived solely from exact-feasible hard problems. Scientific verdicts are based on exact feasibility, while tolerance and soft empirical attribution records include their residuals \citep{engl1996regularization}.
\end{definition}

At the unrestricted budget $P=\Id$, every raw admissibility defect is exactly repairable. For proper budgets such as $\mathsf{AllowClosure}$ or $\mathsf{AllowSensorRejection}$, $r_{P}^{\mathrm{hard}}(\omega)$ is generically nontrivial and is the operational signal of hypothesis feasibility.

\section{Obstruction-guided diagnosis and intervention}
\label{sec:diag}

We now turn the obstruction cochain into a scientific diagnostic. The pipeline has two primary outputs. The per-channel decomposition provides a static map of where disagreement occurs, while the budgeted preimage computation tests which permitted revisions can explain that disagreement. Exact-hard feasibility and the set of retained proper budgets determine the formal verdict; each feasible budget's cost quantifies its own repair burden. Residual-aware regularized records are reported separately as empirical attribution rather than as exact tests.

\subsection{From obstruction to channel diagnosis}
\label{sec:channel_diagnosis}

The first component reads the channel decomposition of \eqref{eq:channel_decomp} without solving any optimization.

\paragraph{Total and per-channel norms.}
We compute the total raw norm and, for each channel $\bullet\in\{\mathrm{state},\mathrm{closure},\mathrm{obs}, \mathrm{meta}\}$, the corresponding per-channel norm:
\begin{align*}
	\norm{\omega}_{2}^{2} = \sum_{(i,j)\in E}\norm{\omega_{ij}}_{2}^{2}, \qquad \norm{\omega^{\bullet}}_{2}^{2} = \sum_{(i,j)\in E}\norm{\omega^{\bullet}_{ij}}_{2}^{2}.
\end{align*}
If $\norm{\omega}_{2}\le\tau_{\mathrm{zero}}$, no dominant channel is reported, and the verdict is $\mathsf{globally\_admissible}$. Otherwise, we form the calibrated fractions $p_{\bullet}$ of Section~\ref{sec:channel_decomp}, with $W_{\bullet}=\Id$ only in the unit-scaled default. If $Z>0$, channel dominance is reported as the channel maximizing $p_{\bullet}$. If $Z=0$, no dominant channel is reported, and the fraction vector is left undefined;
this zero-denominator case can arise only from the chosen calibration weights or whitening maps and does not override the zero-obstruction verdict branch.

\paragraph{Per-overlap norms.}
For each overlap $(i,j)$, we report the per-channel decomposition of $\omega_{ij}$. This per-overlap decomposition produces an overlap-level map that identifies which interfaces carry each kind of disagreement. The household-power cross-framework audit provides an example in which the winter--spring overlap is dominant. In general, the dominant overlap is the overlap maximizing $\norm{\omega_{ij}}_{2}^{2}$.

\paragraph{What the channel diagnosis tells us.}
State dominance means that the predicted state fields disagree on the overlap. This state-channel disagreement is the default mode for locally trained regression models that have not seen one another's regions. Closure dominance means that regional hidden-physics representations or closure parameters are inconsistent across regions. Closure dominance arises when one or more regions have an active conservation contract, as in Theorem~\ref{thm:conservation}, or when one region's closure was learned from data unavailable to the others. Observation dominance indicates disagreement among sensor readings or their summary statistics, and commonly reflects differences in sensor noise or drift. A nonzero optional metadata component indicates that declared contract records disagree across regions. That component is reported as metadata rather than treated as a causal intervention account; none of the reported experiments uses an optional-metadata intervention budget.

This channel-level diagnostic requires only matrix--vector products and no optimization. The resulting total and channel norms therefore provide suitable signals for the streaming monitoring described in Section~\ref{sec:monitor}.

\subsection{Budgeted intervention: exact tests and soft attribution}
\label{sec:budgeted}

The second component treats each declared account of the source of disagreement as an intervention budget and evaluates the cost of the minimum-cost intervention under that account. Hard feasibility determines which proper accounts survive. Costs are reported within each account but do not rank an unrestricted fallback against a more specific budget. The antitonicity result in Theorem~\ref{thm:intervention} explains this separation: Enlarging a budget under a shared metric cannot increase its minimum repair cost. The verdict of Definition~\ref{def:verdict} is then applied.

\begin{definition}[Standard budgets]
	\label{def:standard_budgets}
	The SEAM-$\Omega$ framework uses the following standard intervention budgets, each defined by an orthogonal projector $P$ onto a subspace of $C^{0}(F)$: The $\mathsf{AllowState}$ budget projects onto the state channel and permits correction of predicted state fields. The $\mathsf{AllowClosure}$ budget projects onto the closure channel and permits correction of hidden physics or sources. The $\mathsf{AllowSensorRejection}$ budget projects onto the observation channel and permits sensor readings to be reweighted or rejected. The $\mathsf{AllowAll}$ fallback uses $P=\Id$ and permits unrestricted intervention.
\end{definition}

For each budget, we solve the constrained quadratic program
\begin{equation}
	\label{eq:budgeted_intervention}
	\min_{\delta \in C^{0}(F)} \quad \onehalf \delta^{\top}C \delta \qquad \text{subject to}\qquad D P \delta = \omega,\quad (\Id-P)\delta = 0,
\end{equation}
where $C$ is a positive-definite cost metric. The repaired explanations are then
\begin{equation}
	\label{eq:sign_convention}
	s_{\mathrm{repaired}}  \coloneqq  s - P \delta^{\star},
\end{equation}
so that $D s_{\mathrm{repaired}} = D s - D P \delta^{\star} = \omega - \omega = 0$, recovering admissibility. We adopt this sign convention uniformly: $\delta^{\star}$ is the correction subtracted from the current stalk vector. The norm $\norm{\delta^{\star}}_{C}$ does not depend on the sign; only the interpretation does.

The closed-form solution is given by Theorem~\ref{thm:intervention}. When \eqref{eq:budgeted_intervention} is infeasible, so that $\omega \notin \im(DP)$, we report the hard cost $c_{P}^{\mathrm{hard}}=+\infty$ together with the hard residual $r_{P}^{\mathrm{hard}}(\omega)$ of Definition~\ref{def:budgeted_obstruction}, whose norm quantifies how badly the hypothesis fails.

\paragraph{Exact verdicts and auxiliary soft records.}
The system returns the per-budget record
\begin{align*}
	(\text{budget label}, \text{status}_{b}, c_{b}^{\mathrm{hard}}, r_{b}^{\mathrm{hard}}(\omega))
\end{align*}
and applies Definition~\ref{def:verdict} to the exact-feasible records. The verdict reports the full set of retained proper budgets and treats $\mathsf{AllowAll}$ only as a repairability fallback. Hard costs remain attached to their own budgets and are not combined into a cross-budget score. Auxiliary soft records are reported as specified in Definition~\ref{def:budgeted_obstruction} and always include their residuals, so their intervention magnitudes are not conflated with exact-hard verdicts. Figure~\ref{fig:budget_costs} shows such soft records for the data--physics conflict-attribution study of Section~\ref{sec:group3}.

\paragraph{Interpretation of the budget comparison.}
The budget catalog specifies the candidate diagnostic accounts. The closure-revision budget encodes the account that closure is the permitted locus of revision. The sensor-rejection budget analogously restricts revisions to observations. Other budgets encode further accounts, and the catalog is open: Theorem~\ref{thm:intervention} is stated for the orthogonal projector onto an arbitrary closed subspace $V_{P}\subseteq C^{0}(F)$, so any account that a practitioner can state as a set of permitted revisions enters the comparison as a projector, whether or not that set is one of the named channels. Exact feasibility retains or refutes each account separately. The minimum cost then measures the smallest repair within a retained account, rather than supplying a causal ranking across different allowed subspaces.

\paragraph{Refutation and unresolved attribution.}
Each entry in the catalog is separately refutable, which is what makes the comparison a test rather than a weighting. A budget whose hard problem is infeasible is ruled out by the overlap evidence rather than merely outbid, and its residual $\norm{r_{b}^{\mathrm{hard}}(\omega)}$ records how far the permitted revisions fall short of accounting for $\omega$. When every proper budget is ruled out and only the unrestricted budget remains feasible, that fallback carries no diagnostic specificity, because unrestricted revision repairs every raw defect by construction. The identifiability analysis in Section~\ref{sec:ident} then records which admissible directions remain invisible to the present observation map. In that situation, SEAM reports the unresolved subspace rather than assigning a unique interpretation to the unrestricted repair.

\begin{figure}[pos=t!]
	\centering
	\includegraphics[width=0.92\linewidth]{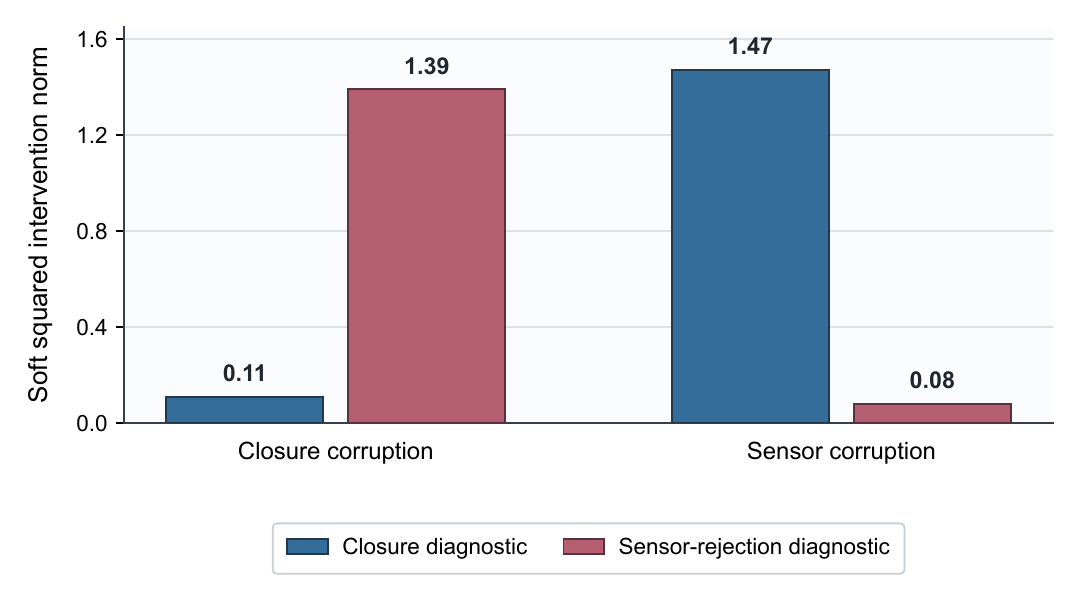}
	\caption{Soft budget-attribution ranking in the data--physics conflict-attribution study. The figure reports per-budget soft squared intervention norms in relative units for two configurations of the same three-region system. In the closure-corruption condition, the closure diagnostic is cheap, and the sensor-rejection diagnostic is expensive. In the sensor-corruption condition, the sensor-rejection diagnostic is cheap, and the closure diagnostic is expensive. For standard single-channel hard budgets under Assumption~A0, a finite hard cost requires full feasibility $\omega\in\im(DP_b)$; wrong-channel single-channel hard costs are $+\infty$ unless the non-target obstruction components vanish. The plotted ratios are therefore empirical soft diagnostics with residuals recorded in the diagnostic record, not theorem-level full hard-cost comparisons.}
	\label{fig:budget_costs}
\end{figure}

\subsection{Closure-restricted repair and missing-physics diagnosis}
\label{sec:missing_physics}

Learning or correcting unresolved closure terms is an established scientific machine learning problem, especially in data-driven turbulence modeling \citep{duraisamy2019turbulence}. Field inversion and machine learning provides a direct computational precedent by inferring spatially distributed functional corrections for deficient closure models \citep{parish2016field}. SEAM's closure-repair test determines whether adjusting the modeled closure removes an overlap obstruction, returning the corresponding minimum-cost reconstruction.

A feasible, low-cost closure-channel preimage supports the missing-closure hypothesis, and repair under the closure-revision budget yields the minimum-cost reconstruction with respect to the selected closure-channel metric. In the channel-separated model used here, the restricted repair solves the preimage problem for $\omega^{\mathrm{closure}}$; state mismatches require a state-revision or channel-coupled budget.

\Needspace{10\baselineskip} Let $\iota_{\mathrm{closure}}$ denote the zero-extension into $C^{0}(F)$, let $\Pi_{\bullet}$ denote the coordinate projection onto channel $\bullet$, and define the closure-channel block
\begin{align*}
	D_C\coloneqq \Pi_{\mathrm{closure}}D\iota_{\mathrm{closure}} \colon C^{0}(F)^{\mathrm{closure}} \to C^{1}(F)^{\mathrm{closure}},
\end{align*}
which is the whole closure-channel content of $D$ because Assumption~A0 gives $\Pi_{\bullet}D\iota_{\mathrm{closure}}=0$ for every $\bullet\ne\mathrm{closure}$. Corollary~\ref{thm:closure_recovery} supplies the minimum-norm closure correction $\delta_C^{\star}=D_C\pinv\omega^{\mathrm{closure}}\in\rowsp(D_C)$ whenever $\omega^{\mathrm{closure}}\in\im D_C$, together with the repaired explanations $(u_i,c_i-\delta_C^{\star}|_i,o_i,\epsilon_i)_{i\in I}$, whose closure-channel defect vanishes, while every non-closure component of $Ds$ is unchanged. For a raw SEAM-$\Omega$ defect $\omega=Ds$, that image condition holds automatically, since $\omega^{\mathrm{closure}}=D_C s_C$ for the closure-channel component $s_C$ of $s$, so the closure repair needs no separate image-membership test.

\begin{remark}[Physical sign and weighted closure costs]
	\label{rem:closure_sign_weighted}
	The vector $\delta_C^{\star}$ is the algebraic correction subtracted from the current closure block. Under $s_{\mathrm{repaired}} =s-\iota_{\mathrm{closure}}\delta_C^{\star}$, the physical closure increment added to the model is therefore $\widehat c_{\mathrm{missing}}\coloneqq-\delta_C^{\star}$. This increment has a physical missing-source interpretation when the closure block faithfully represents the source feature and is identifiable modulo $\ker D_C$ under the current observation map. For a non-identity closure cost $C_C\succ0$, the minimum-cost correction is
	\begin{align*}
		C_C^{-1}D_C^{\top} (D_C C_C^{-1}D_C^{\top})\pinv\omega^{\mathrm{closure}},
	\end{align*}
	which reduces to the Moore--Penrose formula in the Euclidean metric.
\end{remark}

Corollary~\ref{thm:closure_recovery} is stated in Section~\ref{sec:thm_closure_recovery} and proved in Appendix~\ref{app:thm_closure_recovery}. The hidden-source recovery for Burgers study in Section~\ref{sec:hidden_source_recovery} documents the empirical content. For inviscid Burgers with hidden closure summaries, the recovered physical closure increment $-\delta_{C}^{\star}$ has its peak at $\hat{x}_{0} = 0.501$, within $0.2\%$ of the true center $x_{0} = 0.5$ and displaced by $1.7\%$ of the source width $w = 0.06$. The projected recovery has a closure-channel repair residual below $5\%$; the exact image-membership case of Corollary~\ref{thm:closure_recovery} has zero residual. The accompanying state-revision result is a soft attribution record with its residual relative to the same closure-channel right-hand side.

\section{Identifiability}
\label{sec:ident}

Diagnosis tells us what disagrees. Identifiability analysis then determines which admissibility-preserving adjustments are distinguishable by current observations. This distinction separates detected inconsistency from directions that remain unresolved by the available data.

\subsection{Blind admissible directions and the observable quotient}
\label{sec:blind_observable}

Throughout this section, we fix a linear observation map $O\colon C^{0}(F) \to Y$ with $Y$ a finite-dimensional output space. The observation map encodes how sensors see the global section $s\in C^{0}(F)$. Typically, $O$ has a block-diagonal structure that extracts the observation-channel components of the explanations.

The kernel $\ker D$ of the coboundary is the subspace of admissible directions whose restrictions agree on every overlap, thereby leaving the overlap-disagreement vector unchanged. Within $\ker D$, we distinguish two operationally important subspaces.

\begin{definition}[Blind admissible subspace and observable
		admissible quotient]
	\label{def:blind_observable}
	The blind admissible subspace is
	\begin{align*}
		\Nblind  \coloneqq  \ker D  \cap  \ker O  \subseteq  C^{0}(F).
	\end{align*}
	When the dependence of $\Nblind$ on the observation map matters, we write this subspace as $\Nblind(O)$. Directions in $\Nblind$ preserve admissibility and are invisible to the current observation map. The observable admissible quotient is
	\begin{align*}
		\Iobs  \coloneqq  \ker D  /  \Nblind \cong (\ker D)  \cap  \Nblind^{\perp}.
	\end{align*}
	The isomorphism uses the orthogonal complement inside $\ker D$. Directions in $\Iobs$ preserve admissibility and are visible to current observations.
\end{definition}

\begin{remark}[Observable and blind admissible directions]
	Together, $\Nblind$ and $\Iobs$ describe the current observational resolution within the admissible space. The subspace $\Nblind$ contains the unresolved directions, while $\Iobs$ parameterizes the resolved classes.
\end{remark}

The two subspaces have complementary dimensions. Proposition~\ref{thm:identifiability} records the identity
\begin{equation}
	\label{eq:ident_dimensions}
	\dim \Nblind = \dim \ker D  -  \rank(O|_{\ker D}), \qquad \dim \Iobs = \rank(O|_{\ker D}),
\end{equation}
so that $\dim \ker D = \dim \Nblind + \dim \Iobs$. The proof is a direct application of rank--nullity to $O|_{\ker D}$ with $\ker(O|_{\ker D}) = \ker D \cap \ker O = \Nblind$; see Appendix~\ref{app:thm_identifiability}.

\paragraph{Computing $\Nblind$.}
The numerical procedure begins by building an orthonormal basis $V_{\ker}$ of $\ker D$ with a full singular value decomposition (SVD) or an explicit rank-revealing nullspace routine, ensuring that structural right-null directions are retained. The procedure then forms the restricted operator $O V_{\ker}$ and computes either its full right-nullspace basis or a full SVD\@. All right singular directions whose singular values lie below a tolerance $\eta_{\mathrm{svd}}$ are selected, including structural right-null directions omitted by an economy decomposition. Finally, $V_{\ker}$ maps these vectors back to $C^{0}(F)$ to obtain an orthonormal basis $B_{\Nblind}$ of $\Nblind$. The certificate records several checks:
\begin{align*}
	\norm{D B_{\Nblind}}\leq \eta_D\norm{B_{\Nblind}},\qquad \norm{O B_{\Nblind}}\leq \eta_O\norm{B_{\Nblind}},\qquad \norm{B_{\Nblind}^{\top}B_{\Nblind}-\Id}\leq \eta_{\mathrm{orth}},
\end{align*}
together with the dimension identity in Proposition~\ref{thm:identifiability}. The blind-closure identifiability study reports the observation-invisibility component $\norm{O B_{\Nblind}}<10^{-14}$.

\paragraph{What $\Nblind$ tells us.}
The blind admissible subspace specifies the resolution of the SEAM diagnostic. A nonzero $\dim\Nblind$ means that there exist directions in which the local explanations can be adjusted without violating admissibility and without changing any current observation. The identifiability report marks these directions as unresolved by the present data, qualifying any attribution accordingly.

\section{Monitoring learned generators}
\label{sec:monitor}

The diagnostic of Sections~\ref{sec:diag} and~\ref{sec:ident} treats the local generators as static. For learned generators such as neural operators, PINNs, regression models, and hybrid scientific pipelines, the family of local explanations depends on input data that may shift over time. SEAM-$\Omega$ extends naturally to streaming monitoring of these systems.

Let $G_{\theta}$ be a trained operator with parameters $\theta$, and let $\{x_{t}\}_{t\geq 0}$ be a streaming input sequence. Examples include time series of PDE initial conditions, sensor data, and regulatory inputs. At each time $t$, we apply to $x_{t}$ the family of local generators $\{G_{i,\theta}\}_{i\in I}$ obtained by restricting the operator to each region, thereby producing a family of local explanations $e_{i}(t) = G_{i,\theta}(x_{t})$.

\begin{definition}[Streaming obstruction]
	\label{def:streaming_omega}
	The stacked explanation and streaming obstruction at time $t$ are
	\begin{align*}
		s(t)      & \coloneqq (\mathrm{vec}(e_i(t)))_{i\in I}\in C^0(F), \\
		\omega(t) & \coloneqq D s(t)\in C^1(F).
	\end{align*}
\end{definition}

The streaming obstruction is a time-indexed family of vectors in $C^{1}(F)$. At each time, the report contains the total disagreement $\norm{\omega(t)}$ and the norm $\norm{\omega^{\bullet}(t)}$ for each channel. These quantities can be compared with contemporaneous model error or with a baseline recorded during training, but SEAM does not assign a universal alarm threshold.

\paragraph{Example: neural operator under covariate shift.}
Mechanism~M4 appears here in streaming form. For the Fourier neural operator (FNO) architecture \citep{li2021fourier}, trained on a distribution of Burgers initial conditions and evaluated on out-of-distribution (OOD) inputs, the FNO OOD monitoring study in Section~\ref{sec:fno_corr} shows a correlation between $\norm{\omega(t)}$ and the $L^{2}$ prediction error relative to a high-resolution reference. This correlation is the empirical basis for using $\omega$ as a consistency monitor. Channel norms localize which part of the regional record changes, but the experiment does not establish a calibrated probabilistic detector or a universal causal interpretation for that change.

\section{Theoretical results}
\label{sec:theory}

This section collects two central theorems together with companion identifiability and closure-recovery results, and states their shared model conditions. Proofs are consolidated in Appendix~\ref{app:proofs} in the order stated here. Throughout, $D$ denotes the assembled coboundary matrix of an explanation sheaf $F$ over the 1-skeleton of a finite cover's nerve, $D\pinv$ its Moore--Penrose pseudoinverse and associated projector construction \citep{penrose1955generalized}, $P$ an orthogonal projector onto a budget subspace of $C^{0}(F)$, and $C$ a positive-definite cost metric on $C^{0}(F)$. We write $DP$ for the projected coboundary and let $\im(DP)$ denote its image. Observation maps $O\colon C^{0}(F)\to Y$ are linear.

\paragraph{Model conditions and conventions for Section~\ref{sec:theory}.}
The result suite rests on the standing structural Assumption~A0 together with the four repair and model conventions declared below.
\paragraph{A0: block-diagonal channel restrictions.}
The channel-resolved results use Assumption~A0 from Section~\ref{sec:channel_decomp}, where the assumption is adopted for every restriction map of SEAM-$\Omega$.

\paragraph{A1: repair sign convention.}
Corrections are subtracted from the current stalk vector. For a full hard repair, the convention reads $s_{\mathrm{repaired}} = s - P \delta^{\star}$ with $DP \delta^{\star} = \omega$. Corollary~\ref{thm:closure_recovery} uses the channel-restricted version $s_{\mathrm{repaired}}=s-\iota_{\mathrm{closure}}\delta_C^\star$ with $D_C\delta_C^\star=\omega^{\mathrm{closure}}$. This convention is used by Theorem~\ref{thm:intervention}, Corollary~\ref{thm:closure_recovery}, and the repair records of Sections~\ref{sec:diag} and~\ref{sec:missing_physics}.

\paragraph{A2: positive-definite cost metric.}
The condition $C\succ 0$ holds on $C^{0}(F)$, giving the strict convexity required by Theorem~\ref{thm:intervention}.

\paragraph{A3: row-space-aligned one-sided closure repair.}
In Theorem~\ref{thm:conservation}, the conservation-redistribution repair direction at the repaired endpoint lies in the domain-side row space of that endpoint's closure restriction; the baseline overlap closures cancel; the neighboring endpoint contributes zero restricted repair on the overlap, $\rho^{\mathrm{closure}}_{\bar q,e}\delta c_{\bar q}=0$; and $\rho^{\mathrm{closure}}_{q,e}$ has positive rank. The projected form covers repair directions outside the row space; rank-zero restrictions carry no closure signal on the overlap.

\paragraph{A4: pairwise-overlap construction.}
All results use the 1-skeleton of the nerve.

Table~\ref{tab:theorems} summarizes the result suite and its roles in SEAM-$\Omega$ before the individual theorem statements. For Theorem~\ref{thm:conservation} and its row in the table, $\sigma_{\min}^{+}$ denotes the smallest nonzero singular value; the distinction is necessary because the ordinary smallest singular value $\sigma_{\min}$ is zero for a rank-deficient operator.

\begin{table}[pos=t!]
	\centering
	\caption{Central results of SEAM-$\Omega$ and their operational roles.}
	\label{tab:theorems}
	\small
	\begin{tabularx}{\linewidth}{ l l >{\raggedright\arraybackslash}X}
		\toprule
		Result      & Subject                             & Role                                                                                                                                                                                                                                                                                                            \\
		\midrule
		Theorem 1   & Minimum-cost budgeted intervention  & Closed-form $\delta^{\star}$ via restricted operators $A_{P}=DP\iota_{V_{P}}$, $C_{P}=\iota_{V_{P}}^{\top}C\iota_{V_{P}}$, with $s_{\mathrm{repaired}} = s - P\delta^{\star}$.                                                                                                                                  \\
		\addlinespace[1pt]
		Proposition & Identifiability dimensions          & $\dim\Nblind = \dim\ker D - \rank(O|_{\ker D})$, while $\dim\Iobs = \rank(O|_{\ker D})$.                                                                                                                                                                                                                        \\
		\addlinespace[1pt]
		Corollary   & Closure-restricted recoverability   & $\delta_{C}^{\star} = D_{C}\pinv\omega^{\mathrm{closure}} \in\im(D_{C}^{\top})=\rowsp(D_C)$ when $\omega^{\mathrm{closure}}\in\im D_{C}$. The repaired closure blocks are $c_{i}-\delta_{C}^{\star}|_{i}$.                                                                                                      \\
		\addlinespace[1pt]
		Theorem 2   & Conservation-contract detectability & Under A3, comprising baseline cancellation, one-sided zero neighbor contribution, positive rank, and row-space alignment $\delta c_q\in\rowsp(\rho^{\mathrm{closure}}_{q,e})$, $\norm{\omega^{\mathrm{closure}}_{e}} \geq \sigma_{\min}^{+} \epsilon_{\mathrm{repair}}$. The projected bound applies otherwise. \\
		\bottomrule
	\end{tabularx}
\end{table}
\FloatBarrier

\subsection{Theorem 1: minimum-cost budgeted intervention}
\label{sec:thm_intervention}

\begin{seamtheorem}[Minimum-cost budgeted intervention]
	\label{thm:intervention}
	Let $P$ be the orthogonal projector onto a closed subspace $V_{P}\subseteq C^{0}(F)$ called the budget subspace, and let $C\succ 0$ on $C^{0}(F)$. Define the restricted operators
	\begin{align*}
		A_{P}  \coloneqq  D P \iota_{V_{P}}\colon V_{P}\to C^{1}(F), \qquad C_{P}  \coloneqq  \iota_{V_{P}}^{\top} C \iota_{V_{P}} \succ 0 \text{ on } V_{P},
	\end{align*}
	where $\iota_{V_{P}}\colon V_{P}\hookrightarrow C^{0}(F)$ is the inclusion. Consider the constrained quadratic program
	\begin{equation}
		\label{eq:thm_intervention_problem}
		\min_{\delta \in C^{0}(F)}  \onehalf \delta^{\top} C \delta \quad\text{subject to}\quad D P \delta = \omega, (\Id-P) \delta = 0.
	\end{equation}

	\medskip\noindent\textup{\textbf{Feasibility criterion.}}
	The problem is feasible if and only if $\omega \in \im A_{P}$.

	\medskip\noindent\textup{\textbf{Closed-form solution.}}
	When feasible, the unique minimizer is
	\begin{equation}
		\label{eq:thm_intervention_closed}
		\delta^{\star} = \iota_{V_{P}} C_{P}^{-1} A_{P}^{\top} ( A_{P} C_{P}^{-1} A_{P}^{\top} )\pinv \omega,
	\end{equation}
	and the reported squared intervention cost is $\norm{\delta^{\star}}_{C}^{2} = \omega^{\top}(A_{P}C_{P}^{-1}A_{P}^{\top})\pinv\omega$. The optimum of the objective $\onehalf \delta^{\top}C\delta$ in \eqref{eq:thm_intervention_problem} is one half of this value.

	\medskip\noindent\textup{\textbf{Soft Tikhonov form.}}
	For $\lambda > 0$, the unique minimizer of $\onehalf \delta^{\top}C\delta + (\lambda/2)\norm{D\delta - \omega}^{2}$ over $\delta \in V_{P}$ is
	\begin{equation}
		\label{eq:thm_intervention_soft}
		\delta^{\star}_{\lambda} = \iota_{V_{P}} (A_{P}^{\top}A_{P} + \lambda^{-1} C_{P})^{-1} A_{P}^{\top} \omega,
	\end{equation}
	which is well-defined for all $\lambda \in (0, \infty)$.

	\medskip\noindent\textup{\textbf{Antitonicity in budget.}}
	For a fixed ambient cost metric $C$, suppose that $V_{P_{1}} \subseteq V_{P_{2}}$, meaning that the larger budget contains the smaller. Then, for every $\omega$ for which both problems are feasible under that same $C$, the optimal costs satisfy $\norm{\delta^{\star}(P_{2})}_{C}^{2} \leq \norm{\delta^{\star}(P_{1})}_{C}^{2}$. Budget antitonicity therefore requires a shared ambient cost metric;
	budget inclusion alone is insufficient when the two budgets use different budget-specific metrics.

	\medskip\noindent\textup{\textbf{Sign convention.}}
	The repaired stalk vector is $s_{\mathrm{repaired}} = s - P\delta^{\star}$ with $\delta^{\star}\in V_{P}$, so that
	\begin{align*}
		D s_{\mathrm{repaired}} = \omega-DP\delta^{\star} = 0.
	\end{align*}
	The $V_{P}^{\perp}$ component of the stalk vector is unchanged. If $P$ is a channel projector, non-target channels are unchanged under Assumption~A0.
\end{seamtheorem}

\begin{remark}[Restricted inverse and commutativity]
	\label{rem:restricted_inverse}
	The identity $(PCP)\pinv = PC^{-1}P$ is valid in the commutative case $[P,C]=0$ but not for a general metric coupling $V_{P}$ and $V_{P}^{\perp}$. Restricting the metric to $V_P$ gives the compressed inverse
	\begin{align*}
		\iota_{V_P} (\iota_{V_P}^{\top}C\iota_{V_P})^{-1} \iota_{V_P}^{\top},
	\end{align*}
	not generally $PC^{-1}P$. The formulation above with $A_{P}$ and $C_{P}$ avoids the commutativity assumption entirely, and the reference implementation uses this restricted form. For $C = \Id$, as adopted in all experiments, the restricted and ambient expressions coincide.
\end{remark}

\subsection{Identifiability dimensions}
\label{sec:thm_identifiability}

\begin{proposition}[Identifiability dimensions]
	\label{thm:identifiability}
	Let $D\colon C^{0}(F) \to C^{1}(F)$ be the coboundary and let $O\colon C^{0}(F) \to Y$ be a linear observation map. Define
	\begin{align*}
		\Nblind  \coloneqq  \ker D \cap \ker O, \qquad \Iobs  \coloneqq  \ker D / \Nblind.
	\end{align*}
	Then the blind-space dimension is $\dim\Nblind=\dim\ker D-\rank(O|_{\ker D})$, the observable quotient has dimension $\dim\Iobs=\rank(O|_{\ker D})$, and the two dimensions sum to $\dim\ker D=\dim\Nblind+\dim\Iobs$. Moreover, $\Nblind$ is the set of admissibility-preserving directions that are also invisible to $O$, while $\Iobs$ parameterizes the admissibility-preserving directions that $O$ can resolve.
\end{proposition}

\subsection{Closure-restricted recoverability}
\label{sec:thm_closure_recovery}

\begin{corollary}[Closure-restricted recoverability]
	\label{thm:closure_recovery}
	Let $D_{C}$ be the closure-channel block of $D$ under Assumption~A0, \ie, the typed map $D_C\colon C^{0}(F)^{\mathrm{closure}}\to C^{1}(F)^{\mathrm{closure}}$ obtained by restricting the domain to the closure channel and projecting the codomain onto that channel. Suppose $\omega^{\mathrm{closure}} \in \im D_{C}$. Then the unique minimum-norm closure correction in the Euclidean metric is
	\begin{align*}
		\delta_{C}^{\star}  =  D_{C}\pinv \omega^{\mathrm{closure}} \in  \im(D_{C}^{\top}) =  \rowsp(D_C) =  (\ker D_{C})^{\perp}.
	\end{align*}
	The repaired stalks $(u_{i},  c_{i} - \delta_{C}^{\star}|_{i},  o_{i}, \epsilon_{i})_{i\in I}$ have zero closure-channel admissibility defect, and the non-closure channels of $\omega$ are unchanged. Under the sign convention, the zero-extended correction is subtracted consistently with the channel-restricted convention in A1. The recovered $\delta_{C}^{\star}$ is the unique element of $\im(D_C^\top)=\rowsp(D_C)$ that achieves closure-channel admissibility, but adding any element of $\ker D_{C}$ leaves the closure obstruction unchanged. Such a closure direction belongs to $\Nblind$ only after embedding that direction into the full stalk space and checking membership in $\ker D\cap\ker O$ according to Proposition~\ref{thm:identifiability}. For raw admissibility defects $\omega=Ds$, Assumption~A0 makes $\omega^{\mathrm{closure}}\in\im D_C$ automatic.
\end{corollary}

\subsection{Theorem 2: conservation-contract detectability}
\label{sec:thm_conservation}

\begin{seamtheorem}[Conservation-contract detectability]
	\label{thm:conservation}
	Let $F$ be an explanation sheaf with closure-channel restrictions $\rho^{\mathrm{closure}}_{q,e}$ under Assumption~A0. Let $e=(a,b)$ be an oriented overlap edge and let $q\in\{a,b\}$ be the repaired endpoint. Suppose $U_q$ applies a conservation-contract repair whose actual closure-channel vector entering the restriction is $\delta c_q$, and set $\epsilon_{\mathrm{repair}}\coloneqq\norm{\delta c_q}$.

	Assume the following conditions, which together instantiate A3:

	\medskip\noindent\textup{\textbf{Baseline cancellation.}}
	The baseline closure restrictions agree on $e$ before repair: $\rho^{\mathrm{closure}}_{a,e}c_{a}^{0} = \rho^{\mathrm{closure}}_{b,e}c_{b}^{0}$.

	\medskip\noindent\textup{\textbf{Row-space alignment.}}
	The repair direction $\delta c_{q}$ lies in $\rowsp(\rho^{\mathrm{closure}}_{q,e})$.

	\medskip\noindent\textup{\textbf{One-sided zero contribution.}}
	The other endpoint's restricted repair contribution on $e$ is zero: $\rho^{\mathrm{closure}}_{\bar q,e}\delta c_{\bar q}=0$, where $\bar q$ denotes the endpoint other than $q$.

	\medskip\noindent\textup{\textbf{Positive-rank restriction.}}
	The map $\rho^{\mathrm{closure}}_{q,e}$ has positive rank.

	\medskip\noindent\textup{\textbf{Conclusion.}}
	Let $\varepsilon_a\coloneqq-1$ and $\varepsilon_b\coloneqq+1$. Then, under baseline cancellation and the one-sided zero-contribution condition, the restricted defect and its closure-channel lower bound are
	\begin{equation}
		\label{eq:thm_conservation}
		\omega^{\mathrm{closure}}_{e} = \varepsilon_q \rho^{\mathrm{closure}}_{q,e}\delta c_q, \qquad \norm{\omega^{\mathrm{closure}}_{e}} \geq \sigma_{\min}^{+}(\rho^{\mathrm{closure}}_{q,e}) \cdot \epsilon_{\mathrm{repair}}.
	\end{equation}

	\medskip\noindent\textup{\textbf{Projected variant without row-space alignment.}}
	If baseline cancellation, one-sided zero contribution, and the positive-rank restriction hold but row-space alignment fails, the same argument gives the projected bound
	\begin{equation}
		\label{eq:thm_conservation_projected}
		\norm{\omega^{\mathrm{closure}}_{e}} \geq \sigma_{\min}^{+}(\rho^{\mathrm{closure}}_{q,e}) \cdot \norm{\Pi_{\rowsp(\rho^{\mathrm{closure}}_{q,e})} \delta c_{q}},
	\end{equation}
	where $\Pi_{\rowsp(\rho^{\mathrm{closure}}_{q,e})}$ denotes the orthogonal projector onto the indicated row space. This bound equals \eqref{eq:thm_conservation} under row-space alignment and is strictly weaker otherwise. If both endpoints are repaired, a positive lower bound in terms of $\norm{\delta c_q}$ additionally requires a quantitative non-cancellation assumption on the net restricted repair.

	\medskip\noindent\textup{\textbf{Rank-zero convention.}}
	If $\rank(\rho^{\mathrm{closure}}_{q,e})=0$, then $\rho^{\mathrm{closure}}_{q,e}$ carries no closure signal on the overlap: The projected row-space term is zero, and $\sigma_{\min}^{+}(\rho^{\mathrm{closure}}_{q,e})$ is not defined.
\end{seamtheorem}

\begin{remark}[How CT-SEAM satisfies the row-space condition]
	\label{rem:conservation_row_space}
	The conservation-contract repair operator in CT-SEAM, described in Appendix~\ref{app:ctseam}, applies a uniform scalar correction over the 1D closure coordinate used by the overlap restriction. For that experimental closure block, the domain-side repair vector is assumed to lie in $\rowsp(\rho^{\mathrm{closure}}_{q,e})$. In more general closure spaces, this row-space membership must be checked explicitly; otherwise, the projected bound \eqref{eq:thm_conservation_projected} applies.
\end{remark}

\paragraph{Empirical verification.}
The conservation-contract detectability study in Section~\ref{sec:conservation_contract_detectability} verifies the bound on a three-region Burgers configuration with an active conservation repair: The observed $\norm{\omega^{\mathrm{closure}}_{12}} = 1.172$ exceeds the predicted lower bound of $\sigma_{\min}^{+} \epsilon_{\mathrm{repair}} = 0.128$ by a factor of $9.2$. The associated parametric sweep reports a pass for all 25 configurations under the prespecified $10\%$ discretization-tolerant criterion and exhibits the predicted approximately linear scaling with repair magnitude.

\paragraph{Sensitivity floor.}
The bound is tight when the repaired endpoint's closure vector $\delta c_{q}$ aligns with the right singular vector of $\rho^{\mathrm{closure}}_{q,e}$ corresponding to the smallest nonzero singular value. The bound is loose when $\delta c_{q}$ aligns with the dominant singular vectors. The sensitivity floor $\sigma_{\min}^{+}$ is therefore a worst-case detection floor; in practice, generic repair directions yield much larger closure obstructions. The slack factor is a useful diagnostic in its own right, because a small slack indicates a near-minimum-detection repair direction.

\section{Algorithms and computational workflow}
\label{sec:algos}

Here we present the core algorithms in mathematically clean form;
Appendix~\ref{app:algorithms} records the numerical conventions used by the reference implementation. The presentation groups the core diagnostic workflow and its extensions before treating the computational and audit properties of the resulting implementation.

\subsection{Diagnostic workflow and extensions}
\label{sec:diagnostic_intervention_workflow}

\paragraph{\texorpdfstring{Top-level SEAM-$\Omega$ diagnostic pipeline.}{Top-level SEAM-Omega diagnostic pipeline.}}
\label{sec:top_level_pipeline}

The top-level diagnostic pipeline is Algorithm~\ref{alg:diagnose}. The index set $\mathcal K_{\mathrm{bud}}$ labels the finite budget catalog. The pipeline returns a structured diagnostic record.

\begin{algorithm}[ht!]
	\caption{SEAM-$\Omega$ diagnostic pipeline, \texttt{diagnose}.}
	\label{alg:diagnose}
	\begin{algorithmic}[1]
		\Require cover $\cover$, local generators $\{G_{i}\}_{i\in I}$, input data $x$, restriction maps $\{\rho_{i,ij}, \rho_{j,ij}\}_{(i,j)\in E}$, observation map $O$, budget catalog $\{(P_{b}, C_{b})\}_{b\in\mathcal K_{\mathrm{bud}}}$, zero-obstruction tolerance $\tau_{\mathrm{zero}}$.
		\Ensure $\omega$, channel norms, per-budget hard records, tolerance and soft diagnostic records, identifiability basis, diagnostic record.
		\State Generate: for each $i\in I$, compute $e_{i} \gets G_{i}(x,U_{i})$ and $s_{i} \gets \mathrm{vec}(e_{i})$;
		set $s\gets(s_i)_{i\in I}$.
		\State Restrict and obstruct: assemble $D$ from the restriction maps; compute $\omega \gets D s$.
		\State Decompose channels: compute $\norm{\omega^{\bullet}}$ for $\bullet\in \{\mathrm{state},\mathrm{closure},\mathrm{obs},\mathrm{meta}\}$.
		\State Apply budgeted intervention for each budget $b$:
		\Statex \quad solve \eqref{eq:thm_intervention_problem} to obtain $\delta^{\star}_{b}$ or detect infeasibility.
		\Statex \quad record exact feasibility status;
		$c^{\mathrm{hard}}_{b}$, which is finite only for $\mathsf{exact\_feasible}$;
		optional tolerance cost $c^{\mathrm{tol}}_{b}$;
		optional soft squared intervention norm $n^{\mathrm{soft}}_{b,\lambda}$ and objective $J^{\mathrm{soft}}_{b,\lambda}$; hard residual $r_{b}^{\mathrm{hard}}(\omega)$; and any soft residual $r_{b,\lambda}^{\mathrm{soft}}(\omega)$.
		\State Analyze identifiability: compute the orthonormal basis of $\Nblind$ via Proposition~\ref{thm:identifiability}.
		\State Build the diagnostic record: bundle the obstruction, channel norms, budget records, and $\Nblind$ certificate.
		\State \Return diagnostic record.
	\end{algorithmic}
\end{algorithm}

\paragraph{Budgeted intervention solver.}
\label{sec:budget_solver}

Algorithm~\ref{alg:intervene} implements the minimum-cost intervention result of Theorem~\ref{thm:intervention} in closed form and couples that result to an explicit feasibility check. After restricting the coboundary and cost metric to the budget subspace $V_P$, the algorithm uses the hard residual to distinguish exact-feasible and tolerance-feasible repairs from infeasible cases, which are passed to the soft Tikhonov diagnostic.

\begin{algorithm}[ht!]
	\caption{Budgeted intervention solver, \texttt{synthesize\_intervention}.}
	\label{alg:intervene}
	\begin{algorithmic}[1]
		\Require $D$, $\omega$, projector $P$, cost metric $C$, feasibility tolerance $\eta_{\mathrm{feas}} > 0$, Tikhonov regularizer $\lambda > 0$ for soft mode.
		\Ensure $\delta^{\star}$, hard cost, tolerance record, soft diagnostic record, $r_{P}^{\mathrm{hard}}(\omega)$, optional $r_{P,\lambda}^{\mathrm{soft}}(\omega)$, feasibility status.
		\State Let $Q$ be an orthonormal basis and inclusion matrix for $V_P=\im P$.
		\State Form $A_{P} \gets D Q$ and $C_{P} \gets Q^{\top} C Q$.
		\State Compute $A_{P}\pinv$ and $r_{P}^{\mathrm{hard}}(\omega) \gets \omega - A_{P} A_{P}\pinv\omega$.
		\If{$\norm{r_{P}^{\mathrm{hard}}(\omega)} = 0$}
		\State Enter the theorem-level exact-feasible branch.
		\State $z^{\star} \gets C_{P}^{-1}A_{P}^{\top}(A_{P}C_{P}^{-1}A_{P}^{\top})\pinv\omega$.
		\State $\delta^{\star} \gets Q z^{\star}$.
		\State $c^{\mathrm{hard}} \gets {\delta^{\star}}^{\top} C \delta^{\star}$.
		\State \Return $(\delta^{\star}, c^{\mathrm{hard}}, r_{P}^{\mathrm{hard}}(\omega), \mathsf{exact\_feasible})$.
		\ElsIf{$\norm{r_{P}^{\mathrm{hard}}(\omega)} \leq \eta_{\mathrm{feas}}$}
		\State Enter the tolerance-feasible engineering branch.
		\State Compute $z^{\star}$ and $\delta^{\star}$ as above.
		\State $c^{\mathrm{tol}} \gets {\delta^{\star}}^{\top} C \delta^{\star}$.
		\State $c^{\mathrm{hard}}\gets +\infty$ for exact-hard verdicts.
		\State \Return $(\delta^{\star}, (c^{\mathrm{hard}},c^{\mathrm{tol}}), r_{P}^{\mathrm{hard}}(\omega), \mathsf{tolerance\_feasible})$.
		\Else
		\State Enter the soft Tikhonov branch.
		\State $z_{\lambda}^{\star}\gets (A_{P}^{\top}A_{P}+\lambda^{-1}C_{P})^{-1}A_{P}^{\top}\omega$.
		\State $\delta^{\star}_{\lambda} \gets Q z_{\lambda}^{\star}$.
		\State $r_{P,\lambda}^{\mathrm{soft}}(\omega) \gets \omega - A_P z_\lambda^{\star}$.
		\State $c^{\mathrm{hard}}\gets +\infty$.
		\State $n^{\mathrm{soft}}_{\lambda} \gets \norm{\delta^{\star}_{\lambda}}_{C}^{2} = {\delta^{\star}_{\lambda}}^{\top} C \delta^{\star}_{\lambda}$.
		\State $J^{\mathrm{soft}}_{\lambda} \gets \onehalf n^{\mathrm{soft}}_{\lambda} +(\lambda/2)\norm{r_{P,\lambda}^{\mathrm{soft}}(\omega)}^{2}$.
		\State \Return $(\delta^{\star}_{\lambda}, (c^{\mathrm{hard}},n^{\mathrm{soft}}_{\lambda}, J^{\mathrm{soft}}_{\lambda}), (r_{P}^{\mathrm{hard}}(\omega), r_{P,\lambda}^{\mathrm{soft}}(\omega)), \mathsf{infeasible\_soft\_diagnostic})$.
		\EndIf
	\end{algorithmic}
\end{algorithm}

Two implementation notes apply. First, pseudoinverses are computed via a project-specific truncated SVD tolerance, stated uniformly in Appendix~\ref{app:pinv_impl}; library defaults are not assumed, and the exact routine and version are recorded whenever an implementation relies on one. Second, the matrices $A_{P}$ and $C_P$ are typically small, so dense linear algebra is appropriate. For larger problems, exact pseudoinverse projection with a truncated or randomized SVD requires retaining every singular direction above the chosen rank tolerance; smaller retained ranks yield approximate low-rank diagnostics.

\paragraph{Identifiability.}
\label{sec:ident_alg}

Algorithm~\ref{alg:ident} implements the two-stage nullspace construction and returns its numerical certificate.

\begin{algorithm}[ht!]
	\caption{Identifiability analyzer, \texttt{analyze\_identifiability}.}
	\label{alg:ident}
	\begin{algorithmic}[1]
		\Require $D\in\R^{M\times N}$, $O$, singular-value tolerance $\eta_{\mathrm{svd}} > 0$.
		\Ensure orthonormal basis $B_{\Nblind}$ of $\Nblind$, $\dim\Nblind$, channel decomposition of $B_{\Nblind}$, multi-part certificate.
		\State Compute a full SVD or rank-revealing nullspace of $D$: $D = U_{D}\Sigma_{D}V_{D}^{\top}$.
		\State Let $r=\#\{\ell : \sigma_\ell(D)>\eta_{\mathrm{svd}}\}$ and set $V_{\ker}\gets V_{D}[:, r+1:N]$, including structural zero singular directions.
		\State Compute $O' \gets O V_{\ker}$ and its SVD $O' = U_{O'}\Sigma_{O'}V_{O'}^{\top}$.
		\State Let $V_{0}$ be a complete basis of $\ker_{\eta_{\mathrm{svd}}}(O')$, including structural right-null directions omitted by an economy SVD, and set $B_{\Nblind}\gets V_{\ker} V_{0}$.
		\State Decompose each column of $B_{\Nblind}$ into channel components by projecting onto the three primary channels and any optional contract-metadata block.
		\State Verify $\norm{D B_{\Nblind}}$, $\norm{O B_{\Nblind}}$, orthonormality, and the dimension identity of Proposition~\ref{thm:identifiability}.
		\State \Return $(B_{\Nblind}, \dim\Nblind, \text{channel decomposition}, \text{certificate})$.
	\end{algorithmic}
\end{algorithm}

\paragraph{Streaming monitoring.}
\label{sec:monitor_alg}

The monitoring pipeline runs the diagnostic pipeline at each time step of a streaming input but reuses cached operators where possible.

\begin{algorithm}[ht!]
	\caption{Streaming monitoring.}
	\label{alg:monitor}
	\begin{algorithmic}[1]
		\Require fixed cover $\cover$, fixed restrictions, fixed budgets, learned generator $G_{\theta}$, streaming inputs $\{x_{t}\}$.
		\Ensure stream of $(\omega(t), \text{channel norms}(t), \text{hard and soft diagnostic budget records}(t))$.
		\State Assemble $D$ once.
		\State Cache pseudoinverses and projectors for all standard budgets.
		\For{$t = 1, 2, \ldots$}
		\State Compute $e_{i}(t) \gets G_{i,\theta}(x_{t})$ for each $i$.
		\State Set $s(t)\gets(\mathrm{vec}(e_i(t)))_{i\in I}$ and compute $\omega(t) \gets D s(t)$.
		\State Compute channel norms and hard and soft diagnostic budget records.
		\State Emit summary.
		\EndFor
	\end{algorithmic}
\end{algorithm}

\subsection{Computational complexity and auditability}
\label{sec:complexity}

We summarize the time complexity in terms of the total stalk dimension $N \coloneqq \sum_{i} d_{i}$, the total overlap stalk dimension $M \coloneqq \sum_{(i,j)} m_{ij}$, the number of regions $|I|$, and the number of budgets $|\mathcal K_{\mathrm{bud}}|$.

The coboundary $D$ has dimensions $M\times N$ and is block sparse with $\bigO(|E|)$ nonzero blocks. Sparse storage, assembly, and matrix--vector costs are $\bigO(\operatorname{nnz}(D))$, where $\operatorname{nnz}(D)$ counts scalar nonzeros in all restriction blocks. This sparse-operation cost becomes $\bigO(M+N)$ only under a local-interpolation sparsity assumption. Computing $\omega=Ds$ costs $\bigO(MN)$ in the dense regime and $\bigO(\operatorname{nnz}(D))$ in the sparse regime.

For each budget label $b$, budgeted intervention through Algorithm~\ref{alg:intervene} forms the restricted operator $A_{P_b}=DQ_b$ with $p_b\coloneqq\dim V_{P_b}$. A dense thin SVD or an equivalent rectangular factorization costs $\bigO(Mp_b\min\{M,p_b\})$ per budget, in addition to matrix--vector products with $\omega$ and any cost-metric factorization. The dense catalog cost is $\bigO(\sum_{b\in\mathcal K_{\mathrm{bud}}}Mp_b\min\{M,p_b\})$. Directly factoring a dense non-identity restricted metric $Q_b^{\top} C_b Q_b$ adds $\bigO(p_b^3)$ unless the factor is cached or structured.

Identifiability analysis requires a rank-revealing nullspace computation for $D$ and an SVD of $O V_{\ker}$. Each analyzed matrix follows the dense compact SVD cost convention $\bigO(ab\min\{a,b\})$.

\paragraph{Diagnostic record and audit trail.}
\label{sec:case_file_main}

Every \texttt{diagnose} call returns a structured diagnostic record containing $\norm{\omega}$, channel and per-overlap norms, exact-hard feasibility and cost records, separate tolerance and soft records with their residuals, and the identifiability certificate. These fields are the quantities needed to audit the scientific verdict; no additional aggregate score or discrete classification layer is introduced.

\FloatBarrier

\section{Experiments}
\label{sec:experiments}

We evaluate SEAM-$\Omega$ in nineteen experiments spanning synthetic PDE conservation laws; random-sinusoid Burgers initial conditions;
FNO OOD monitoring; real-world open datasets from the University of California, Irvine (UCI) Machine Learning Repository, namely Metro Interstate Traffic Volume, Air Quality, and Individual Household Electric Power Consumption; and United States Geological Survey (USGS) streamflow at three National Water Information System (NWIS) sites. The suite also uses synthetic financial multi-regime data and a synthetic industrial multi-zone fault scenario. The experiments are organized into eight groups aligned with the scientific claims that each group evaluates. Aggregate reporting and repeated-run checks follow Appendix~\ref{app:seed_protocol}. Appendix~\ref{app:protocols} records the remaining common configuration and per-experiment settings needed to interpret and reproduce the results. Cross-experiment comparisons use the raw and edge-count-normalized reports introduced next.

\subsection{Reporting obstruction magnitudes}
\label{sec:exp_norm}

Raw obstruction magnitudes $\norm{\omega}_{2}$ retain the units of the overlap stalks and depend on the scale of $s$ and on the number and dimensions of the overlaps. Every experiment below reports the raw norm, which records absolute magnitude in physical units. Comparisons across covers of different sizes additionally use the edge-count-normalized report
\begin{align*}
	|E|^{-1/2}\norm{\omega}_{2} =\bigl(|E|^{-1}\textstyle\sum_{(i,j)\in E} \norm{\omega_{ij}}^{2}\bigr)^{1/2},
\end{align*}
the root-mean-square (RMS) per-overlap obstruction, defined when $|E|>0$ and controlling for the number of overlaps.

\subsection{Local correctness does not imply admissibility}
\label{sec:group1}

\subsubsection{Local--global amplitude mismatch}
\label{sec:local_global_amplitude_mismatch}

\paragraph{Setup.}
We use the cover $[0,1] = U_{1}\cup U_{2}\cup U_{3}$ with a $10\%$ overlap fraction. Region $U_{i}$ carries the analytic generator $u_{i}(x) = A_{i}\sin(2\pi x)$ with amplitudes $A_{1} = 1.0$, $A_{2} = 1.5$, and $A_{3} = 0.7$. Local explanations carry only a state block. The three regional generators are independently specified and internally exact, making this experiment the controlled realization of Mechanism~M1.

\paragraph{Results.}
The raw obstruction norm is $\norm{\omega}\approx 0.71$ and is dominated by the state channel. We assess admissibility-restoring interventions by their budget-specific hard feasibility, residuals, and costs. Under the unrestricted budget, the hard repair residual $\norm{\omega-D\delta_{\mathrm{all}}^{\star}}$ falls below $10^{-12}$. The closure-revision budget is infeasible because the explanations carry no closure block to absorb the mismatch, whereas the state-revision budget succeeds with cost $\norm{\delta^{\star}}^{2}\approx 0.50$. The state channel is dominant on both overlaps, and the state-revision budget is the only retained proper account. This verdict agrees with the amplitudes having been specified independently across the three regional generators.

\subsubsection{Zero-floor negative control}
\label{sec:zero_floor_negative_control}

\paragraph{Setup.}
A valid obstruction measure must report $\norm{\omega} = 0$ when no inconsistency is present. Using the USGS Potomac second-order autoregressive system, denoted AR(2), we evaluate three conditions on a three-region abstract cover: an identical-model control, which uses identical AR(2) weights across regions; a heterogeneous-seasonal reference, which uses the calibrated seasonal weights from the Potomac streamflow seasonality baseline; and a noise-perturbed identical-model control, which adds independent and identically distributed Gaussian noise at the $10^{-8}$ scale to predictor outputs from otherwise identical models.

\paragraph{Results.}
The identical-model control gives $\norm{\omega} = 0.0000 \pm 0.0000$. The heterogeneous-seasonal reference gives $\norm{\omega} = 0.0266 \pm 0.0000$ and exactly reproduces the Potomac streamflow seasonality baseline. The noise-perturbed identical-model control gives $\norm{\omega} = (2.71 \pm 0.78)\times 10^{-8}$, well below $10^{-6}$ and several orders of magnitude below the smallest real-world obstruction of $0.0266$ from the Potomac streamflow seasonality study. The zero floor is robust to floating-point perturbations at the $10^{-8}$ scale.

\subsubsection{Household-power cross-framework audit}
\label{sec:household_power_audit}

\paragraph{Setup.}
We use the UCI Individual Household Electric Power Consumption dataset \citep{hebrail2012household}. The target is \texttt{Global\_active\_power} in kW. Features comprise the cyclical hour represented by sine and cosine, standardized voltage, and standardized reactive power. Separate winter, spring, summer, and autumn models are fitted by three officially maintained baseline frameworks: scikit-learn (SK) \texttt{LinearRegression} \citep{pedregosa2011sklearn}; XGBoost (XGB) \texttt{XGBRegressor} \citep{chen2016xgboost}; and statsmodels (SM), with ordinary least squares (OLS) implemented by \texttt{OLS} \citep{seabold2010statsmodels}. Each seasonal model is trained on an $80\%$ split and evaluated on the held-out $20\%$. The four seasonal models are wrapped by \texttt{from\_predictor} and passed to \texttt{diagnose} on a path cover with overlaps $(\mathrm{winter},\mathrm{spring})$, $(\mathrm{spring},\mathrm{summer})$, and $(\mathrm{summer},\mathrm{autumn})$.

\paragraph{Local accuracy.}
All three model families achieve mean seasonal $R^{2} > 0.999$, meeting the local-accuracy criterion used here.

\paragraph{Global admissibility audit.}
SEAM-$\Omega$ detects nontrivial obstruction in all three model families: $\norm{\omega}_{\mathrm{SK}} = 9.424 \pm 0.732$, $\norm{\omega}_{\mathrm{XGB}} = 9.391 \pm 0.719$, and $\norm{\omega}_{\mathrm{SM}} = 9.424 \pm 0.732$. The state channel is dominant for all three model families, and the dominant overlap is $(\mathrm{winter},\mathrm{spring})$: The winter--spring transition carries the largest seasonal gradient in domestic electricity consumption.

\paragraph{Physical and operational interpretation.}
The obstruction norm is in units of kW$\cdot$(point-count)$^{1/2}$, reflecting the $L^{2}$ disagreement across three overlaps with 30 evaluation points per overlap. The global norm corresponds to an RMS disagreement of approximately $9.424/\sqrt{3\cdot30}\approx 1.0$\,kW per evaluation point. An automated demand-response system that switches between seasonal models at the transition would exhibit prediction jumps of this magnitude.

\subsection{Missing closure is recoverable}
\label{sec:group2}

\subsubsection{Hidden-source recovery for Burgers}
\label{sec:hidden_source_recovery}

\paragraph{Setup.}
The true PDE is the inviscid Burgers equation with a hidden Gaussian source:
\begin{align*}
	u_{t} + \onehalf (u^{2})_{x} = A \exp(-(x - x_{0})^{2}/(2 w^{2})),
\end{align*}
where $A = 0.6$, $x_{0} = 0.5$, and $w = 0.06$. A source-free finite-volume Burgers solver is run on each region, and the local explanation emits a closure-summary block measuring the missing-source residual on the overlap. This block is audited through the closure restriction, while the state-revision budget handles state mismatches.

\paragraph{Results.}
The closure-channel recovery module computes the Moore--Penrose projected correction $\delta_{C}^{\star}$; under the sign convention, the recovered physical source increment is $-\delta_C^{\star}$. The relative closure-channel residual of this projected recovery is $\norm{\omega^{\mathrm{closure}} - D_{C}\delta_{C}^{\star}} / \norm{\omega^{\mathrm{closure}}} < 0.05$, equivalently below $5\%$. The recovered closure peak is $\hat{x}_{0} = 0.501$, within $0.2\%$ of the true center $x_{0} = 0.5$ and displaced by $1.7\%$ of the source width $w = 0.06$. Corollary~\ref{thm:closure_recovery} gives zero residual under exact image membership; the finite-grid reconstruction above is the projected numerical counterpart of that exact case. The closure channel is dominant on both overlaps.

\paragraph{Scientific interpretation.}
For the identifiable closure-summary model used in this experiment, the recovered physical increment $-\delta_{C}^{\star}$ is the minimum-norm closure reconstruction. The reconstruction makes the missing-physics hypothesis explicit and provides the recovered source profile together with its projected residual, following the physical interpretation specified in Section~\ref{sec:missing_physics}. For the saved pre-repair raw cochain, $\norm{\omega}=2.020994$; the overlap norms are $1.428951$ on $U_{12}$ and $1.429166$ on $U_{23}$. The closure-channel norm is $2.020986$, compared with $0.005873$ in the state channel and zero in the observation and optional contract-metadata channels. Thus, the closure channel contributes $99.9992\%$ of $\norm{\omega}^{2}$, supporting the closure-dominated diagnostic reading. Figure~\ref{fig:obstruction_heatmap} shows these per-overlap and per-channel summaries of the repair-input $\omega$.

\begin{figure}[pos=t!]
	\centering
	\includegraphics[width=0.90\linewidth]{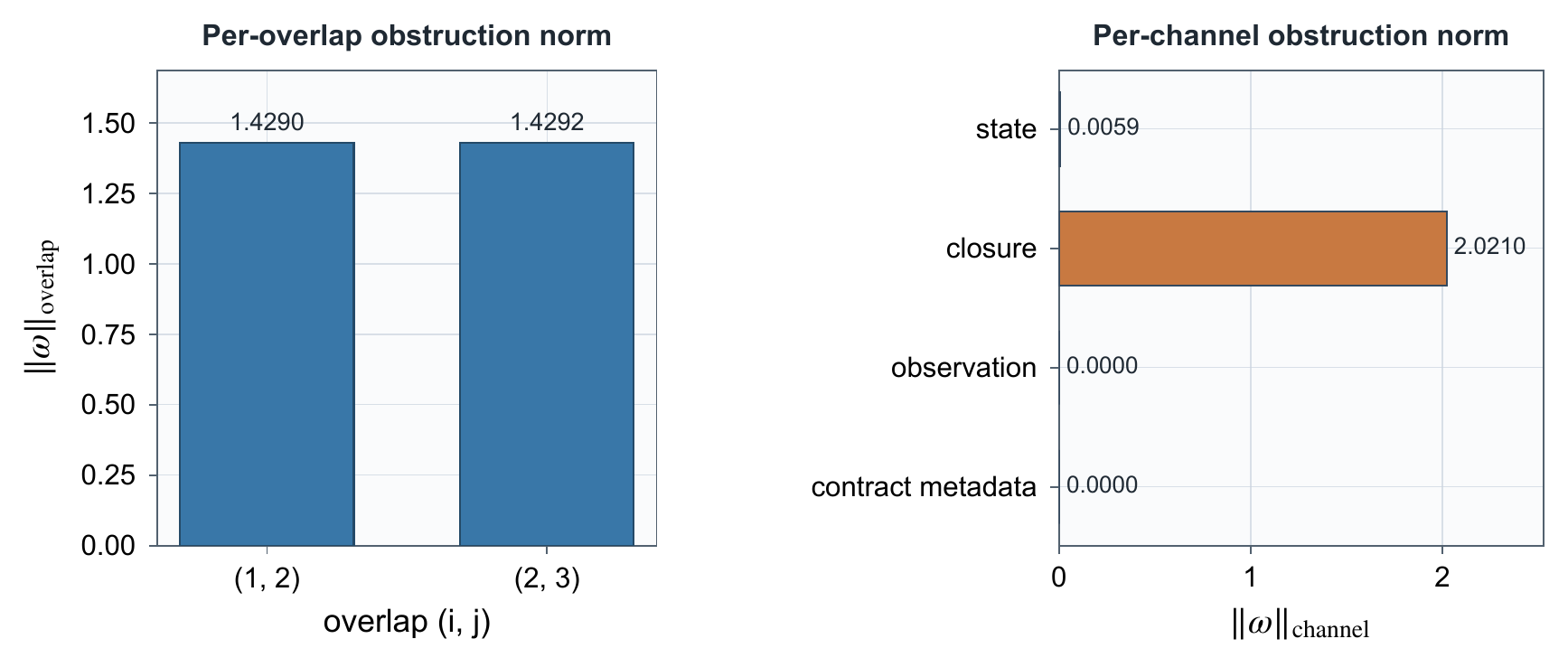}
	\caption{Per-overlap and per-channel obstruction summary for hidden-source recovery for Burgers, computed directly from the saved pre-repair raw cochain $\omega=Ds$. The left plot indexes the two overlaps $U_{12}$ and $U_{23}$ of the three-region cover; the right plot indexes the three primary channels and the optional contract-metadata block, which is absent in this experiment and therefore has zero norm. The closure channel contributes $99.9992\%$ of $\norm{\omega}^{2}$, while the other channels contribute negligibly. This concentration is the visual signature behind the closure-dominated diagnostic reading. Repair under the closure-revision budget is the minimum-norm closure recovery of Corollary~\ref{thm:closure_recovery}.}
	\label{fig:obstruction_heatmap}
\end{figure}

\subsubsection{Conservation-contract detectability}
\label{sec:conservation_contract_detectability}

\paragraph{Setup.}
Three Burgers regions cover $[0,1]$. The central region $U_{2}$ has an active conservation contract. The repair magnitude used in Theorem~\ref{thm:conservation} and reported here is $\epsilon_{\mathrm{repair}}\coloneqq\norm{\delta c_q}$, the norm of the actual closure-channel vector entering the overlap restriction at the repaired endpoint. This endpoint repair magnitude is distinct from the accumulated absolute repair magnitude $\epsilon_{q}^{\mathrm{abs}}\coloneqq\sum_t|\delta_q(t)|$, because the closure block accumulates signed increments, while the optional metadata block records their magnitudes (Appendix~\ref{app:ctseam_motivation}). Regions $U_{1}$ and $U_{3}$ have no conservation repair, so their restricted repair contributions vanish on their incident overlaps. The 1-skeleton of this cover's nerve is the three-vertex path joining regions $1$, $2$, and $3$ in that order;
its two edges correspond to the overlaps $U_{12}$ and $U_{23}$.

\paragraph{Results.}
$\norm{\omega^{\mathrm{closure}}_{12}} = 1.172$ and $\norm{\omega^{\mathrm{closure}}_{23}} = 1.172$. The closure-revision intervention recovers the repair magnitude to within $0.83\%$ relative error. The smallest nonzero singular value of the closure restriction is $\sigma_{\min}^{+} = 2.828$; the lower bound in Theorem~\ref{thm:conservation} is $0.128$; and the observed norm exceeds the bound by a factor of $9.2$ under the one-sided zero-contribution setting of Theorem~\ref{thm:conservation}.

\paragraph{Parametric sweep setup.}
We sweep five values of the repair magnitude in Theorem~\ref{thm:conservation}, $\epsilon_{\mathrm{repair}}=\norm{\delta c_q}$, namely $\epsilon_{\mathrm{repair}}\in \{0.05, 0.10, 0.20, 0.30, 0.50\}$. This parametric sweep yields 25 configurations, whose run construction is specified in Appendix~\ref{app:seed_protocol}. The discretization-tolerant criterion is $\norm{\omega^{\mathrm{closure}}_{ij}} \geq 0.9 \cdot \sigma_{\min}^{+} \epsilon_{\mathrm{repair}}$ with a prespecified $10\%$ tolerance.

\paragraph{Parametric sweep results.}
All 25 configurations satisfy this criterion. The minimum observed slack across all configurations exceeds $0.9$. Both the empirical norm and the lower bound scale approximately linearly with $\epsilon_{\mathrm{repair}}$, consistent with the linear structure of the diagnostic. The diagnostic records retain the unadjusted ratios $\norm{\omega^{\mathrm{closure}}_{ij}}/ (\sigma_{\min}^{+}\epsilon_{\mathrm{repair}})$ alongside the discretization-tolerant result. Figure~\ref{fig:conservation_sweep} plots the seed-aggregated empirical $\norm{\omega^{\mathrm{closure}}_{ij}}$ against the theoretical lower bound at the five repair magnitudes.

\begin{figure}[pos=t!]
	\centering
	\includegraphics[width=0.65\linewidth]{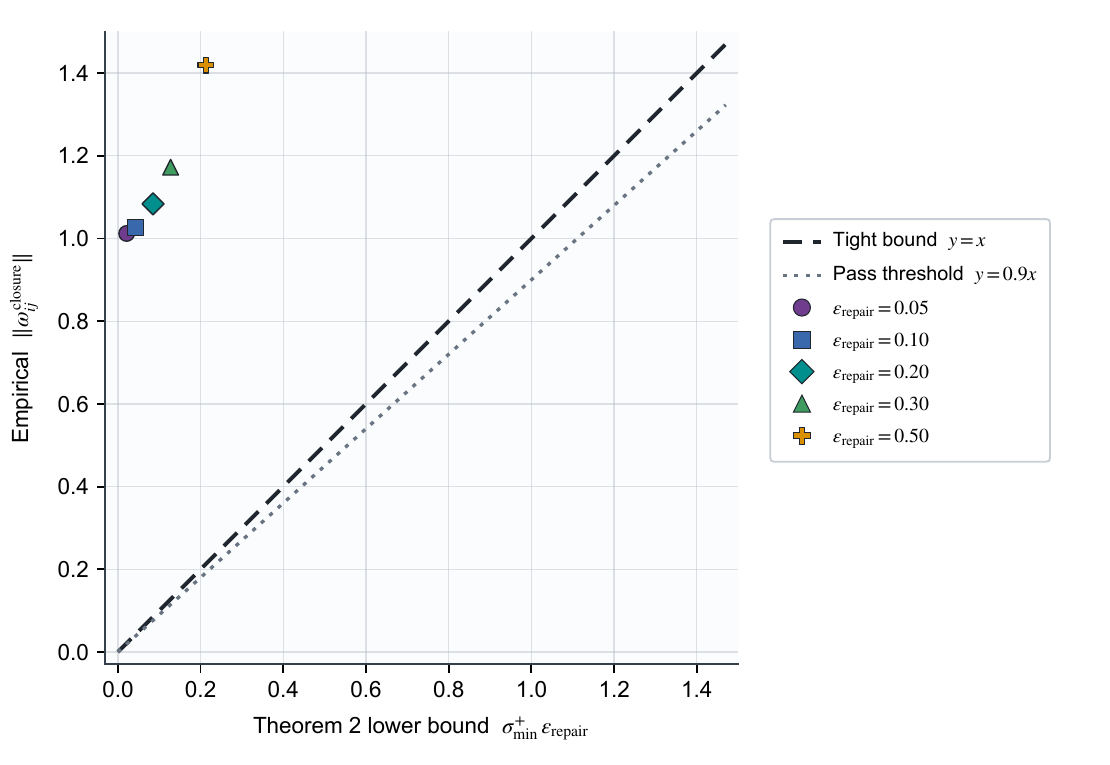}
	\caption{Theorem~\ref{thm:conservation} parametric sweep. All 25 configurations pass, corresponding to a $100\%$ rate. Seed-aggregated observed $\norm{\omega^{\mathrm{closure}}_{ij}}$ as a function of the lower bound $\sigma_{\min}^{+} \epsilon_{\mathrm{repair}}$ at five values of $\epsilon_{\mathrm{repair}}=\norm{\delta c_q}$. With five seeds per value, the plotted points summarize the 25 configurations in the run construction of Appendix~\ref{app:seed_protocol}. All configurations satisfy the prespecified $10\%$ discretization-tolerant criterion. Empirical norms can exceed the lower bound because the row-space component of $\delta c_q$ has components along singular directions associated with singular values larger than $\sigma_{\min}^{+}$; row-space-orthogonal kernel components are annihilated and do not increase the observed obstruction.}
	\label{fig:conservation_sweep}
\end{figure}

\subsection{Data--physics conflict attribution}
\label{sec:group3}

\paragraph{Setup.}
The Burgers solver with a correct Gaussian closure prior is run on three overlapping regions. In the sensor-corruption condition, one region's sensor mean is shifted by $5\sigma$; in the closure-corruption condition, one region's closure parameter is biased. The sensor-corruption condition realizes Mechanism~M3, and the closure-corruption condition is the matching physics-side control. The budgets compared are the sensor-rejection budget, which reweights sensors, and the closure-revision budget, which shifts closure parameters. This experiment evaluates the soft budget-attribution branch; squared intervention norms and residuals are reported for both injected corruption types under the reporting contract of Definition~\ref{def:budgeted_obstruction}.

\paragraph{Results: sensor corruption.}
The observation-channel soft squared intervention norm is $n^{\mathrm{soft}}_{\mathrm{obs}} = 0.08$, the closure-channel counterpart is $n^{\mathrm{soft}}_{\mathrm{closure}} = 1.47$, and their ratio favors the sensor-rejection budget by a factor of $18.4$. The residual-aware empirical diagnosis is a data conflict.

\paragraph{Results: closure corruption.}
The closure-channel soft squared intervention norm is $n^{\mathrm{soft}}_{\mathrm{closure}} = 0.11$, the observation-channel counterpart is $n^{\mathrm{soft}}_{\mathrm{obs}} = 1.39$, and their ratio favors the closure-revision budget by a factor of $12.6$. The residual-aware empirical diagnosis is a closure conflict.

In both conditions, the budget with the lower squared intervention norm identifies the injected source of disagreement; the diagnostic records retain the corresponding residuals.

\subsection{Blind-closure identifiability}
\label{sec:group4}

\paragraph{Setup.}
The closure block of each region is spanned by two basis functions. A visible function is detected by $O$, whereas an invisible function belongs to $\ker O$. We run the identifiability analyzer from Algorithm~\ref{alg:ident}.

\paragraph{Results.}
$\dim \Nblind = 1$ in the global explanation space. The computed orthonormal basis of $\Nblind$ consists of a single vector whose support lies entirely in the closure channel. The corresponding channel fractions are $100\%$ closure and $0\%$ for each of the state, observation, and optional metadata blocks. Observation invisibility satisfies $\norm{O B_{\Nblind}} < 10^{-14}$. Two closure configurations with identical predicted states and sensor readings, differing only in the blind admissible direction, are constructed and verified to differ by $\norm{s - s'} > 0.15$. Together, the constructed pair and the observation-invisibility certificate establish a 1D blind admissible direction.

\subsection{Governing-model discrimination and backend interoperability}
\label{sec:group5}

\subsubsection{PDE model-form discrimination: Burgers versus advection}
\label{sec:pde_geometry_discrimination}

\paragraph{Setup.}
The true PDE is linear advection, $u_{t} + c u_{x} = 0$, with $c = 0.5$ on $[0,1]$, periodic boundary conditions, and a smooth bump as the initial condition. Two candidate generators are evaluated: finite-volume Burgers and finite-volume linear advection at $c=0.5$.

\paragraph{Results.}
$\norm{\omega}_{\mathrm{Burgers}} = 0.84$ and $\norm{\omega}_{\mathrm{advection}} = 0.031$; ranking by obstruction selects linear advection with a $96\%$ relative obstruction gap. The Burgers obstruction is localized near the steep region of the bump, where the Burgers characteristic speed $u$ departs most from the constant advection speed $c$. Both state-only comparisons have state-channel obstruction above the default $\tau_{\mathrm{zero}}=10^{-6}$, but the advection run has substantially lower obstruction. The comparison therefore selects advection as the lower-obstruction model form while retaining its nonzero state-channel diagnostic.

\paragraph{Observed regime.}
Across the tested model forms and parameter range, obstruction increases monotonically with model mismatch. The compared quantity is the raw admissibility defect $\omega=Ds$ of Definition~\ref{def:raw_defect}. Because this comparison uses state-only explanations, the reported obstruction lies in the state channel; that obstruction is the raw defect, not a budget-specific residual or an intervention cost.

In this experiment, model mismatch specifically denotes the discrepancy between the candidate transport dynamics and the data-generating linear advection dynamics. The correct candidate transports the bump at the constant characteristic speed $c=0.5$, whereas the Burgers candidate uses the state-dependent speed $u$; the difference between the two speeds is largest near the steep portion of the bump, which is also where the overlap defects concentrate. Thus, the value $0.031$ for advection represents a small but nonzero residual overlap disagreement at the tested resolution, while the value $0.84$ for Burgers records the much larger incompatibility induced by the wrong governing dynamics. The monotone statement is an empirical observation for this controlled comparison and its tested parameter range, not a claim that $\norm{\omega}$ must be monotone in an arbitrary parametrization of model error.

\subsubsection{Backend interoperability: finite-volume and CT-SEAM generators}
\label{sec:backend_interoperability}

\paragraph{Setup.}
Two backends produce local explanations on the same three-region Burgers domain. Backend~A is a plain finite-volume Burgers generator with zero closure. Backend~B is the same solver with an active conservation-theoretic closure implemented by CT-SEAM as described in Appendix~\ref{app:ctseam}, and with region-varying repair amplitudes. Both outputs are audited with the same cover, restrictions, channel metrics, and intervention catalog.

\paragraph{Results.}
The diagnostic-record comparison gives the raw obstruction-norm ratio $\norm{\omega}_{B}/\norm{\omega}_{A} > 1$, confirming that the conservation repair introduces detectable closure-channel disagreement. The channel report assigns the additional signal in Backend~B to the closure channel, as expected from its region-varying repair. The diagnostic records report the obstruction, channel, and intervention summaries for both backends.

\subsection{Cross-domain and stochastic-generator evaluations}
\label{sec:group6}

\subsubsection{Random-sinusoid Burgers stress test}
\label{sec:random_sinusoid_burgers}

\paragraph{Setup.}
The Burgers initial condition is generated as a superposition of $N_{\sin} = 10$ random sinusoids with integer frequency $k_{n}\sim\Unif\{1,\ldots,8\}$ and amplitude $a_{n}\sim\mathcal{N}(0, 1/k_{n})$. The construction adapts the random-sinusoid family used in PDEBench to this ten-wave, Gaussian-amplitude distribution \citep{takamoto2022pdebench}. The local solver generates all trajectories used in this study; PDEBench supplies methodological precedent rather than experiment data.

\paragraph{Results.}
The aggregate run protocol gives $\norm{\omega} = 0.4754 \pm 0.0000$.

\subsubsection{Metro traffic seasonality}
\label{sec:metro_traffic_seasonality}

\paragraph{Setup.}
We use the UCI Metro Interstate Traffic Volume dataset \citep{hogue2019metro}. A three-region seasonal cover spans winter, summer, and a transition region, with one AR(2) predictor fitted per region. This study, together with the streamflow and air-quality studies that follow, instantiates Mechanism~M2 on measured data.

\paragraph{Results.}
The aggregate run protocol gives $\norm{\omega} = 0.0973 \pm 0.0000$. The obstruction reflects genuine cross-seasonal mismatch in the AR(2) parameters.

\subsubsection{Potomac streamflow seasonality and multi-site regime comparison}
\label{sec:potomac_streamflow_seasonality}

\paragraph{Potomac streamflow setup.}
We use USGS streamflow data \citep{usgs2016nwis} from site 01646500 on the Potomac River at the Little Falls Pump Station near Washington, District of Columbia (DC). The dataset covers 2021-10-01 through 2022-09-30 with 365 daily observations. The cover has three hydrological regions, with one AR(2) predictor fitted per region.

\paragraph{Potomac streamflow results.}
The aggregate run protocol gives $\norm{\omega} = 0.0266 \pm 0.0000$. This norm is the smallest value across the real-world experiments, consistent with the Potomac's moderate seasonality.

\paragraph{Multi-site comparison setup.}
We use three USGS sites with distinct hydrological regimes: 01413500 is the East Branch Delaware River at Margaretville, New York, with Catskill snowmelt and strong seasonality; 01646500 is the Potomac River near Washington, DC, at the Little Falls Pump Station, with moderate seasonality; and 02169500 is the Congaree River at Columbia, South Carolina, with mild southern seasonality. The hypothesis states that the southern site's obstruction is below the maximum northern-site value.

\paragraph{Multi-site comparison results.}
$\norm{\omega}_{01413500} = 0.054$, $\norm{\omega}_{02169500} = 0.037$, and $\norm{\omega}_{01646500} = 0.027$. The hypothesis is supported: The southern site is below the northern-site maximum. The strong-seasonality site carries the largest obstruction, whereas the moderate- and mild-seasonality sites do not follow the stated label order under a common configuration.

\subsubsection{Air-quality temporal consistency}
\label{sec:air_quality_temporal_consistency}

\paragraph{Setup.}
We use the UCI Air Quality dataset \citep{devito2008airquality}, which contains 9357 hourly records of the four pollutant channels CO, C$_6$H$_6$, NO$_x$, and NO$_2$ from an Italian monitoring station. The cover has three temporal regions, with one OLS predictor fitted per region.

\paragraph{Results.}
The fixed-data protocol gives $\norm{\omega} = 0.2850$ for the UCI dataset and preprocessing protocol. The computation is deterministic. The higher $\norm{\omega}$ relative to streamflow is consistent with the stronger seasonality of air quality, governed by temperature-dependent photochemistry.

\subsubsection{Financial regime and industrial fault detection}
\label{sec:financial_industrial_detection}

\paragraph{Financial regime detection.}
Synthetic log returns are generated across bull, bear, and sideways regimes. The regime-dependent volatility of those returns is motivated by a selected subset of the stylized facts documented by \citet{cont2001empirical}. The generator represents this selected subset of asset-return properties. The heterogeneous case yields $\norm{\omega} = 0.0137 \pm 0.0007$, with the closure channel serving as a residual-volatility proxy and dominating the aggregate result. The bull--bear overlap contributes $\approx 57\%$ of the total obstruction; the bear--sideways overlap contributes $\approx 43\%$. The identical-model control gives $\norm{\omega} = 0$.

\paragraph{Industrial multi-zone fault detection.}
We simulate a synthetic three-zone serial process with one OLS regressor per zone and inject a $0.3\sigma$ step fault into the midstream zone halfway through the evaluation window. Under normal operation, $\norm{\omega}_{\mathrm{normal}} = 0.491 \pm 0.199$; under fault injection, $\norm{\omega}_{\mathrm{fault}} = 2.039 \pm 0.501$, an increase by a factor of $4.1$. The state channel dominates the aggregate fault response, correctly attributing the inconsistency to process-variable changes rather than to a closure or sensor conflict. The identical-zone control gives $\norm{\omega} = 0.000$.

\subsection{Learned-generator monitoring}
\label{sec:group7}

\subsubsection{Controlled predictor perturbation}
\label{sec:controlled_predictor_perturbation}

\paragraph{Setup.}
Starting from the USGS Potomac AR(2) predictors calibrated in the Potomac streamflow seasonality study, we perturb the spring predictor weights in 10 random directions with magnitudes $\sigma\in\{0, 0.05, 0.10, 0.20, 0.40, 0.80, 1.60\}$ and compute $\norm{\omega}$.

\paragraph{Results.}
We define the empirical monotonicity fraction as the proportion of adjacent perturbation-level pairs for which $\norm{\omega}$ strictly increases. This fraction is $1.00\pm0.00$ under the aggregate run protocol. The zero-perturbation floor is $\norm{\omega} = 0.0266$, equal to the Potomac streamflow seasonality baseline. The diagnostic is strictly monotone on this dataset across all tested magnitudes and directions. This behavior is empirical monotonicity in perturbation magnitude for this dataset and sweep, not a general guarantee. Figure~\ref{fig:inconsistency_detection} plots $\norm{\omega}$ as a function of $\sigma$, with mean $\pm$ standard deviation across the 10 perturbation directions.

\begin{figure}[pos=t!]
	\centering
	\includegraphics[width=0.84\linewidth]{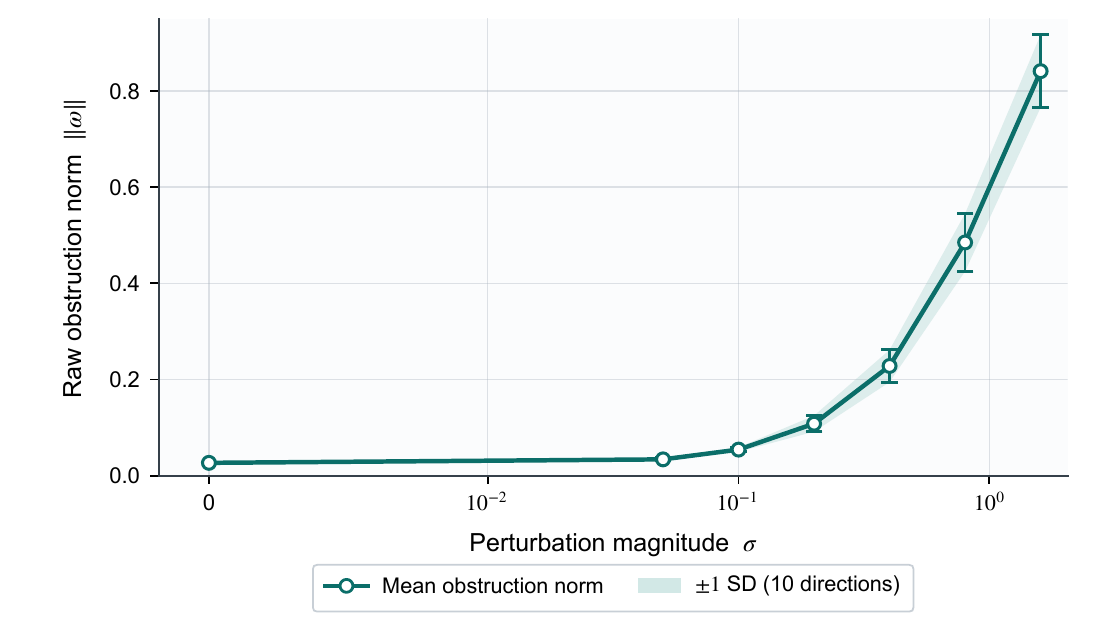}
	\caption{Inconsistency detection shows that $\norm{\omega}$ scales monotonically with injected perturbation magnitude. Starting from the USGS Potomac AR(2) predictors calibrated in the Potomac streamflow seasonality study, the spring predictor weights are perturbed in 10 random directions with magnitudes $\sigma \in \{0, 0.05, 0.10, 0.20, 0.40, 0.80, 1.60\}$. The curve shows the mean $\norm{\omega}$ across directions; error bars and the shaded band show $\pm$ one standard deviation. The vertical axis reports the raw obstruction norm. Across the full range, the empirical monotonicity fraction is $1.00\pm0.00$ under the aggregate run protocol.}
	\label{fig:inconsistency_detection}
\end{figure}

\subsubsection{FNO OOD monitoring}
\label{sec:fno_corr}

\paragraph{Setup.}
An FNO \citep{li2021fourier} is trained on a distribution of Burgers initial conditions and evaluated on OOD covariate-shifted inputs. The FNO-1D architecture has width $64$, four spectral layers, mode cutoff $m=16$, Gaussian error linear unit activation, and approximately $0.65$ million parameters. The streaming obstruction $\omega(t)$ is computed from the FNO regional explanations and their overlap restrictions. The study is the experimental realization of Mechanism~M4. A high-resolution reference solver independently supplies the $L^{2}$ prediction-error coordinate in the correlation analysis.

\paragraph{Dataset.}
The dataset represents the 1D Burgers equation on $\Domain = [0,1]$ over $t\in[0, 0.5]$, with viscosity $\nu = 0.01$ and periodic boundary conditions. The spatial and temporal grids have $N_{x} = 128$ and $N_{t} = 100$ points, respectively. The training-distribution initial conditions are superpositions of three sinusoids:
\begin{align*}
	u_{0}(x) = \sum_{\ell=1}^{3} a_{\ell}\sin(2\pi q_{\ell}x + \phi_{\ell}).
\end{align*}
Here $a_{\ell}\sim \Unif[-0.5, 0.5]$, $q_{\ell}\in\{1, 2, 3\}$ chosen uniformly, and $\phi_{\ell}\sim\Unif[0, 2\pi]$.

\paragraph{Training.}
Training uses $N_{\mathrm{train}} = 1000$ trajectories, a batch size of $32$, the Adam optimizer with a learning rate of $10^{-3}$, and cosine annealing over $200$ epochs. The loss is the relative $L^{2}$ error over the $(x, t)$ grid. Repeated-run training and aggregation follow Appendix~\ref{app:seed_protocol}, with all other settings held fixed.

\paragraph{OOD shift families.}
Three parametric shifts modify the training initial-condition distribution. The frequency shift uses $q_{\ell}\in\{4,5,6\}$ instead of $\{1,2,3\}$. The amplitude shift uses $a_{\ell}\sim\Unif[-1.5,1.5]$ instead of $\Unif[-0.5,0.5]$. The spectral-truncation shift uses a band-limited initial condition truncated to four Fourier modes, with random coefficients drawn from $\Unif[-0.5,0.5]$. Each shift family contains $N_{\mathrm{ood}}=500$ test trajectories. For the auxiliary visualization in Figure~\ref{fig:fno_ood}, a separate 20-level sweep uses $\alpha\in[0,1]$ as a path coordinate from the training distribution at $\alpha=0$ to a fixed OOD endpoint at $\alpha=1$. This seed-averaged sweep is distinct from the per-trajectory endpoint evaluations of the three shift families.

\paragraph{Reference solver.}
A high-resolution finite-volume scheme uses $N_{x}^{\mathrm{ref}} = 1024$, $N_{t}^{\mathrm{ref}} = 800$, and Courant--Friedrichs--Lewy (CFL) number $0.4$. Reference solutions are restricted to the FNO grid by spectral downsampling. The relative $L^{\infty}$ residual with respect to an even higher-resolution run is below $5\times 10^{-4}$.

\paragraph{SEAM-$\Omega$ cover.}
The cover has three temporal regions $U_{1} = [0, 0.2]$, $U_{2} = [0.15, 0.35]$, and $U_{3} = [0.30, 0.5]$, with two overlaps. Each region's local generator is the FNO, advanced from initial conditions sampled at the region's start time. The restrictions are identity maps on the overlap-time intervals, with spatial averaging at $n_{\mathrm{eval}} = 8$ equispaced points.

\paragraph{Reported statistics.}
For each shift family, we record the per-trajectory pair $(\norm{\omega(t_{\mathrm f})}, \norm{u^{\mathrm{FNO}}(t_{\mathrm f}) - u^{\mathrm{ref}}(t_{\mathrm f})}_{L^{2}})$ at the final time $t_{\mathrm f} = 0.5$. Each family therefore yields a sample of $N_{\mathrm{ood}}$ pairs. The reported statistics are the Pearson correlation $r$, the Spearman rank correlation $\rho$, and the $95\%$ bootstrap confidence interval (CI) for $r$ from 1000 bootstrap resamples, with run-wise aggregation specified in Appendix~\ref{app:seed_protocol}.

\paragraph{Results.}
Across the three OOD families, the Pearson correlation between $\norm{\omega(t_{\mathrm f})}$ and the $L^{2}$ prediction error is $r = 0.995 \pm 0.003$, with Spearman $\rho = 0.987 \pm 0.005$. The associated $95\%$ bootstrap CIs fall within $[0.991,0.998]$. Across the three tested Burgers--FNO shift families, $\omega$ is a strong regime-shift monitoring signal.

\paragraph{Comparison to confidence-based monitoring.}
We compare $\omega$ with per-trajectory predictive variance from the separate five-member FNO ensemble specified in Appendix~\ref{app:seed_protocol}. The Pearson correlation between ensemble variance and $L^{2}$ error is $r_{\mathrm{var}} = 0.71 \pm 0.04$ on the same OOD test sets, a substantially weaker signal than $\omega$. Ensemble variance also requires five times the computation of a single FNO, whereas $\omega(t)$ requires one FNO forward pass per region with one coboundary multiplication.

\begin{figure}[pos=t!]
	\centering
	\includegraphics[width=0.96\linewidth]{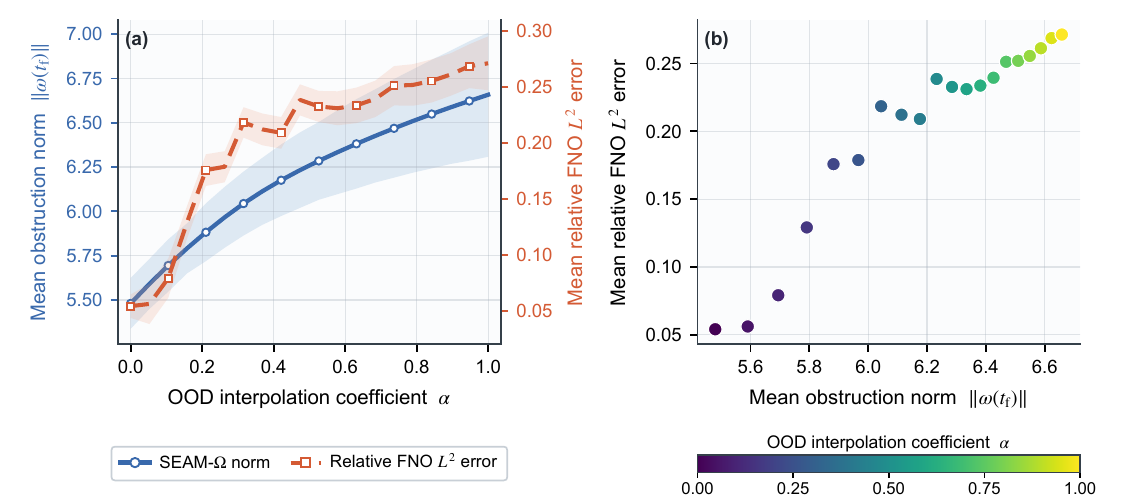}
	\caption{FNO OOD monitoring in the auxiliary 20-level interpolation sweep. The left plot shows the mean SEAM-$\Omega$ obstruction norm and mean relative FNO $L^{2}$ prediction error as functions of the OOD interpolation coefficient $\alpha$; shaded envelopes show $\pm$ one standard deviation over seeds. The right plot shows the same seed-averaged quantities, colored by $\alpha$. Low-$\alpha$ cases occupy the lower-left corner, and increasing distribution shift moves the sequence toward the upper-right corner. The within-sweep correlation $r=0.953$ is a visualization-level statistic for the 20 seed-averaged points; that statistic is distinct from the headline per-trajectory Pearson and Spearman correlations across the three shift families reported in Section~\ref{sec:fno_corr}.}
	\label{fig:fno_ood}
\end{figure}

\subsection{Ablations, robustness, and scalability}
\label{sec:group8}

\subsubsection{Cover-sensitivity ablation}
\label{sec:cover_sensitivity_ablation}

\paragraph{Evaluation points per overlap.}
Varying $n_{\mathrm{eval}}\in\{1,2,3,5\}$ with abstract identity restrictions causes $\norm{\omega}$ to scale as $\sqrt{n_{\mathrm{eval}}}$, consistent with the identity restriction contributing one independent dimension per evaluation point.

\paragraph{Number of regions.}
Varying $n_{r}\in\{2,3,4\}$ leaves the per-overlap disagreement stable, while $\norm{\omega}^{2}$ scales approximately linearly with the number of overlaps $|E|$; equivalently, $\norm{\omega}$ scales approximately as $\sqrt{|E|}$. The per-overlap disagreement metric $|E|^{-1/2}\norm{\omega}$ is therefore a region-count-invariant summary.

\subsubsection{\texorpdfstring{Baseline comparison between naive metrics and SEAM-$\Omega$}{Baseline comparison between naive metrics and SEAM-Omega}}
\label{sec:baseline_comparison}

\paragraph{Setup.}
We compare SEAM-$\Omega$ against two naive cross-region disagreement metrics in the metro traffic seasonality, Potomac streamflow seasonality, and air-quality temporal-consistency studies: cross-prediction root-mean-square error (RMSE) at the overlap evaluation points and mean absolute pairwise difference. Both are normalized by the mean prediction magnitude.

\paragraph{Results.}
SEAM-$\Omega$ orders the three datasets by decreasing $\norm{\omega}$ as air quality, traffic, and streamflow, with respective values $0.2850$, $0.0973$, and $0.0266$. Normalized cross-prediction RMSE orders the same datasets as air quality, streamflow, and traffic, with respective values $0.4512$, $0.3270$, and $0.1572$. That baseline therefore matches SEAM-$\Omega$ on the largest case and reverses the ordering of traffic and streamflow. This ordering difference is descriptive rather than a performance ranking because the three datasets do not supply a ground-truth order of inconsistency. The scalar baselines summarize overall disagreement, whereas SEAM-$\Omega$ also retains channel attribution, budgeted repair analysis, and identifiability.

\subsubsection{GPU implementation check}
\label{sec:gpu_scaling}

The GPU coboundary is tested on synthetic problems with $|I| = 32$ regions and $d_{i} = 10^{3}$ per region. GPU and CPU $\omega$ agree to within a relative tolerance of $10^{-10}$. Per-call wall time on a single NVIDIA A100 GPU is below $5\,\mathrm{ms}$ for the matrix--vector product; the SVD dominates the total runtime as in the CPU regime, but large-scale runs require a rank-revealing truncated approximation that captures all singular values above the chosen tolerance before the projection is treated as exact.

\section{Related work}
\label{sec:related}

SEAM connects four strands of literature: sheaf-theoretic consistency;
physics-informed learning, neural operators, and domain decomposition;
identifiability in inverse problems; and model discrepancy together with distribution-shift monitoring. The comparison below focuses on how each strand treats overlap agreement, repair, observational ambiguity, or shift.

\subsection{Sheaf-theoretic signal processing and consistency}
\label{sec:related_sheaves}

\citet{robinson2014topological} systematized the cellular-sheaf formalism for signal processing, while the thesis of \citet{curry2014sheaves} provides a foundational treatment of sheaves, cosheaves, and their applications. The exposition by \citet{ghrist2014elementary} introduced the formalism to an applied audience, and \citet{bredon1997sheaf} provides the classical reference for sheaf theory and cohomology via derived functors. At the level of coupled scientific models, \citet{robinson2017multimodel} shows how sheaves assemble interacting local models into systems of equations, including dynamical systems and PDEs, whereas \citet{hansen2021opinion} apply sheaves to opinion dynamics, providing a nontrivial example of restriction maps as scientific objects. In a distinct application area, \citet{abramsky2011sheaf} use the failure of global-section existence as the local-to-global obstruction that characterizes contextuality. SEAM shares the obstruction viewpoint, but its local objects are structured scientific explanations, and its operational outputs are channel reports, repairs, and identifiability certificates.

For heterogeneous measurements, \citet{robinson2017sheaves} develops sheaves as a sensor-integration data structure with consistency and fusion tools, whereas \citet{robinson2020assignments} formalizes a consistency radius for locally supplied, potentially noisy assignments. That radius is a scalar summary of the same overlap disagreement that SEAM-$\Omega$ keeps as a cochain and resolves into channel and overlap components.

In machine learning, \citet{hansen2020sheaf} formulate neural networks whose message passing is defined on a specified cellular sheaf, while \citet{bodnar2022neural} parameterize and learn the sheaf structure for heterophilic graph learning. Both constructions use sheaf-Laplacian diffusion, whose algebraic background comes from the spectral theory of the sheaf Laplacian \citep{hansen2019spectral} and the broader graph Hodge-Laplacian framework \citep{lim2020hodge}. These constructions read the sheaf Laplacian as a diffusion operator inside a learning objective. SEAM instead places the same local-to-global algebra on heterogeneous scientific explanations through typed restriction maps, and uses channel-resolved disagreement for diagnosis and intervention.

\paragraph{Explanation-admissibility synthesis.}
SEAM-$\Omega$ builds on global-section and assignment-consistency ideas by treating a structured scientific explanation as the compared object. The stalks of the explanation sheaf separate state, closure, and observation channels, with optional contract metadata, and the geometry-defined restriction maps keep overlap semantics auditable. The resulting obstruction feeds budgeted hypothesis tests and an identifiability quotient.

\subsection{Physics-informed learning, neural operators, and domain decomposition}
\label{sec:related_pinn}

PINNs embed differential-equation residuals and associated conditions in training objectives \citep{raissi2019physics}, while universal differential equations place trainable components inside differential-equation models \citep{rackauckas2020universal}. Both sit within the broader physics-informed machine learning program \citep{karniadakis2021physics}. Studies of PINN optimization in difficult PDE regimes motivate complementary validation of solution behavior alongside the soft physics-residual objective \citep{krishnapriyan2021failure}. Domain-decomposed PINNs directly address interface coupling. The cPINN formulation enforces flux and average-solution conditions at discrete subdomain interfaces, XPINNs generalize the construction to space--time decompositions, and FBPINNs use compactly supported networks on overlapping subdomains \citep{jagtap2020conservative,jagtap2020extended,moseley2023finite}. Neural operators \citep{kovachki2023neural,lu2021learning} learn solution maps; the FNO \citep{li2021fourier} supplies the specific spectral configuration used in the FNO OOD monitoring study. Interface assembly is already central to these methods, and SEAM applies the same assembly logic to structured explanations emitted by heterogeneous generators. \citet{chen2025locality} also demonstrate that neural operators can violate the physical principle of locality, and propose a data decomposition that restricts predictions to a finite domain of dependence. This locality constraint acts inside the learned generator. SEAM's overlap audit acts on the explanations emitted by any generator.

Classical domain-decomposition methods for PDEs \citep{quarteroni1999domain} decompose a global problem into subproblems on overlapping or non-overlapping subdomains; iterative Schwarz, Dirichlet--Neumann, and Robin--Robin methods enforce interface conditions through iteration on the trace mismatch. The trace mismatch is, formally, a single-overlap version of $\omega$. SEAM turns this mismatch into a diagnostic output, accepts data-driven generators alongside PDE solvers, and augments the interface report with channel decomposition, budgeted hypothesis tests, and identifiability.

\paragraph{Generator integration.}
Domain-decomposed PINNs jointly train local networks to produce one PDE solution, with interface quantities appearing as solution constraints or loss terms. SEAM applies a common audit to explanations from PINNs, operators, classical solvers, and regressors, and uses channel-resolved overlap defects to test and compare repair hypotheses. A PINN, a DeepONet, or an FNO can therefore be wrapped as a SEAM local generator. The FNO OOD correlation result in Section~\ref{sec:fno_corr} demonstrates this capability: SEAM monitors a learned operator without modifying the operator's parameters.

\subsection{Identifiability in inverse problems}
\label{sec:related_id_design}

Identifiability analysis in inverse problems and parameter estimation has a long history. \citet{bellman1970structural} introduced the structural-identifiability question for dynamical systems, and \citet{walter1997identification} systematized structural and practical identifiability for parametric models. Profile-likelihood analysis of partially observed dynamical systems demonstrates how structural and practical non-identifiability appear in finite-data inference \citep{raue2009identifiability}. SEAM-$\Omega$ specializes this question to admissible scientific explanations: The blind admissible subspace $\Nblind = \ker D\cap\ker O$ contains directions that preserve overlap agreement while remaining invisible to the current observation map, and the quotient $\Iobs$ contains the observable admissible classes.

\subsection{Model discrepancy and distribution-shift monitoring}
\label{sec:related_discrepancy_shift}

Two traditions quantify the gap between a model and the reality the model describes. Bayesian discrepancy modeling calibrates that gap as a statistical object, whereas shift monitoring detects the gap opening under streaming inputs.

\paragraph{Bayesian model discrepancy.}
\citet{kennedy2001bayesian} model the gap between a computer model and observations as a Gaussian-process discrepancy term, calibrated jointly with the model parameters. This Bayesian discrepancy model is the closest probabilistic analog of $\omega$ because the discrepancy term provides a structured account of where the model fails to match reality. Ignoring or misspecifying that discrepancy can confound physical parameters and produce biased, overconfident inference \citep{brynjarsdottir2014discrepancy}, a caution that also motivates SEAM's separation of closure, observation, and blind directions. SEAM provides a deterministic, multi-channel, local-to-global counterpart: The framework decomposes discrepancy across an explanation sheaf and reports the obstruction cochain.

\paragraph{Consistency-based diagnosis and data reconciliation.}
Constraint-relaxation methods diagnose inconsistent model--data collections by identifying which assumptions must be relaxed \citep{hegde2018consistency}, while process data reconciliation adjusts measurements to satisfy declared balance equations, and uses residuals to identify gross errors \citep{narasimhan2000reconciliation}. Consistency-based diagnosis similarly asks which component assumptions can account for a conflict \citep{reiter1987diagnosis}. SEAM applies this logic to typed regional explanations: Budget subspaces state the permitted revisions, and exact feasibility determines which accounts survive.

\paragraph{Streaming shift detection.}
OOD detection, covariate-shift detection, and streaming model monitoring are active areas. Concept-drift methods detect and adapt to changes in evolving supervised streams \citep{gama2014concept}, while classical sequential change detection includes the cumulative sum statistic \citep{page1954continuous}. Confidence scores supply a foundational OOD baseline \citep{hendrycks2017baseline}, although modern neural-network probabilities may require explicit calibration \citep{guo2017calibration}. Dataset-shift detectors instead compare test samples or distributions with the training distribution; for example, \citet{rabanser2019failing} evaluate two-sample and domain-discrimination approaches for detecting and characterizing dataset shift. Predictive confidence is another monitoring signal, but its calibration can degrade under shift \citep{ovadia2019uncertainty}.

\paragraph{Admissibility monitoring.}
SEAM monitors explanation-admissibility directly from model outputs, complementing input-distribution shift detectors. Shifts that break overlap admissibility produce channel-attributed signals, while admissibility-preserving gauge shifts leave the signal invariant.

\FloatBarrier

\section{Discussion}
\label{sec:discussion}

\subsection{\texorpdfstring{Design choices and scope of SEAM-$\Omega$}{Design choices and scope of SEAM-Omega}}
\label{sec:design_choices}

Three modeling choices delimit the claims of SEAM-$\Omega$.

\paragraph{Finite-dimensional linear setting.}
SEAM-$\Omega$ uses vector-space stalks, linear restriction maps, and a matrix coboundary, yielding the closed-form diagnostic and intervention operators proved here.

\paragraph{Geometry-defined restrictions.}
Restrictions are constructed from problem geometry, keeping $D$ transparent as a scientific object. Learned restriction maps would require their own identifiability and validation analysis.

\paragraph{Pairwise-overlap complex.}
SEAM-$\Omega$ uses only pairwise overlaps, represented by the 1-skeleton of the nerve. Accordingly, the guarantees of SEAM-$\Omega$ concern pairwise compatibility and do not establish higher-order overlap consistency.

\subsection{Interpretation and operational responses}
\label{sec:interpretation}

The following six properties state what the obstruction resolves and connect each diagnostic pattern to an operational response.

\paragraph{Gauge invariance and blind directions.}
A simultaneous shift $s \mapsto s + v$ with $v\in\ker D$ leaves $\omega = Ds$ unchanged. If the shift is also in $\ker O$, the direction $v$ lies in $\Nblind$ and is invisible to current sensors. The identifiability report records this unresolved direction through $\dim\Nblind$.

\paragraph{Restriction-visible repairs.}
If the repaired endpoint's conservation-redistribution repair vector has a component in $\ker\rho^{\mathrm{closure}}_{q,e}$, that component is removed by the overlap restriction. The projected form of Theorem~\ref{thm:conservation}, \eqref{eq:thm_conservation_projected}, therefore measures the row-space-visible repair component; row-space alignment recovers \eqref{eq:thm_conservation}. A large gap between $\norm{\delta c_q}$ and $\norm{\Pi_{\rowsp(\rho^{\mathrm{closure}}_{q,e})}\delta c_q}$, or a non-negligible kernel-component norm $\norm{(\Id-\Pi_{\rowsp(\rho^{\mathrm{closure}}_{q,e})})\delta c_{q}}$, signals the need for a full-column-rank closure restriction or a repair operator aligned with the relevant row space, as in the CT-SEAM experiments of Remark~\ref{rem:conservation_row_space}.

\paragraph{Localization resolution.}
A discrepancy whose true spatial support is smaller than the finest overlap is averaged into uniformly distributed $\omega_{ij}$ components across the overlaps containing that support. SEAM reports $\norm{\omega} > 0$ at overlap resolution. Comparable neighboring norms with no dominant overlap identify where a refined cover will add localization power.

\paragraph{Residual-aware observation attribution.}
Observation-channel attribution is interpreted from the channel norm together with the budget residual rather than from that norm alone. A low-norm but high-residual sensor-rejection record therefore remains unresolved.

\paragraph{Budget ambiguity and catalog dependence.}
When several proper budgets are exact-feasible, the verdict retains all of those budgets rather than using cost to force a causal choice. Soft records with similar intervention magnitudes are likewise reported as ambiguous, with their residuals retained in the diagnostic record. Retention is also conditional on the declared catalog. Exact feasibility eliminates only accounts that were stated, so a cause absent from the catalog is neither tested nor excluded, making a single retained budget as specific as the catalog that produced that budget.

\paragraph{Overlap consistency and absolute accuracy.}
The obstruction norm $\norm{\omega}$ compares neighboring regional explanations with each other rather than with a reference solution. With the cover, restriction maps, discretization, and channel scaling held fixed, the norm aggregates the overlap-wise differences between the restricted explanations, thereby reporting mutual agreement rather than correctness. Under that reading, $\norm{\omega}=0$ certifies agreement on overlaps. Absolute-error checks on regions with ground truth provide the complementary test for shared biases that preserve overlap agreement. The FNO OOD monitoring study in Section~\ref{sec:fno_corr} evaluates both overlap consistency and reference-solver error.

\section{Conclusion}
\label{sec:conclusion}

We introduced SEAM as a generator-agnostic paradigm that treats the global admissibility of scientific explanations as a local-to-global gluing problem, with SEAM-$\Omega$ as its finite explanation-sheaf instantiation. Starting from structured regional explanations with state, closure, and observation channels plus optional contract metadata, SEAM-$\Omega$ restricts neighboring explanations to their overlaps and assembles their disagreements into the 1-cochain $\omega=Ds$. The channel and overlap components of $\omega$ resolve that disagreement by entry and by channel rather than collapsing the cochain into a single score, while the budget-specific images $\im(DP)$ determine repair feasibility. Budgeted preimages then test allowable repair hypotheses: A declared account is refuted when the revisions that account permits cannot remove the obstruction, and priced when those revisions can. Residual-aware soft records provide separately reported empirical attribution when exact feasibility fails.

The blind admissible subspace $\ker D\cap\ker O$ separates inconsistency from what current observations cannot identify, and the streaming $\omega(t)$ supplies an empirical consistency signal under distribution shift. The finite-dimensional result suite provides guarantees for minimum-cost intervention and conservation-contract detectability, with companion identifiability and closure-recovery results. The resulting algorithms assemble $D$, compute budgeted interventions and a basis of $\Nblind$, and emit an auditable diagnostic record.

Nineteen experiments across synthetic PDEs, FNO OOD monitoring, four open datasets, and synthetic financial and industrial systems jointly show that locally accurate or physically plausible generators can remain globally inadmissible, that closure-restricted repair recovers the injected closure sources, and that channel-resolved obstruction distinguishes and monitors multiple failure modes. A zero-obstruction control, a monotone response to injected predictor perturbations, cross-domain channel attribution, and strong correlation between streaming obstruction and FNO prediction error support these findings.

SEAM-$\Omega$ makes overlap consistency exactly certifiable: $\omega=0$ if and only if neighboring restrictions agree. The framework turns violations into auditable diagnoses and interventions, so that explanation-admissibility becomes a tractable global explanation-consistency audit for scientific machine learning.

\clearpage

\printcredits

\section*{Funding}
This work was supported through doctoral training funding provided to Gnankan Landry Regis N'guessan by the African Institute for Mathematical Sciences Data Science Doctoral Program and the Nelson Mandela African Institution of Science and Technology.

\section*{Declaration of competing interest}
The authors declare that they have no known competing financial interests or personal relationships that could have appeared to influence the work reported in this paper.

\section*{Acknowledgments}
Gnankan Landry Regis N'guessan thanks his doctoral supervisor and the Axiom Research Group for sustained discussion of the framework.

\section*{Data and code availability}
The third-party datasets analyzed in this study are publicly available from the sources listed in Appendix~\ref{app:datasets}. UCI data are released under CC BY 4.0, and USGS data are in the public domain under Section 105 of Title 17 of the United States Code. The common experiment configuration, data sources, and preprocessing are specified in Appendix~\ref{app:protocols}. The source code, configuration files, generated synthetic and PDE cases, trained FNO checkpoints, seed-resolved outputs, diagnostic records, pinned environment specifications, and the publication-figure pipeline are available from the corresponding author on reasonable request.

\clearpage
\appendix
\counterwithin{equation}{section}
\counterwithin{figure}{section}
\counterwithin{table}{section}

\section{List of notation}
\label{app:notation}

Table~\ref{tab:notation} lists the symbols used throughout the paper. Sans-serif denotes formal named constants in definitions, equations, and algorithms, such as $\mathsf{AllowClosure}$ and $\mathsf{globally\_admissible}$; monospace denotes literal algorithmic identifiers.

\begin{table}[pos=h!]
	\centering
	\caption{List of notation.}
	\label{tab:notation}
	\small
	\begin{tabularx}{\linewidth}{ l >{\raggedright\arraybackslash}X}
		\toprule
		Symbol                                     & Meaning                                                                            \\
		\midrule
		$\Domain\subseteq\R^{d}$                   & computational domain                                                               \\
		$\cover=(U_i)_{i\in I}$                    & finite cover indexed by regions $I$                                                \\
		$U_{ij}, E$                                & pairwise overlap and nerve-edge set                                                \\
		$F$                                        & finite explanation sheaf                                                           \\
		$\Fui, \Fuij$                              & region and overlap stalks                                                          \\
		$e_i=(u_i,c_i,o_i,\epsilon_i)$             & local explanation with three primary channels and optional contract metadata       \\
		$s=(\mathrm{vec}(e_i))_{i\in I}$           & stack of local explanations; a 0-cochain                                           \\
		$\rho_{i,ij}$                              & restriction from region $i$ to overlap $U_{ij}$                                    \\
		$C^{0}(F), C^{1}(F)$                       & region and overlap cochain spaces                                                  \\
		\addlinespace[4pt]
		$\dz, D$                                   & coboundary operator and its matrix                                                 \\
		$\omega=\dz s=Ds$                          & raw admissibility defect; a 1-cochain                                              \\
		$\omega_{ij}, \omega^{\bullet}$            & overlap and channel components of $\omega$                                         \\
		$P, V_P$                                   & budget projector and allowed correction subspace                                   \\
		$r_P^{\mathrm{hard}}(\omega)$              & hard budget-feasibility residual                                                   \\
		$C, \norm{\delta}_{C}^{2}$                 & positive-definite cost metric and intervention cost                                \\
		$\delta^{\star}$                           & minimum-cost admissibility correction                                              \\
		$D_C, \delta_C^{\star}$                    & closure-channel coboundary block and minimum-norm closure correction               \\
		\addlinespace[4pt]
		$O$                                        & observation map                                                                    \\
		$\Nblind=\ker D\cap\ker O$                 & blind admissible subspace                                                          \\
		$\Iobs=\ker D/\Nblind$                     & observable admissible quotient                                                     \\
		$B_{\Nblind}$                              & orthonormal basis matrix for $\Nblind$                                             \\
		$\epsilon_{\mathrm{repair}}$               & closure-repair norm entering an overlap; distinct from accumulated repair metadata \\
		$\sigma_{\min}^{+}$                        & smallest nonzero singular value                                                    \\
		$\eta_{\mathrm{feas}},\eta_{\mathrm{svd}}$ & feasibility and singular-value tolerances                                          \\
		$\tau_{\mathrm{zero}}$                     & zero-obstruction tolerance                                                         \\
		$\bigO(\cdot)$                             & asymptotic complexity notation; distinct from the observation map $O$              \\
		\bottomrule
	\end{tabularx}
\end{table}

\section{Proofs}
\label{app:proofs}

This appendix collects the proofs of the two theorems and their companion proposition and corollary, together with the auxiliary pseudoinverse identities that those proofs use.

\subsection{Pseudoinverse identities used throughout}
\label{app:pinv_identities}

We collect the standard Moore--Penrose identities used in Section~\ref{sec:theory}. For a real $m\times n$ matrix $A$, the Moore--Penrose pseudoinverse $A\pinv$ is the unique matrix satisfying the following four conditions:

\begin{equation*}
	\begin{aligned}
		A A\pinv A = A \quad \textup{(MP1)},               & \qquad A\pinv A A\pinv = A\pinv \quad \textup{(MP2)},     \\
		(A A\pinv)^{\top} = A A\pinv \quad \textup{(MP3)}, & \qquad (A\pinv A)^{\top} = A\pinv A \quad \textup{(MP4)}.
	\end{aligned}
\end{equation*}

The following identities are immediate consequences:

\begin{lemma}
	\label{lem:pinv_projections}
	$AA\pinv$ is the orthogonal projector onto $\im A$. $A\pinv A$ is the orthogonal projector onto $(\ker A)^{\perp}$. Equivalently, $\Id - A\pinv A$ is the orthogonal projector onto $\ker A$.
\end{lemma}

\begin{proof}
	By identity~MP3, $AA\pinv$ is symmetric. By identity~MP1, $(AA\pinv)^{2}=A A\pinv A A\pinv=A A\pinv$, so the matrix is idempotent and hence an orthogonal projector. The image of $AA\pinv$ is $\im A$. Clearly, $\im(AA\pinv)\subseteq \im A$, and for $v = Au\in\im A$, we have $AA\pinv v = AA\pinv Au = Au = v$, so $\im(AA\pinv) = \im A$. The statement about $A\pinv A$ is analogous.
\end{proof}

\begin{lemma}[Minimum-norm preimage]
	\label{lem:min_norm_preimage}
	For $b\in\im A$, the set $\{x : Ax = b\}$ has a unique minimum-norm element, given by $A\pinv b\in(\ker A)^{\perp}$.
\end{lemma}

\begin{proof}
	The preimage set is the affine subspace $x_{0} + \ker A$ for any one preimage $x_{0}$. The minimum-norm element with respect to the Euclidean norm is the orthogonal projection of $0$ onto this affine subspace, which is the unique element in $(\ker A)^{\perp}$. By Lemma~\ref{lem:pinv_projections}, $A\pinv b\in(\ker A)^{\perp}$, and $AA\pinv b = b$ since $b\in\im A$. Hence, $A\pinv b$ is the desired element.
\end{proof}

\subsection{Theorem 1: detailed proof}
\label{app:thm_intervention}

\begin{theorem*}[Minimum-cost budgeted intervention]
	Let $P$ be the orthogonal projector onto a closed subspace $V_{P}\subseteq C^{0}(F)$, let $\iota_{V_{P}}\colon V_{P}\hookrightarrow C^{0}(F)$ be the inclusion, and let $C\succ 0$ on $C^{0}(F)$. Set $A_{P} \coloneqq DP\iota_{V_{P}}$ and $C_{P} \coloneqq \iota_{V_{P}}^{\top}C\iota_{V_{P}}$. The problem $\min_{\delta} \onehalf \delta^{\top}C\delta$ subject to $DP\delta = \omega$ and $(\Id-P)\delta = 0$ is feasible if and only if $\omega\in\im A_{P}$; in the feasible case, the unique minimizer is
	\begin{align*}
		\delta^{\star} =  \iota_{V_{P}} C_{P}^{-1} A_{P}^{\top} (A_{P} C_{P}^{-1} A_{P}^{\top})\pinv \omega.
	\end{align*}
	The reported squared intervention cost is
	\begin{align*}
		\norm{\delta^\star}_{C}^{2} = \omega^{\top} (A_P C_P^{-1} A_P^{\top})\pinv\omega,
	\end{align*}
	while the optimum of the quadratic objective is half this value. For $\lambda>0$, the soft restricted solution is
	\begin{align*}
		\delta_\lambda^\star = \iota_{V_P} (A_P^{\top}A_P+\lambda^{-1}C_P)^{-1}A_P^{\top}\omega.
	\end{align*}
	For the same ambient cost metric $C$, if $V_{P_1}\subseteq V_{P_2}$ and both hard problems are feasible, the reported squared hard costs are antitone in the budget. No such conclusion is asserted for two unrelated budget-specific metrics. The repair convention is $s_{\mathrm{repaired}}=s-P\delta^\star$.
\end{theorem*}

\begin{proof}
	\textup{\textbf{Feasibility argument.}} The constraint $(\Id-P)\delta = 0$ is $\delta\in V_{P}$. Combined with $DP\delta = \omega$, the feasible set is $V_{P}\cap (DP|_{V_{P}})^{-1}(\omega) = V_{P}\cap A_{P}^{-1}(\omega)$, nonempty if and only if $\omega\in\im A_{P}$.

	\medskip\noindent\textup{\textbf{Closed-form argument.}}
	Working in $V_{P}$, write $\delta = \iota_{V_{P}}z$ for $z\in V_{P}$. The objective becomes $\onehalf z^{\top}C_{P}z$, which is strictly convex on $V_{P}$ because $C_{P}\succ 0$. The constraint is $A_{P}z = \omega$. The Lagrangian is $L(z,\mu) \coloneqq \onehalf z^{\top}C_{P}z - \mu^{\top}(A_{P}z - \omega)$. Stationarity gives $C_{P}z = A_{P}^{\top}\mu$, so $z = C_{P}^{-1}A_{P}^{\top}\mu$. Substituting into the constraint yields $A_{P}C_{P}^{-1}A_{P}^{\top}\mu = \omega$. A direct computation gives $\im(A_P C_P^{-1} A_P^\top)=\im A_P$, since $\ker(A_P C_P^{-1} A_P^\top)=\ker A_P^\top$; feasibility therefore makes the multiplier equation solvable. The minimum-norm solution of that equation is $\mu^{\star} = (A_{P}C_{P}^{-1}A_{P}^{\top})\pinv\omega$. Hence, $z^{\star} = C_{P}^{-1}A_{P}^{\top}\mu^{\star}$ and $\delta^{\star} = \iota_{V_{P}}z^{\star}$ give the claimed form. With $M \coloneqq A_P C_P^{-1} A_P^\top$,
	\begin{align*}
		{z^\star}^{\top} C_P z^\star = {\mu^\star}^{\top}M\mu^\star = \omega^\top M\pinv M M\pinv\omega = \omega^\top M\pinv\omega,
	\end{align*}
	which is the reported squared intervention cost. The objective value of $\onehalf z^\top C_P z$ is half of that cost. Uniqueness of the minimizer $\delta^{\star}$ follows from strict convexity of that objective on $V_{P}$, where $C_{P}\succ 0$.

	\medskip\noindent\textup{\textbf{Soft Tikhonov argument.}}
	The objective $\onehalf \delta^{\top}C\delta + (\lambda/2)\norm{D\delta - \omega}^{2}$, restricted to $\delta = \iota_{V_{P}}z$, becomes $\onehalf z^{\top}C_{P}z + (\lambda/2)\norm{A_{P}z - \omega}^{2}$. Stationarity in $z$ gives $(C_{P} + \lambda A_{P}^{\top}A_{P})z = \lambda A_{P}^{\top}\omega$, equivalently $(A_{P}^{\top}A_{P} + \lambda^{-1}C_{P})z = A_{P}^{\top}\omega$. The operator on the left is positive definite on $V_{P}$ as the sum of a positive-semidefinite term and a positive-definite term, so the unique solution is $z^{\star}_{\lambda} = (A_{P}^{\top}A_{P} + \lambda^{-1}C_{P})^{-1}A_{P}^{\top}\omega$, and $\delta^{\star}_{\lambda} = \iota_{V_{P}}z^{\star}_{\lambda}$.

	\medskip\noindent\textup{\textbf{Budget-antitonicity argument.}}
	For the same ambient metric $C$ and $V_{P_{1}}\subseteq V_{P_{2}}$, the feasible set at $P_{1}$ is contained in that at $P_{2}$. The minimum of the same strictly convex objective over a smaller feasible set is at least the minimum over the larger set; hence, $\norm{\delta^{\star}(P_{1})}_{C}^{2}\geq \norm{\delta^{\star}(P_{2})}_{C}^{2}$. If the metric changes with the budget, this set-inclusion argument no longer proves any cost inequality.
\end{proof}

\subsection{Identifiability proposition: detailed proof and monotonicity corollary}
\label{app:thm_identifiability}

\begin{proposition*}[Identifiability dimensions]
	With $\Nblind = \ker D\cap\ker O$ and $\Iobs = \ker D/\Nblind$, the dimensions satisfy $\dim\Nblind=\dim\ker D-\rank(O|_{\ker D})$, $\dim\Iobs=\rank(O|_{\ker D})$, and $\dim\ker D=\dim\Nblind+\dim\Iobs$.
\end{proposition*}

\begin{proof}
	The identity $\ker(O|_{\ker D}) = \ker D\cap\ker O = \Nblind$ follows from the definition of the restriction. Rank--nullity for $O|_{\ker D}\colon \ker D\to Y$ gives $\dim\ker D = \dim\ker(O|_{\ker D}) + \rank(O|_{\ker D})$, which is the first identity. The quotient dimension formula gives $\dim\Iobs = \dim\ker D - \dim\Nblind = \rank(O|_{\ker D})$, the second identity. Adding the first two identities cancels $\rank(O|_{\ker D})$, leaving the third, $\dim\ker D=\dim\Nblind+\dim\Iobs$.
\end{proof}

\begin{corollary}[Monotonicity of $\Nblind$ under finer observations]
	\label{cor:ident_nested}
	If $O_{1}, O_{2}$ have $\ker O_{1}\supseteq\ker O_{2}$, meaning that $O_{2}$ observes everything $O_{1}$ observes and possibly more, then $\dim\Nblind(O_{2})\leq \dim\Nblind(O_{1})$, and correspondingly $\dim\Iobs(O_{2})\geq \dim\Iobs(O_{1})$.
\end{corollary}

\begin{proof}
	$\ker D\cap\ker O_{2}\subseteq \ker D\cap\ker O_{1}$, so $\dim\Nblind(O_{2})\leq \dim\Nblind(O_{1})$. The complementary inequality on $\Iobs$ follows from $\dim\Iobs = \dim\ker D - \dim\Nblind$.
\end{proof}

\subsection{Closure-recovery corollary: detailed proof}
\label{app:thm_closure_recovery}

\begin{corollary*}[Closure-restricted recoverability]
	Let $D_{C}$ be the closure-channel block of $D$ under Assumption~A0, as a typed map $C^{0}(F)^{\mathrm{closure}}\to C^{1}(F)^{\mathrm{closure}}$. If $\omega^{\mathrm{closure}}\in\im D_{C}$, the unique minimum-norm closure correction in the Euclidean metric is $\delta_{C}^{\star} = D_{C}\pinv \omega^{\mathrm{closure}} \in \im(D_C^\top)=\rowsp(D_C)$, and the repaired stalks $(u_{i}, c_{i} - \delta_{C}^{\star}|_{i}, o_{i}, \epsilon_{i})_{i\in I}$ have closure-channel admissibility defect equal to zero, with non-closure channels unchanged.
\end{corollary*}

\begin{proof}
	Apply Theorem~\ref{thm:intervention}, or directly Lemma~\ref{lem:min_norm_preimage}, to the restricted closure-channel problem
	\begin{align*}
		\min_{z\in C^{0}(F)^{\mathrm{closure}}} \onehalf\norm{z}^{2} \quad\text{subject to}\quad D_C z=\omega^{\mathrm{closure}}.
	\end{align*}
	The closed-form solution to this smaller problem is $\delta_{C}^{\star}=D_C\pinv\omega^{\mathrm{closure}}$. The correction inserted into the full stalk vector is the zero-extension $\iota_{\mathrm{closure}}\delta_C^\star$; throughout this proof, $\delta_C^\star$ denotes the restricted closure vector whenever no ambiguity is possible. The pseudoinverse range identity $D_{C}\pinv y\in (\ker D_{C})^{\perp} = \im(D_C^\top)=\rowsp(D_C)$ for $y\in\im D_{C}$ gives the membership statement. For the raw defect $\omega=Ds$, substituting the zero-extended correction into the full stalk vector under the A1 sign convention gives $s_{\mathrm{repaired}}=s- \iota_{\mathrm{closure}}\delta_{C}^{\star}$. Consequently, on the closure channel, $(Ds_{\mathrm{repaired}})^{\mathrm{closure}} =\omega^{\mathrm{closure}}-D_{C}\delta_{C}^{\star}=0$. Because $D$ is block diagonal under Assumption~A0, $(Ds_{\mathrm{repaired}})^{\bullet}=(Ds)^{\bullet}$ for every non-closure channel $\bullet$.
\end{proof}

\subsection{Theorem 2: detailed proof}
\label{app:thm_conservation}

\begin{theorem*}[Conservation-contract detectability]
	For an oriented edge $e=(a,b)$ and repaired endpoint $q\in\{a,b\}$, assume baseline cancellation, one-sided zero neighbor contribution on $e$, positive rank of $\rho^{\mathrm{closure}}_{q,e}$, and row-space alignment $\delta c_q\in\rowsp(\rho^{\mathrm{closure}}_{q,e})$. Then
	\begin{align*}
		\norm{\omega^{\mathrm{closure}}_{e}} \geq \sigma_{\min}^{+}(\rho^{\mathrm{closure}}_{q,e}) \cdot \epsilon_{\mathrm{repair}}.
	\end{align*}
	Without row-space alignment, the same argument gives the projected bound \eqref{eq:thm_conservation_projected}.
\end{theorem*}

\begin{proof}
	Set $\rho \coloneqq \rho^{\mathrm{closure}}_{q,e}$. Baseline cancellation and the one-sided zero-contribution assumption give $\omega^{\mathrm{closure}}_{e}=\varepsilon_q\rho \delta c_q$, with $\varepsilon_q\in\{-1,+1\}$. Decompose $\delta c_q=\delta c_q^\parallel+\delta c_q^\perp$, where $\delta c_q^\parallel\in\rowsp(\rho)$ and $\delta c_q^\perp\in\ker\rho$. Then $\rho\delta c_q=\rho\delta c_q^\parallel$. The singular-value lower bound for $\rho$ restricted to its row space gives $\norm{\rho\delta c_q^\parallel} \geq \sigma_{\min}^{+}(\rho)\norm{\delta c_q^\parallel}$. Hence,
	\begin{align*}
		\norm{\omega^{\mathrm{closure}}_{e}} = \norm{\rho \delta c_{q}} \geq \sigma_{\min}^{+} \norm{\Pi_{\rowsp(\rho)} \delta c_{q}},
	\end{align*}
	which is \eqref{eq:thm_conservation_projected}. Under row-space alignment, $\delta c_{q}\in\rowsp(\rho)$, so the kernel component is zero and $\norm{\Pi_{\rowsp(\rho)} \delta c_{q}} = \norm{\delta c_{q}} = \epsilon_{\mathrm{repair}}$, recovering \eqref{eq:thm_conservation}.
\end{proof}

\section{Algorithmic details and complexity}
\label{app:algorithms}

This appendix collects implementation-level details of the algorithms presented in Section~\ref{sec:algos}, with complexity analysis and numerical-stability considerations.

\subsection{Coboundary assembly}
\label{app:coboundary_assembly}

The coboundary matrix $D$ has a block structure indexed by $(\text{edge}, \text{region})$. For each edge $(i,j)\in E$, the corresponding row block is $[ \ldots, -\rho_{i,ij}, \ldots, +\rho_{j,ij}, \ldots ]$, with zeros in all column blocks not indexed by $i$ or $j$.

\paragraph{Block sparsity.}
Each edge contributes exactly two nonzero column blocks. The matrix $D$ therefore has exactly $2|E|$ nonzero blocks. For typical covers with $|E| \ll |I|^{2}$, storage and matrix--vector costs scale with $|E|$ rather than with the number of region pairs. Path covers, for example, satisfy $|E| = |I| - 1$.

\paragraph{Channel block-diagonal structure.}
When the restriction maps respect the channel decomposition, $D$ is block diagonal with respect to the channel partition. The channel blocks $D^{\bullet}$ for $\bullet\in\{\mathrm{state},\mathrm{closure},\mathrm{obs},\mathrm{meta}\}$ can be assembled independently. This decomposition reduces the SVD cost by a constant factor because a sum of smaller SVDs is cheaper than one large SVD when the channel sizes differ. The block-diagonal structure is also essential for the channel decomposition of $\omega$ in \eqref{eq:channel_decomp}.

\paragraph{Numerical realization.}
After assembly, Algorithm~\ref{alg:intervene} is applied once for each budget basis $Q$, whose orthonormal columns span $V_P=\im P$. The algorithm uses $A_P=DQ$ and $C_P=Q^{\top}CQ$, the hard image-membership residual, the minimum-cost formula of Theorem~\ref{thm:intervention}, and the separately reported Tikhonov record. The returned correction lies in $V_P$ and is subtracted from the current stalk vector according to \eqref{eq:sign_convention}.

\subsection{Pseudoinverse implementation}
\label{app:pinv_impl}

We use the truncated SVD-based pseudoinverse with a rank-tolerance threshold $\eta_{\mathrm{rank}} \coloneqq \max(m, n)\cdot \epsilon_{\mathrm{mach}}\cdot \sigma_{\max}$ as a SEAM-specific convention, where $\epsilon_{\mathrm{mach}}$ is the machine epsilon of the working floating-point type and $\sigma_{\max}$ the largest singular value. The exact routine and version are recorded whenever an implementation uses a library default. For ill-conditioned matrices, including restriction maps with near-rank-deficient overlap blocks, the threshold can be tightened to $10\cdot \epsilon_{\mathrm{mach}}\cdot\sigma_{\max}$ at the cost of including near-singular directions in the pseudoinverse range.

\paragraph{Stability under noise.}
For a perturbed matrix $D + \Delta$, pseudoinverse perturbation bounds require fixed rank or a spectral gap. Under such a rank-stability assumption, bounds of \citet{wedin1972perturbation,wedin1973perturbation} control $\norm{(D+\Delta)\pinv - D\pinv}$ with constants depending on the nonzero singular gap. Without fixed rank, the pseudoinverse need not be continuous. Therefore, the obstruction residual is sensitive to noise in $D$ when $\sigma_{\min}^{+}(D)$ is small. In practice, when $\rank(D)>0$, we monitor the positive-rank condition number
\begin{align*}
	\kappa_{+}(D) \coloneqq \sigma_{\max}(D)/\sigma_{\min}^{+}(D)
\end{align*}
and report warnings when $\kappa_{+}(D) > 10^{10}$. When $\rank(D)=0$, $\kappa_{+}(D)$ is undefined because there is no nonzero singular value; rank-zero cases are reported separately.

\subsection{Budgeted intervention solver: numerical considerations}
\label{app:budget_solver_num}

Algorithm~\ref{alg:intervene} uses the closed-form solution in the feasible case and the Tikhonov-regularized solution in the infeasible case. The transition between the two is governed by the feasibility tolerance $\eta_{\mathrm{feas}}$. Three numerical considerations apply.

\paragraph{Choice of $\eta_{\mathrm{feas}}$.}
We set
\begin{align*}
	\eta_{\mathrm{feas}}\coloneqq \max\{10^{-8}\norm{\omega},\eta_{\mathrm{feas},\min}\}
\end{align*}
by default, with a fixed machine-scale $\eta_{\mathrm{feas},\min}>0$. This positive-floor definition preserves the stated $\eta_{\mathrm{feas}}>0$ condition even when $\omega=0$, avoiding spurious Tikhonov fallbacks triggered by linear-algebra roundoff. The choice can be overridden per call.

\paragraph{Choice of $\lambda$.}
The Tikhonov regularizer controls the trade-off between residual size and intervention norm. By default, we set
\begin{align*}
	\lambda \coloneqq \frac{\max\{\norm{\omega},\varepsilon_{\omega}\}} {10 \max\{\sigma_{\max}(A_P),\varepsilon_{\sigma}\}},
\end{align*}
with positive numerical floors. The soft branch is bypassed entirely when $\norm{\omega}\le\tau_{\mathrm{zero}}$ and the case is already globally admissible. The choice can be overridden per call.

\paragraph{Cost metric $C$.}
The default cost metric is the identity, giving Euclidean intervention norms. Supported domain-specific cost metrics include a Mahalanobis norm with respect to a prior covariance, channel-weighted norms, and energy norms.

\subsection{Identifiability analyzer: numerical considerations}
\label{app:ident_num}

\paragraph{Choice of $\eta_{\mathrm{svd}}$.}
By default, we set
\begin{align*}
	\eta_{\mathrm{svd}}\coloneqq \max\{10^{-8}\sigma_{\max},\eta_{\mathrm{svd},\min}\}
\end{align*}
for the kernel basis, with the same positive floor used for both the SVD of $D$ and the SVD of $O V_{\ker}$. A larger tolerance yields a more inclusive numerical $\ker D$ and can create spurious blind directions by declaring genuinely nonzero singular values to be zero; a smaller tolerance is more conservative and can miss near-null or roundoff-null blind directions.

\paragraph{Blind-subspace certificate.}
After computing the blind admissible basis $B_{\Nblind}$, we verify $\norm{D B_{\Nblind}}$, $\norm{O B_{\Nblind}}$, orthonormality, and the numerical dimension identity from Proposition~\ref{thm:identifiability}. Together, these checks form the certificate; $\norm{O B_{\Nblind}}$ records the observation-invisibility component of that certificate.

\paragraph{Channel decomposition.}
The columns of $B_{\Nblind}$ are decomposed into block components by projecting onto the three primary channel subspaces and the optional metadata subspace. The resulting fractions are included in the diagnostic record.

\subsection{Monitoring pipeline: caching}
\label{app:monitor_caching}

The monitoring pipeline in Algorithm~\ref{alg:monitor} caches all quantities that do not depend on the streaming input.

\paragraph{Cached quantities.}
The cache stores the coboundary matrix $D$, its SVD factors $U_D,\Sigma_D,V_D$, the pseudoinverse $D\pinv$, and the channel-block submatrices $D^{\bullet}$. The cache also stores each projected coboundary $DP_b$ with its pseudoinverse. When the observation map is static, the blind admissible basis $B_{\Nblind}$ is cached as well.

\paragraph{Per-call quantities.}
Only $s(t)=(\mathrm{vec}(e_i(t)))_{i\in I}$, $\omega(t)=Ds(t)$, and the per-budget intervention costs need to be recomputed per call. The wall-clock cost is dominated by one matrix--vector product with $D$ plus $|\mathcal K_{\mathrm{bud}}|$ projection operations.

\section{CT-SEAM conservation backend}
\label{app:ctseam}

This appendix describes the CT-SEAM local generator used in selected PDE experiments. The backend supplies closure-rich explanations through the common SEAM-$\Omega$ generator interface in the conservation-contract detectability and backend-interoperability studies.

\subsection{Motivation and construction}
\label{app:ctseam_motivation}

For PDE systems with conservation laws such as mass, momentum, and energy, a local finite-volume solver may drift from exact conservation due to discretization, boundary handling, or numerical flux choices. A conservation contract enforces conservation by applying a mass-redistribution repair at each step; the accumulated absolute magnitude of that repair may be recorded as optional contract metadata.

\paragraph{Local generator.}
Given a finite-volume Burgers solver on region $U_{i}$, the generator first computes each state update by standard finite-volume time integration. The generator then evaluates the conservation residual
\begin{align*}
	\Delta m_{i}(t) \coloneqq \int_{U_{i}} u(t)\,dx - \int_{U_{i}} u(0)\,dx + \text{flux contributions}.
\end{align*}
Whenever $|\Delta m_{i}(t)| > \tau_{\mathrm{cons}}$, the generator applies a mass-redistribution repair by adding $\delta_{i}(t) \coloneqq - \Delta m_{i}(t) / |U_{i}|$ to the state on $U_{i}$. This correction restores exact conservation. The generator accumulates the repair in the closure block according to $c_{i} \mathrel{+}= \delta_{i}(t) \cdot \chi_{U_{i}}$, where $\chi_{U_{i}}$ is the indicator function of $U_{i}$ in the closure space.

\paragraph{Explanation block.}
The state block of a CT-SEAM explanation contains the conservation-corrected solution. The closure block contains the accumulated repair history, represented by one scalar for each region and optionally weighted by a Gaussian centered at the repair location. The observation block is empty because this solver consumes no sensor data. The optional metadata block contains the accumulated absolute repair magnitude $\epsilon_{i}^{\mathrm{abs}} \coloneqq \sum_{t}|\delta_{i}(t)|$.

\subsection{Role in Theorem 2}
\label{app:ctseam_thm_conservation}

Theorem~\ref{thm:conservation} connects the actual closure-channel vector entering an overlap restriction to the SEAM-$\Omega$ closure-channel obstruction. For an oriented edge $e$ and repaired endpoint $q$, under baseline cancellation, one-sided zero neighbor contribution, positive rank, and row-space alignment,
\begin{align*}
	\norm{\omega^{\mathrm{closure}}_{e}} \geq \sigma_{\min}^{+}(\rho^{\mathrm{closure}}_{q,e}) \cdot \norm{\delta c_q}.
\end{align*}
The conservation-contract detectability study and its discretization-tolerant parametric sweep evaluate this diagnostic using the repair-vector convention defined in that study's setup.

The shared interface also accepts the analytic and finite-volume generators used in the local--global amplitude-mismatch, data--physics conflict-attribution, PDE model-form discrimination, and random-sinusoid Burgers studies, as well as the tabular regression generators used in the application studies. The conservation-contract detectability and backend-interoperability studies use CT-SEAM specifically, with the detectability study providing the parametric evaluation of the claim in Theorem~\ref{thm:conservation}.

\section{Experimental protocols, baselines, and data provenance}
\label{app:protocols}

This appendix records the common experimental configuration, external data sources, preprocessing or generation steps, comparative baselines, and robustness checks needed to interpret the results of Section~\ref{sec:experiments}.

\subsection{Common configuration and random-seed protocol}
\label{app:seed_protocol}

\paragraph{Cover construction.}
1D PDE experiments use a uniform partition of $[0,1]$ into three regions with a $10\%$ overlap fraction. Real-world experiments use either seasonal temporal covers or abstract overlap-list covers, as detailed below.

\paragraph{Solvers.}
PDE solvers use finite-volume schemes with periodic boundaries. Time integration is by explicit Euler with CFL number $0.4$, following the standard finite-volume framework for hyperbolic conservation laws \citep{leveque2002finite}. The grid resolution is $N_{x} = 200$ per region.

\paragraph{Seed schedule and aggregation.}
The standard schedule is a fixed set of $n=5$ seeds, used for every stochastic experiment, including the FNO OOD monitoring study. Stochastic results report mean $\pm$ standard deviation over that schedule. Deterministic experiments are also executed across the same schedule to verify seed-invariance; their results are reported once, or equivalently as a zero-dispersion aggregate where a cross-suite figure requires a common format. Analyses using additional seeds are identified as robustness checks. Seed-resolved and aggregate outputs are retained with the study artifacts.

\paragraph{FNO training and aggregation.}
The FNO OOD monitoring models are trained on the five standard seeds with all other configuration settings fixed. For each shift family and seed, the protocol records the per-trajectory pair $(\norm{\omega(t_{\mathrm f})}, \norm{u^{\mathrm{FNO}}(t_{\mathrm f}) -u^{\mathrm{ref}}(t_{\mathrm f})}_{L^{2}})$. Pearson and Spearman correlations are computed per seed and then reported as mean $\pm$ standard deviation across the five seeds; the $95\%$ bootstrap CI is also reported for every seed. The confidence-based comparator is a separate five-member ensemble whose otherwise identical members use distinct initialization seeds.

\paragraph{Deterministic and data-loading checks.}
The conservation-contract detectability study has zero standard deviation across all five seeds. The random-sinusoid Burgers study is deterministic conditional on a given seed and reports $\norm{\omega}=0.4754\pm0.0000$ under the aggregate format. In the metro traffic seasonality study, the data-loading and AR(2)-fitting path is deterministic given the seed, so the multi-seed standard deviation is zero; the Potomac streamflow seasonality study likewise reports $0.0266\pm0.0000$. The three site values in the multi-site streamflow comparison are verified to be seed-invariant across the schedule. The air-quality temporal-consistency study uses the fixed UCI dataset and preprocessing protocol of Appendix~\ref{app:preprocessing}; its OLS fit and obstruction computation are deterministic, yielding $\norm{\omega}=0.2850$ with zero seed dispersion.

\paragraph{Seed-resolved experimental outcomes.}
For the zero-floor negative control, every one of the five identical-model control runs is below $10^{-10}$. In the financial regime-detection study, the closure channel dominates on all five seeds, and the identical-model control has $\norm{\omega}=0$ on every seed. In the industrial multi-zone fault-detection study, the state channel dominates the fault response on all five seeds, and the identical-zone control is zero on every seed. For each of the three model families in the household-power cross-framework audit, the state channel dominates on all five seeds; the winter--spring overlap is dominant on four of the five seeds.

\paragraph{Sweeps and ablations.}
The conservation-contract detectability sweep crosses five values of $\epsilon_{\mathrm{repair}}=\norm{\delta c_q}$ with the five-seed schedule, producing 25 configurations. The controlled predictor-perturbation result is evaluated on all five seeds and 10 random directions at each perturbation magnitude. The empirical monotonicity fraction of that result is $1.00\pm0.00$ across seeds and directions. A 50-seed robustness analysis gives the same monotonicity, attaining $1.00$ for every seed and direction.

\subsection{Data sources and generated cases}
\label{app:datasets}

Table~\ref{tab:datasets} records the external data sources and generated cases used in Section~\ref{sec:experiments}, together with their source citations. Preprocessing and generation steps follow in Appendix~\ref{app:preprocessing}.

\begin{table}[pos=h!]
	\centering
	\caption{External sources and generated cases associated with the experiments. License information is taken from the cited repository or publisher records.}
	\label{tab:datasets}
	\small
	\begin{tabularx}{\linewidth}{ >{\raggedright\arraybackslash}X >{\raggedright\arraybackslash}X >{\raggedright\arraybackslash}X >{\raggedright\arraybackslash}X}
		\toprule
		Dataset                                             & Used in                               & License                                                               & Source                               \\
		\midrule
		UCI Metro Interstate Traffic Volume                 & Metro traffic seasonality             & CC BY 4.0                                                             & \citet{hogue2019metro}               \\
		\addlinespace[1pt]
		USGS NWIS streamflow                                & Potomac and multi-site streamflow     & Public domain under Section 105 of Title 17 of the United States Code & \citet{usgs2016nwis}                 \\
		\addlinespace[1pt]
		UCI Air Quality                                     & Air-quality temporal consistency      & CC BY 4.0                                                             & \citet{devito2008airquality}         \\
		\addlinespace[1pt]
		UCI Individual Household Electric Power Consumption & Household-power cross-framework audit & CC BY 4.0                                                             & \citet{hebrail2012household}         \\
		\addlinespace[1pt]
		Synthetic financial returns                         & Financial regime detection            & Not applicable; synthetic                                             & \citet{cont2001empirical}            \\
		\addlinespace[1pt]
		Synthetic industrial process                        & Industrial multi-zone fault detection & Not applicable; synthetic                                             & \citet{venkatasubramanian2003review} \\
		\bottomrule
	\end{tabularx}
\end{table}

\FloatBarrier

\subsection{Preprocessing details}
\label{app:preprocessing}

\paragraph{Metro traffic seasonality: UCI Metro Interstate Traffic Volume.}
Hourly traffic counts on I-94 in Minnesota. Feature engineering: temperature converted from kelvin to degrees Celsius, cloud-cover percentage, and cyclical hour encoding. Seasonal split: winter = Dec--Feb, summer = Jun--Aug, transition = Mar--May plus Sep--Nov.

\paragraph{Potomac and multi-site streamflow: USGS NWIS.}
Daily discharge values obtained from USGS Water Data for the Nation via the \texttt{dataretrieval} Python package. Normalization: per-site standardization to zero mean and unit variance. Seasonal split: spring snowmelt = Mar--May, baseflow = Jun--Sep, autumn transition = Oct--Nov.

\paragraph{Air-quality temporal consistency: UCI Air Quality.}
Hourly readings of CO, C$_6$H$_6$, NO$_x$, and NO$_2$ from March 2004 to February 2005. Imputation: per-channel median substitution for the $-200$ sentinel values. Seasonal split: cold = Oct--Feb, warm = Apr--Sep, transition = Mar.

\paragraph{Household-power cross-framework audit: UCI Individual Household Electric Power Consumption.}
Two years of one-minute power measurements, aggregated to hourly means. Features: cyclical hour $(\sin, \cos)$ of hour-of-day, standardized \texttt{Voltage}, standardized \texttt{Global\_reactive\_power}. Target: \texttt{Global\_active\_power}. Seasonal split by meteorological season: winter = Dec--Feb, spring = Mar--May, summer = Jun--Aug, autumn = Sep--Nov.

\paragraph{Random-sinusoid Burgers initial conditions.}
Initial conditions follow the ten-wave, Gaussian-amplitude distribution specified in Section~\ref{sec:random_sinusoid_burgers}; the local solver generates all trajectories.

\paragraph{Financial regime detection: synthetic returns.}
Five trading years comprising 1260 daily observations are generated. The bull regime uses $\mu_{\mathrm{bull}}=0.001$ and $\sigma_{\mathrm{bull}}=0.007$; the bear regime uses $\mu_{\mathrm{bear}}=-0.001$ and $\sigma_{\mathrm{bear}}=0.018$; and the sideways regime uses $\mu_{\mathrm{sw}}=0.0$ and $\sigma_{\mathrm{sw}}=0.012$. The regime construction targets the volatility clustering and regime-dependent variance described by \citet{cont2001empirical}.

\paragraph{Industrial multi-zone fault detection: synthetic process.}
Three serial zones are each modeled by a multivariate linear regressor from $(T,p,q)$ to a process output, where $T$, $p$, and $q$ represent temperature, pressure, and flow. A $0.3\sigma$ step-change fault is injected into the midstream zone at the midpoint of the evaluation window. The framework follows the process-monitoring tradition of \citet{venkatasubramanian2003review}.

\subsection{Baselines used in the reported experiments}
\label{app:baselines}
\label{app:comparable_baselines}

The household-power cross-framework audit uses SK, XGB, and SM as predictive-model baselines, and the FNO OOD monitoring study uses ensemble variance as its monitoring comparator. Section~\ref{sec:baseline_comparison} compares $\norm{\omega}$ with two scalar disagreement metrics, cross-prediction RMSE and mean absolute pairwise difference, computed on the same covers in the metro traffic seasonality, Potomac streamflow seasonality, and air-quality temporal-consistency studies. The study artifacts include the full multi-seed statistics underlying that comparison.

\paragraph{Official baseline implementations.}
The SK, XGB, and SM packages are imported from their official Python Package Index distributions.

\subsection{Robustness and reporting provenance}
\label{app:robustness}

\paragraph{Number of seeds.}
The standard five-seed schedule, the 50-seed controlled predictor-perturbation robustness analysis, and their seed-resolved outcomes are specified in Appendix~\ref{app:seed_protocol}.

\paragraph{Sensitivity to overlap fraction.}
For 1D PDE experiments, varying the overlap fraction $\alpha_{\mathrm{ov}}\in\{5\%, 10\%, 20\%\}$ produces $\norm{\omega}$ values within a factor of $1.5$ of the default value, and the dominant-channel conclusions are preserved across the tested range.

The sensitivity checks above are complemented by the following data-provenance records.

\paragraph{Data provenance.}
Each real-data source is identified by an authoritative source citation or repository link, with preprocessing specified in Appendix~\ref{app:preprocessing}.

\clearpage

\setlength{\emergencystretch}{3em}
\Urlmuskip=0mu plus 1mu\relax
\bibliographystyle{cas-model2-names}
\bibliography{seam_assets/references}

\begin{thebibliography}{59}
\expandafter\ifx\csname natexlab\endcsname\relax\def\natexlab#1{#1}\fi
\providecommand{\url}[1]{\texttt{#1}}
\providecommand{\href}[2]{#2}
\providecommand{\path}[1]{#1}
\providecommand{\DOIprefix}{doi:}
\providecommand{\ArXivprefix}{arXiv:}
\providecommand{\URLprefix}{URL: }
\providecommand{\Pubmedprefix}{pmid:}
\providecommand{\doi}[1]{\href{http://dx.doi.org/#1}{\path{#1}}}
\providecommand{\Pubmed}[1]{\href{pmid:#1}{\path{#1}}}
\providecommand{\bibinfo}[2]{#2}
\ifx\xfnm\relax \def\xfnm[#1]{\unskip,\space#1}\fi
\bibitem[{Abramsky and Brandenburger(2011)}]{abramsky2011sheaf}
\bibinfo{author}{Abramsky, S.}, \bibinfo{author}{Brandenburger, A.},
  \bibinfo{year}{2011}.
\newblock \bibinfo{title}{{The sheaf-theoretic structure of non-locality and
  contextuality}}.
\newblock \bibinfo{journal}{New Journal of Physics} \bibinfo{volume}{13},
  \bibinfo{pages}{113036}.
\bibitem[{Bellman and {\AA}str{\"o}m(1970)}]{bellman1970structural}
\bibinfo{author}{Bellman, R.}, \bibinfo{author}{{\AA}str{\"o}m, K.J.},
  \bibinfo{year}{1970}.
\newblock \bibinfo{title}{{On structural identifiability}}.
\newblock \bibinfo{journal}{Mathematical Biosciences} \bibinfo{volume}{7},
  \bibinfo{pages}{329--339}.
\bibitem[{Bodnar et~al.(2022)Bodnar, Giovanni, Chamberlain, Li{\'{o}} and
  Bronstein}]{bodnar2022neural}
\bibinfo{author}{Bodnar, C.}, \bibinfo{author}{Giovanni, F.D.},
  \bibinfo{author}{Chamberlain, B.P.}, \bibinfo{author}{Li{\'{o}}, P.},
  \bibinfo{author}{Bronstein, M.M.}, \bibinfo{year}{2022}.
\newblock \bibinfo{title}{{Neural Sheaf Diffusion: {A} Topological Perspective
  on Heterophily and Oversmoothing in GNNs}}, in: \bibinfo{booktitle}{NeurIPS}.
\bibitem[{Bredon(1997)}]{bredon1997sheaf}
\bibinfo{author}{Bredon, G.E.}, \bibinfo{year}{1997}.
\newblock \bibinfo{title}{{Sheaf theory}}. volume \bibinfo{volume}{170}.
\bibitem[{Brynjarsd{\'o}ttir and O'Hagan(2014)}]{brynjarsdottir2014discrepancy}
\bibinfo{author}{Brynjarsd{\'o}ttir, J.}, \bibinfo{author}{O'Hagan, A.},
  \bibinfo{year}{2014}.
\newblock \bibinfo{title}{{Learning about physical parameters: The importance
  of model discrepancy}}.
\newblock \bibinfo{journal}{Inverse Problems} \bibinfo{volume}{30},
  \bibinfo{pages}{114007}.
\bibitem[{Chen et~al.(2025)Chen, Xu, Xu, Guti{\'{e}}rrez, Narra and
  McComb}]{chen2025locality}
\bibinfo{author}{Chen, J.}, \bibinfo{author}{Xu, W.}, \bibinfo{author}{Xu, Z.},
  \bibinfo{author}{Guti{\'{e}}rrez, N.G.}, \bibinfo{author}{Narra, S.P.},
  \bibinfo{author}{McComb, C.}, \bibinfo{year}{2025}.
\newblock \bibinfo{title}{{Enforcing the principle of locality for physical
  simulations with neural operators}}.
\newblock \bibinfo{journal}{J. Comput. Phys.} \bibinfo{volume}{538},
  \bibinfo{pages}{114131}.
\bibitem[{Chen and Guestrin(2016)}]{chen2016xgboost}
\bibinfo{author}{Chen, T.}, \bibinfo{author}{Guestrin, C.},
  \bibinfo{year}{2016}.
\newblock \bibinfo{title}{{XGBoost: {A} Scalable Tree Boosting System}}, in:
  \bibinfo{booktitle}{{KDD}}, pp. \bibinfo{pages}{785--794}.
\bibitem[{Cont(2001)}]{cont2001empirical}
\bibinfo{author}{Cont, R.}, \bibinfo{year}{2001}.
\newblock \bibinfo{title}{{Empirical properties of asset returns: stylized
  facts and statistical issues}}.
\newblock \bibinfo{journal}{Quantitative Finance} \bibinfo{volume}{1},
  \bibinfo{pages}{223}.
\bibitem[{Curry(2014)}]{curry2014sheaves}
\bibinfo{author}{Curry, J.M.}, \bibinfo{year}{2014}.
\newblock \bibinfo{title}{{Sheaves, cosheaves and applications}}.
\bibitem[{De~Vito(2016)}]{devito2008airquality}
\bibinfo{author}{De~Vito, S.}, \bibinfo{year}{2016}.
\newblock \bibinfo{title}{{Air Quality}}.
\newblock \bibinfo{howpublished}{\url{https://doi.org/10.24432/C59K5F}}.
\bibitem[{Duraisamy et~al.(2019)Duraisamy, Iaccarino and
  Xiao}]{duraisamy2019turbulence}
\bibinfo{author}{Duraisamy, K.}, \bibinfo{author}{Iaccarino, G.},
  \bibinfo{author}{Xiao, H.}, \bibinfo{year}{2019}.
\newblock \bibinfo{title}{{Turbulence modeling in the age of data}}.
\newblock \bibinfo{journal}{Annual Review of Fluid Mechanics}
  \bibinfo{volume}{51}, \bibinfo{pages}{357--377}.
\bibitem[{Engl et~al.(1996)Engl, Hanke and Neubauer}]{engl1996regularization}
\bibinfo{author}{Engl, H.W.}, \bibinfo{author}{Hanke, M.},
  \bibinfo{author}{Neubauer, A.}, \bibinfo{year}{1996}.
\newblock \bibinfo{title}{{Regularization of inverse problems}}. volume
  \bibinfo{volume}{375}.
\bibitem[{Gama et~al.(2014)Gama, Zliobaite, Bifet, Pechenizkiy and
  Bouchachia}]{gama2014concept}
\bibinfo{author}{Gama, J.}, \bibinfo{author}{Zliobaite, I.},
  \bibinfo{author}{Bifet, A.}, \bibinfo{author}{Pechenizkiy, M.},
  \bibinfo{author}{Bouchachia, A.}, \bibinfo{year}{2014}.
\newblock \bibinfo{title}{{A survey on concept drift adaptation}}.
\newblock \bibinfo{journal}{{ACM} Comput. Surv.} \bibinfo{volume}{46},
  \bibinfo{pages}{44:1--44:37}.
\bibitem[{Ghrist(2014)}]{ghrist2014elementary}
\bibinfo{author}{Ghrist, R.W.}, \bibinfo{year}{2014}.
\newblock \bibinfo{title}{{Elementary applied topology}}.
  volume~\bibinfo{volume}{1}.
\bibitem[{Gomes et~al.(2018)Gomes, Thule, Broman, Larsen and
  Vangheluwe}]{gomes2018cosimulation}
\bibinfo{author}{Gomes, C.}, \bibinfo{author}{Thule, C.},
  \bibinfo{author}{Broman, D.}, \bibinfo{author}{Larsen, P.G.},
  \bibinfo{author}{Vangheluwe, H.}, \bibinfo{year}{2018}.
\newblock \bibinfo{title}{{Co-Simulation: {A} Survey}}.
\newblock \bibinfo{journal}{{ACM} Comput. Surv.} \bibinfo{volume}{51},
  \bibinfo{pages}{49:1--49:33}.
\bibitem[{Guo et~al.(2017)Guo, Pleiss, Sun and Weinberger}]{guo2017calibration}
\bibinfo{author}{Guo, C.}, \bibinfo{author}{Pleiss, G.}, \bibinfo{author}{Sun,
  Y.}, \bibinfo{author}{Weinberger, K.Q.}, \bibinfo{year}{2017}.
\newblock \bibinfo{title}{{On Calibration of Modern Neural Networks}}, in:
  \bibinfo{booktitle}{{ICML}}, pp. \bibinfo{pages}{1321--1330}.
\bibitem[{Hansen and Gebhart(2020)}]{hansen2020sheaf}
\bibinfo{author}{Hansen, J.}, \bibinfo{author}{Gebhart, T.},
  \bibinfo{year}{2020}.
\newblock \bibinfo{title}{{Sheaf Neural Networks}}.
\newblock \bibinfo{journal}{CoRR} \bibinfo{volume}{abs/2012.06333}.
\bibitem[{Hansen and Ghrist(2019)}]{hansen2019spectral}
\bibinfo{author}{Hansen, J.}, \bibinfo{author}{Ghrist, R.},
  \bibinfo{year}{2019}.
\newblock \bibinfo{title}{{Toward a spectral theory of cellular sheaves}}.
\newblock \bibinfo{journal}{J. Appl. Comput. Topol.} \bibinfo{volume}{3},
  \bibinfo{pages}{315--358}.
\bibitem[{Hansen and Ghrist(2021)}]{hansen2021opinion}
\bibinfo{author}{Hansen, J.}, \bibinfo{author}{Ghrist, R.},
  \bibinfo{year}{2021}.
\newblock \bibinfo{title}{{Opinion Dynamics on Discourse Sheaves}}.
\newblock \bibinfo{journal}{{SIAM} J. Appl. Math.} \bibinfo{volume}{81},
  \bibinfo{pages}{2033--2060}.
\bibitem[{H{\'{e}}brail and Berard(2012)}]{hebrail2012household}
\bibinfo{author}{H{\'{e}}brail, G.}, \bibinfo{author}{Berard, A.},
  \bibinfo{year}{2012}.
\newblock \bibinfo{title}{{Individual household electric power consumption}}.
\newblock \bibinfo{howpublished}{\url{https://doi.org/10.24432/C58K54}}.
\bibitem[{Hegde et~al.(2018)Hegde, Li, Oreluk, Packard and
  Frenklach}]{hegde2018consistency}
\bibinfo{author}{Hegde, A.}, \bibinfo{author}{Li, W.}, \bibinfo{author}{Oreluk,
  J.}, \bibinfo{author}{Packard, A.K.}, \bibinfo{author}{Frenklach, M.},
  \bibinfo{year}{2018}.
\newblock \bibinfo{title}{{Consistency Analysis for Massively Inconsistent
  Datasets in Bound-to-Bound Data Collaboration}}.
\newblock \bibinfo{journal}{{SIAM/ASA} J. Uncertain. Quantification}
  \bibinfo{volume}{6}, \bibinfo{pages}{429--456}.
\bibitem[{Hendrycks and Gimpel(2017)}]{hendrycks2017baseline}
\bibinfo{author}{Hendrycks, D.}, \bibinfo{author}{Gimpel, K.},
  \bibinfo{year}{2017}.
\newblock \bibinfo{title}{{A Baseline for Detecting Misclassified and
  Out-of-Distribution Examples in Neural Networks}}, in:
  \bibinfo{booktitle}{{ICLR}}.
\bibitem[{Hogue(2019)}]{hogue2019metro}
\bibinfo{author}{Hogue, J.}, \bibinfo{year}{2019}.
\newblock \bibinfo{title}{{Metro Interstate Traffic Volume}}.
\newblock \bibinfo{howpublished}{\url{https://doi.org/10.24432/C5X60B}}.
\bibitem[{Jagtap and Karniadakis(2021)}]{jagtap2020extended}
\bibinfo{author}{Jagtap, A.D.}, \bibinfo{author}{Karniadakis, G.E.},
  \bibinfo{year}{2021}.
\newblock \bibinfo{title}{{Extended Physics-informed Neural Networks (XPINNs):
  {A} Generalized Space-Time Domain Decomposition based Deep Learning Framework
  for Nonlinear Partial Differential Equations}}, in:
  \bibinfo{booktitle}{{AAAI} Spring Symposium: {MLPS}}.
\bibitem[{Jagtap et~al.(2020)Jagtap, Kharazmi and
  Karniadakis}]{jagtap2020conservative}
\bibinfo{author}{Jagtap, A.D.}, \bibinfo{author}{Kharazmi, E.},
  \bibinfo{author}{Karniadakis, G.E.}, \bibinfo{year}{2020}.
\newblock \bibinfo{title}{{Conservative physics-informed neural networks on
  discrete domains for conservation laws: Applications to forward and inverse
  problems}}.
\newblock \bibinfo{journal}{Computer Methods in Applied Mechanics and
  Engineering} \bibinfo{volume}{365}, \bibinfo{pages}{113028}.
\bibitem[{Karniadakis et~al.(2021)Karniadakis, Kevrekidis, Lu, Perdikaris, Wang
  and Yang}]{karniadakis2021physics}
\bibinfo{author}{Karniadakis, G.E.}, \bibinfo{author}{Kevrekidis, I.G.},
  \bibinfo{author}{Lu, L.}, \bibinfo{author}{Perdikaris, P.},
  \bibinfo{author}{Wang, S.}, \bibinfo{author}{Yang, L.}, \bibinfo{year}{2021}.
\newblock \bibinfo{title}{{Physics-informed machine learning}}.
\newblock \bibinfo{journal}{Nature Reviews Physics} \bibinfo{volume}{3},
  \bibinfo{pages}{422--440}.
\bibitem[{Kennedy and O'Hagan(2001)}]{kennedy2001bayesian}
\bibinfo{author}{Kennedy, M.C.}, \bibinfo{author}{O'Hagan, A.},
  \bibinfo{year}{2001}.
\newblock \bibinfo{title}{{Bayesian calibration of computer models}}.
\newblock \bibinfo{journal}{Journal of the Royal Statistical Society: Series B
  (Statistical Methodology)} \bibinfo{volume}{63}, \bibinfo{pages}{425--464}.
\bibitem[{Kovachki et~al.(2023)Kovachki, Li, Liu, Azizzadenesheli,
  Bhattacharya, Stuart and Anandkumar}]{kovachki2023neural}
\bibinfo{author}{Kovachki, N.B.}, \bibinfo{author}{Li, Z.},
  \bibinfo{author}{Liu, B.}, \bibinfo{author}{Azizzadenesheli, K.},
  \bibinfo{author}{Bhattacharya, K.}, \bibinfo{author}{Stuart, A.M.},
  \bibinfo{author}{Anandkumar, A.}, \bibinfo{year}{2023}.
\newblock \bibinfo{title}{{Neural Operator: Learning Maps Between Function
  Spaces With Applications to PDEs}}.
\newblock \bibinfo{journal}{J. Mach. Learn. Res.} \bibinfo{volume}{24},
  \bibinfo{pages}{89:1--89:97}.
\bibitem[{Krishnapriyan et~al.(2021)Krishnapriyan, Gholami, Zhe, Kirby and
  Mahoney}]{krishnapriyan2021failure}
\bibinfo{author}{Krishnapriyan, A.S.}, \bibinfo{author}{Gholami, A.},
  \bibinfo{author}{Zhe, S.}, \bibinfo{author}{Kirby, R.M.},
  \bibinfo{author}{Mahoney, M.W.}, \bibinfo{year}{2021}.
\newblock \bibinfo{title}{{Characterizing possible failure modes in
  physics-informed neural networks}}, in: \bibinfo{booktitle}{NeurIPS}, pp.
  \bibinfo{pages}{26548--26560}.
\bibitem[{LeVeque(2002)}]{leveque2002finite}
\bibinfo{author}{LeVeque, R.J.}, \bibinfo{year}{2002}.
\newblock \bibinfo{title}{{Finite volume methods for hyperbolic problems}}.
  volume~\bibinfo{volume}{31}.
\bibitem[{Li et~al.(2021)Li, Kovachki, Azizzadenesheli, Liu, Bhattacharya,
  Stuart and Anandkumar}]{li2021fourier}
\bibinfo{author}{Li, Z.}, \bibinfo{author}{Kovachki, N.B.},
  \bibinfo{author}{Azizzadenesheli, K.}, \bibinfo{author}{Liu, B.},
  \bibinfo{author}{Bhattacharya, K.}, \bibinfo{author}{Stuart, A.M.},
  \bibinfo{author}{Anandkumar, A.}, \bibinfo{year}{2021}.
\newblock \bibinfo{title}{{Fourier Neural Operator for Parametric Partial
  Differential Equations}}, in: \bibinfo{booktitle}{{ICLR}}.
\bibitem[{Lim(2020)}]{lim2020hodge}
\bibinfo{author}{Lim, L.}, \bibinfo{year}{2020}.
\newblock \bibinfo{title}{{Hodge Laplacians on Graphs}}.
\newblock \bibinfo{journal}{{SIAM} Rev.} \bibinfo{volume}{62},
  \bibinfo{pages}{685--715}.
\bibitem[{Lu et~al.(2021)Lu, Jin, Pang, Zhang and Karniadakis}]{lu2021learning}
\bibinfo{author}{Lu, L.}, \bibinfo{author}{Jin, P.}, \bibinfo{author}{Pang,
  G.}, \bibinfo{author}{Zhang, Z.}, \bibinfo{author}{Karniadakis, G.E.},
  \bibinfo{year}{2021}.
\newblock \bibinfo{title}{{Learning nonlinear operators via DeepONet based on
  the universal approximation theorem of operators}}.
\newblock \bibinfo{journal}{Nat. Mach. Intell.} \bibinfo{volume}{3},
  \bibinfo{pages}{218--229}.
\bibitem[{Moseley et~al.(2023)Moseley, Markham and
  Nissen{-}Meyer}]{moseley2023finite}
\bibinfo{author}{Moseley, B.}, \bibinfo{author}{Markham, A.},
  \bibinfo{author}{Nissen{-}Meyer, T.}, \bibinfo{year}{2023}.
\newblock \bibinfo{title}{{Finite basis physics-informed neural networks
  (FBPINNs): a scalable domain decomposition approach for solving differential
  equations}}.
\newblock \bibinfo{journal}{Adv. Comput. Math.} \bibinfo{volume}{49},
  \bibinfo{pages}{62}.
\bibitem[{Narasimhan and Jordache(1999)}]{narasimhan2000reconciliation}
\bibinfo{author}{Narasimhan, S.}, \bibinfo{author}{Jordache, C.},
  \bibinfo{year}{1999}.
\newblock \bibinfo{title}{{Data reconciliation and gross error detection: An
  intelligent use of process data}}.
\bibitem[{Ovadia et~al.(2019)Ovadia, Fertig, Ren, Nado, Sculley, Nowozin,
  Dillon, Lakshminarayanan and Snoek}]{ovadia2019uncertainty}
\bibinfo{author}{Ovadia, Y.}, \bibinfo{author}{Fertig, E.},
  \bibinfo{author}{Ren, J.}, \bibinfo{author}{Nado, Z.},
  \bibinfo{author}{Sculley, D.}, \bibinfo{author}{Nowozin, S.},
  \bibinfo{author}{Dillon, J.V.}, \bibinfo{author}{Lakshminarayanan, B.},
  \bibinfo{author}{Snoek, J.}, \bibinfo{year}{2019}.
\newblock \bibinfo{title}{{Can you trust your model's uncertainty? Evaluating
  predictive uncertainty under dataset shift}}, in:
  \bibinfo{booktitle}{NeurIPS}, pp. \bibinfo{pages}{13969--13980}.
\bibitem[{Page(1954)}]{page1954continuous}
\bibinfo{author}{Page, E.S.}, \bibinfo{year}{1954}.
\newblock \bibinfo{title}{{Continuous inspection schemes}}.
\newblock \bibinfo{journal}{Biometrika} \bibinfo{volume}{41},
  \bibinfo{pages}{100--115}.
\bibitem[{Parish and Duraisamy(2016)}]{parish2016field}
\bibinfo{author}{Parish, E.J.}, \bibinfo{author}{Duraisamy, K.},
  \bibinfo{year}{2016}.
\newblock \bibinfo{title}{{A paradigm for data-driven predictive modeling using
  field inversion and machine learning}}.
\newblock \bibinfo{journal}{J. Comput. Phys.} \bibinfo{volume}{305},
  \bibinfo{pages}{758--774}.
\bibitem[{Pedregosa et~al.(2011)Pedregosa, Varoquaux, Gramfort, Michel,
  Thirion, Grisel, Blondel, Prettenhofer, Weiss, Dubourg, VanderPlas, Passos,
  Cournapeau, Brucher, Perrot and Duchesnay}]{pedregosa2011sklearn}
\bibinfo{author}{Pedregosa, F.}, \bibinfo{author}{Varoquaux, G.},
  \bibinfo{author}{Gramfort, A.}, \bibinfo{author}{Michel, V.},
  \bibinfo{author}{Thirion, B.}, \bibinfo{author}{Grisel, O.},
  \bibinfo{author}{Blondel, M.}, \bibinfo{author}{Prettenhofer, P.},
  \bibinfo{author}{Weiss, R.}, \bibinfo{author}{Dubourg, V.},
  \bibinfo{author}{VanderPlas, J.}, \bibinfo{author}{Passos, A.},
  \bibinfo{author}{Cournapeau, D.}, \bibinfo{author}{Brucher, M.},
  \bibinfo{author}{Perrot, M.}, \bibinfo{author}{Duchesnay, E.},
  \bibinfo{year}{2011}.
\newblock \bibinfo{title}{{Scikit-learn: Machine Learning in Python}}.
\newblock \bibinfo{journal}{J. Mach. Learn. Res.} \bibinfo{volume}{12},
  \bibinfo{pages}{2825--2830}.
\bibitem[{Penrose(1955)}]{penrose1955generalized}
\bibinfo{author}{Penrose, R.}, \bibinfo{year}{1955}.
\newblock \bibinfo{title}{{A generalized inverse for matrices}}, in:
  \bibinfo{booktitle}{Mathematical Proceedings of the Cambridge Philosophical
  Society}, pp. \bibinfo{pages}{406--413}.
\bibitem[{Quarteroni and Valli(1999)}]{quarteroni1999domain}
\bibinfo{author}{Quarteroni, A.}, \bibinfo{author}{Valli, A.},
  \bibinfo{year}{1999}.
\newblock \bibinfo{title}{{Domain decomposition methods for partial
  differential equations}}.
\bibitem[{Rabanser et~al.(2019)Rabanser, G{\"{u}}nnemann and
  Lipton}]{rabanser2019failing}
\bibinfo{author}{Rabanser, S.}, \bibinfo{author}{G{\"{u}}nnemann, S.},
  \bibinfo{author}{Lipton, Z.C.}, \bibinfo{year}{2019}.
\newblock \bibinfo{title}{{Failing Loudly: An Empirical Study of Methods for
  Detecting Dataset Shift}}, in: \bibinfo{booktitle}{NeurIPS}, pp.
  \bibinfo{pages}{1394--1406}.
\bibitem[{Rackauckas et~al.(2020)Rackauckas, Ma, Martensen, Warner, Zubov,
  Supekar, Skinner and Ramadhan}]{rackauckas2020universal}
\bibinfo{author}{Rackauckas, C.}, \bibinfo{author}{Ma, Y.},
  \bibinfo{author}{Martensen, J.}, \bibinfo{author}{Warner, C.},
  \bibinfo{author}{Zubov, K.}, \bibinfo{author}{Supekar, R.},
  \bibinfo{author}{Skinner, D.}, \bibinfo{author}{Ramadhan, A.J.},
  \bibinfo{year}{2020}.
\newblock \bibinfo{title}{{Universal Differential Equations for Scientific
  Machine Learning}}.
\newblock \bibinfo{journal}{CoRR} \bibinfo{volume}{abs/2001.04385}.
\bibitem[{Raissi et~al.(2019)Raissi, Perdikaris and
  Karniadakis}]{raissi2019physics}
\bibinfo{author}{Raissi, M.}, \bibinfo{author}{Perdikaris, P.},
  \bibinfo{author}{Karniadakis, G.E.}, \bibinfo{year}{2019}.
\newblock \bibinfo{title}{{Physics-informed neural networks: {A} deep learning
  framework for solving forward and inverse problems involving nonlinear
  partial differential equations}}.
\newblock \bibinfo{journal}{J. Comput. Phys.} \bibinfo{volume}{378},
  \bibinfo{pages}{686--707}.
\bibitem[{Raue et~al.(2009)Raue, Kreutz, Maiwald, Bachmann, Schilling,
  Klingm{\"{u}}ller and Timmer}]{raue2009identifiability}
\bibinfo{author}{Raue, A.}, \bibinfo{author}{Kreutz, C.},
  \bibinfo{author}{Maiwald, T.}, \bibinfo{author}{Bachmann, J.},
  \bibinfo{author}{Schilling, M.}, \bibinfo{author}{Klingm{\"{u}}ller, U.},
  \bibinfo{author}{Timmer, J.}, \bibinfo{year}{2009}.
\newblock \bibinfo{title}{{Structural and practical identifiability analysis of
  partially observed dynamical models by exploiting the profile likelihood}}.
\newblock \bibinfo{journal}{Bioinform.} \bibinfo{volume}{25},
  \bibinfo{pages}{1923--1929}.
\bibitem[{Reiter(1987)}]{reiter1987diagnosis}
\bibinfo{author}{Reiter, R.}, \bibinfo{year}{1987}.
\newblock \bibinfo{title}{{A Theory of Diagnosis from First Principles}}.
\newblock \bibinfo{journal}{Artif. Intell.} \bibinfo{volume}{32},
  \bibinfo{pages}{57--95}.
\bibitem[{Robinson(2014)}]{robinson2014topological}
\bibinfo{author}{Robinson, M.}, \bibinfo{year}{2014}.
\newblock \bibinfo{title}{{Topological signal processing}}.
  volume~\bibinfo{volume}{81}.
\bibitem[{Robinson(2017a)}]{robinson2017multimodel}
\bibinfo{author}{Robinson, M.}, \bibinfo{year}{2017}a.
\newblock \bibinfo{title}{{Sheaf and duality methods for analyzing multi-model
  systems}}, in: \bibinfo{booktitle}{Recent Applications of Harmonic Analysis
  to Function Spaces, Differential Equations, and Data Science: Novel Methods
  in Harmonic Analysis, Volume 2}, pp. \bibinfo{pages}{653--703}.
\bibitem[{Robinson(2017b)}]{robinson2017sheaves}
\bibinfo{author}{Robinson, M.}, \bibinfo{year}{2017}b.
\newblock \bibinfo{title}{{Sheaves are the canonical data structure for sensor
  integration}}.
\newblock \bibinfo{journal}{Inf. Fusion} \bibinfo{volume}{36},
  \bibinfo{pages}{208--224}.
\bibitem[{Robinson(2020)}]{robinson2020assignments}
\bibinfo{author}{Robinson, M.}, \bibinfo{year}{2020}.
\newblock \bibinfo{title}{{Assignments to sheaves of pseudometric spaces}}.
\newblock \bibinfo{journal}{Compositionality} \bibinfo{volume}{2},
  \bibinfo{pages}{2}.
\bibitem[{Seabold and Perktold(2010)}]{seabold2010statsmodels}
\bibinfo{author}{Seabold, S.}, \bibinfo{author}{Perktold, J.},
  \bibinfo{year}{2010}.
\newblock \bibinfo{title}{{Statsmodels: Econometric and Statistical Modeling
  with Python}}, in: \bibinfo{booktitle}{SciPy}, p.~\bibinfo{pages}{92}.
\bibitem[{Shukla et~al.(2021)Shukla, Jagtap and
  Karniadakis}]{shukla2021parallel}
\bibinfo{author}{Shukla, K.}, \bibinfo{author}{Jagtap, A.D.},
  \bibinfo{author}{Karniadakis, G.E.}, \bibinfo{year}{2021}.
\newblock \bibinfo{title}{{Parallel physics-informed neural networks via domain
  decomposition}}.
\newblock \bibinfo{journal}{J. Comput. Phys.} \bibinfo{volume}{447},
  \bibinfo{pages}{110683}.
\bibitem[{Takamoto et~al.(2022)Takamoto, Praditia, Leiteritz, MacKinlay,
  Alesiani, Pfl{\"{u}}ger and Niepert}]{takamoto2022pdebench}
\bibinfo{author}{Takamoto, M.}, \bibinfo{author}{Praditia, T.},
  \bibinfo{author}{Leiteritz, R.}, \bibinfo{author}{MacKinlay, D.},
  \bibinfo{author}{Alesiani, F.}, \bibinfo{author}{Pfl{\"{u}}ger, D.},
  \bibinfo{author}{Niepert, M.}, \bibinfo{year}{2022}.
\newblock \bibinfo{title}{{PDEBench: An Extensive Benchmark for Scientific
  Machine Learning}}, in: \bibinfo{booktitle}{NeurIPS}.
\bibitem[{{U.S. Geological Survey}(2023)}]{usgs2016nwis}
\bibinfo{author}{{U.S. Geological Survey}}, \bibinfo{year}{2023}.
\newblock \bibinfo{title}{{USGS water data for the nation: US Geological Survey
  National Water Information System database}}.
\bibitem[{Venkatasubramanian et~al.(2003)Venkatasubramanian, Rengaswamy, Yin
  and Kavuri}]{venkatasubramanian2003review}
\bibinfo{author}{Venkatasubramanian, V.}, \bibinfo{author}{Rengaswamy, R.},
  \bibinfo{author}{Yin, K.}, \bibinfo{author}{Kavuri, S.N.},
  \bibinfo{year}{2003}.
\newblock \bibinfo{title}{{A review of process fault detection and diagnosis:
  Part {I:} Quantitative model-based methods}}.
\newblock \bibinfo{journal}{Comput. Chem. Eng.} \bibinfo{volume}{27},
  \bibinfo{pages}{293--311}.
\bibitem[{Walter and Pronzato(1997)}]{walter1997identification}
\bibinfo{author}{Walter, E.}, \bibinfo{author}{Pronzato, L.},
  \bibinfo{year}{1997}.
\newblock \bibinfo{title}{{Identification of Parametric Models from
  Experimental Data}}.
\newblock \bibinfo{address}{Berlin}.
\bibitem[{Wedin(1972)}]{wedin1972perturbation}
\bibinfo{author}{Wedin, P.{\AA}.}, \bibinfo{year}{1972}.
\newblock \bibinfo{title}{{Perturbation bounds in connection with singular
  value decomposition}}.
\newblock \bibinfo{journal}{BIT Numerical Mathematics} \bibinfo{volume}{12},
  \bibinfo{pages}{99--111}.
\bibitem[{Wedin(1973)}]{wedin1973perturbation}
\bibinfo{author}{Wedin, P.{\AA}.}, \bibinfo{year}{1973}.
\newblock \bibinfo{title}{{Perturbation theory for pseudo-inverses}}.
\newblock \bibinfo{journal}{BIT Numerical Mathematics} \bibinfo{volume}{13},
  \bibinfo{pages}{217--232}.
\bibitem[{Willard et~al.(2023)Willard, Jia, Xu, Steinbach and
  Kumar}]{willard2022integrating}
\bibinfo{author}{Willard, J.}, \bibinfo{author}{Jia, X.}, \bibinfo{author}{Xu,
  S.}, \bibinfo{author}{Steinbach, M.S.}, \bibinfo{author}{Kumar, V.},
  \bibinfo{year}{2023}.
\newblock \bibinfo{title}{{Integrating Scientific Knowledge with Machine
  Learning for Engineering and Environmental Systems}}.
\newblock \bibinfo{journal}{{ACM} Comput. Surv.} \bibinfo{volume}{55},
  \bibinfo{pages}{66:1--66:37}.

\end{thebibliography}

\end{document}